\documentclass[10pt,journal,cspaper,compsoc]{IEEEtran}

\usepackage{enumerate}
\usepackage{amssymb}
\usepackage{amsmath}
\usepackage{amsthm}
\usepackage{graphicx}
\usepackage{algorithm}
\usepackage{algpseudocode}
\usepackage{float}
\usepackage[caption=false]{subfig}
\usepackage{todonotes}
\usepackage[nocompress]{cite}
\usepackage{verbatim}
\usepackage{color}
\usepackage{multicol}
\usepackage{url}

\usepackage{array}
\newcolumntype{H}{>{\setbox0=\hbox\bgroup}c<{\egroup}@{}}

\newtheorem{theorem}{Theorem}

\newcommand{\x}{\ensuremath{\mathbf{x}}}
\newcommand{\w}{\ensuremath{\mathbf{w}}}

\newcommand{\remove}[1]{}

\begin{document}

\title{Active Sampling of Pairs and Points for Large-scale Linear Bipartite Ranking}

\author{Wei-Yuan~Shen,~\IEEEmembership{}%
        and~Hsuan-Tien~Lin,~\IEEEmembership{Member,~IEEE}% <-this % stops a space
\thanks{W.-Y. Shen and H.-T. Lin are with the Department
of Computer Science and Information Engineering, National Taiwan University, Taiwan,
 e-mail: \{r00922024, htlin\}@csie.ntu.edu.tw.}% <-this % stops a space
\thanks{Manuscript received August ??, 2014; revised ??.}}

\IEEEcompsoctitleabstractindextext{%
\begin{abstract}

Bipartite ranking is a fundamental ranking problem that learns to order relevant instances ahead of irrelevant ones. The pair-wise approach for bi-partite ranking construct a quadratic number of pairs to solve the problem, which is infeasible for large-scale data sets. The point-wise approach, albeit more efficient, often results in inferior performance. That is, it is difficult to conduct bipartite ranking accurately and efficiently at the same time. In this paper, we develop a novel active sampling scheme within the pair-wise approach to conduct bipartite ranking efficiently. The scheme is inspired from active learning and can reach a competitive ranking performance while focusing only on a small subset of the many pairs during training. Moreover, we propose a general Combined Ranking and Classification (CRC) framework  to accurately conduct bipartite ranking. The framework unifies point-wise and pair-wise approaches and is simply based on the idea of treating each instance point as a pseudo-pair. Experiments on 14 real-word large-scale data sets demonstrate that the proposed algorithm of Active Sampling within CRC, when coupled with a linear Support Vector Machine, usually outperforms state-of-the-art point-wise and pair-wise ranking approaches in terms of both accuracy and efficiency.

\end{abstract}

}

\maketitle

% is not selected <OR> if conference mode is selected - because compsoc
% conference papers position the abstract like regular (non-compsoc)
\IEEEdisplaynotcompsoctitleabstractindextext

% \fi
\IEEEpeerreviewmaketitle

\section{Introduction} \label{sec:intro}

The bipartite ranking problem aims at learning a ranking function
that orders positive instances ahead of negative ones. 
For example, in information retrieval, bipartite ranking can be used to order the preferred documents in front of the less-preferred ones within a list of search-engine results;
in bioinformatics, bipartite ranking can be used to identify genes related to a certain disease by ranking the relevant genes higher than irrelevant ones.
Bipartite ranking algorithms take some positive and negative instances as the input data, and produce a ranking function that maps an instance to a real-valued score.
Given a pair of a positive instance and a negative one, we say that the pair is mis-ordered if the ranking function gives a higher score to the negative instance.
The performance of the ranking function is measured by the probability of mis-ordering an unseen pair of randomly chosen positive and negative instances, 
which is equal to one minus the Area Under the ROC Curve (AUC)~\cite{fawcett2006introduction}, 
a popular criterion for evaluating the sensitivity and the specificity of binary classifiers 
in many real-world tasks~\cite{clemenccon2008ranking} and large-scale data mining competitions~\cite{caruana2004kdd,wu2012two}.

Given the many potential applications in information retrieval, bioinformatics, and recommendation systems, bipartite ranking has received much research attention in the past two decades~\cite{freund2003efficient,cortes2004auc,joachims2006training,liu2009learning,kotllowski2011bipartite,ertekin2011equivalence}.
Many existing bipartite ranking algorithms explicitly or implicitly reduce the problem to binary classification to inherit the benefits from the well-developed methods in binary classification~\cite{herbrich2000large,freund2003efficient,brefeld2005auc,kotllowski2011bipartite,ertekin2011equivalence}.
The majority of those reduction-based algorithms can be categorized to two approaches: the pair-wise approach and the point-wise one.
The pair-wise approach transforms the input data of positive and negative instances to {\it pairs} of instances, and learns
a binary classifier for predicting whether the first instance in a pair should be scored higher than the second one. 
Note that for an input data set that contains $N^+$ positive instances and $N^-$ negative ones, 
the pair-wise approach trains a binary classifier by optimizing an objective function that consists of 
$N^+N^-$ terms, one for each pair of instances.
The pair-wise approach comes with strong theoretical guarantee. 
For example,~\cite{balcan2007robust} shows that a low-regret ranking function can indeed be formed by a low-regret binary classifier.
The strong theoretical guarantee leads to promising experimental results in many
state-of-the-art bipartite ranking algorithms, such as RankSVM~\cite{herbrich2000large}, RankBoost~\cite{freund2003efficient} and RankNet~\cite{burges2005learning}.
Nevertheless, the number of pairs in the input data can easily be of size $\Theta(N^2)$, where $N$ is the size
of the input data, if the data is not extremely unbalanced.
The quadratic number of pairs with respect to $N$ makes the pair-wise approach computationally infeasible for large-scale data sets in general, except in a few special algorithms like RankBoost~\cite{freund2003efficient} or the efficient linear RankSVM~\cite{joachims2006training}. 
RankBoost enjoys an efficient implementation by reducing the quadratic number of pair-wise terms in the objective function
to a linear number of equivalent terms; efficient linear RankSVM transforms the pair-wise optimization formulation to an equivalent formulation that can be solved in subquadratic time complexity~\cite{kotllowski2011bipartite}.

On the other hand, the point-wise approach directly runs binary classification on the positive and negative instance {\it points} of the input data, and takes the scoring function behind the resulting binary classifier as the ranking function.
In some special cases~\cite{freund2003efficient,rajaram2007diverse}, such as AdaBoost~\cite{freund1997decision} and its pair-wise sibling RankBoost~\cite{freund2003efficient}, the point-wise approach is shown to be equivalent to the corresponding pair-wise one~\cite{rudin2009margin,ertekin2011equivalence}. 
In other cases, the point-wise approach often operates with an approximate objective function that involves only $N$ terms~\cite{kotllowski2011bipartite}.
For example, \cite{kotllowski2011bipartite} shows that minimizing the exponential or the logistic loss function on the instance points decreases an upper bound on the number of mis-ordered pairs within the input data.
Because of the approximate nature of the point-wise approach, its ranking performance can sometimes be inferior to the pair-wise approach.

From the discussion above, we see that the pair-wise approach leads to more satisfactory performance while the point-wise approach comes with efficiency, and there is a trade-off between the two.
In this paper, we are interested in designing bipartite ranking algorithms that enjoy both satisfactory performance and efficiency for large-scale bipartite ranking.
The concrete problem setup will be provided in Section~\ref{sec:setup}. We focus on using the linear Support Vector Machine~(SVM)~\cite{vapnik1999nature} given its recent advances for efficient large-scale learning~\cite{yuan2012recent}.
We first show that the loss function behind the usual point-wise SVM~\cite{vapnik1999nature} minimizes an upper bound on the loss function behind RankSVM, 
which suggests that the point-wise SVM could be an approximate bipartite ranking algorithm that enjoys efficiency. 
Then, we design a better ranking algorithm with two major contributions.

As illustrated in Section~\ref{sec:alg}, firstly, we study an active sampling scheme to select important pairs for the pair-wise approach and name the scheme Active Sampling for RankSVM (ASRankSVM).
The scheme makes the pair-wise SVM computationally feasible by focusing only on a few valuable pairs out of the quadratic number of pairs
, and allows us to overcome the challenge of having a quadratic number of pairs.
The active sampling scheme is inspired by active learning, another popular machine learning setup that aims to save the efforts of labeling~\cite{settles2010active}. 
More specifically, we discuss the similarity and differences between active sampling (selecting a few of valuable pairs within a pool of potential pairs)
and pool-based active learning (labeling a few of valuable instances within a pool of unlabeled instances), and propose some active sampling
strategies based on the similarity.
Secondly, we propose a general framework that unifies the point-wise SVM and the pair-wise SVM (RankSVM) as special cases. 
The framework, called combined ranking and classification (CRC), is simply based on the idea of treating each instance point as a pseudo-pair.
The CRC framework coupled with active sampling (ASCRC) improves the performance of the point-wise SVM by considering not only points but also pairs in its objective function.

Performing active sampling within the CRC framework leads to a promising algorithm for large-scale linear bipartite ranking. 
In Section~\ref{sec:exp}, we conduct experiments on 14 real-world large-scale data sets and compare the proposed algorithms (ASRankSVM and ASCRC) with several state-of-the-art bipartite ranking algorithms, including the point-wise linear SVM~\cite{liblinear}, the efficient linear RankSVM~\cite{joachims2006training}, and the Combined Ranking and Regression (CRR) algorithm~\cite{sculley2010combined} which is closely related to the CRC framework.
In addition, we demonstrate the robustness and the efficiency of the active sampling strategies and discuss some advantages and disadvantages of different strategies. 
The results show that~ASRankSVM is able to efficiently sample only $8,000$ of the more than millions of the possible pairs to achieve better performance than other state-of-the-art algorithms that use all the pairs, while ASCRC that considers the pseudo-pairs can sometimes be helpful. Those results validate that the proposed algorithm can indeed enjoy both satisfactory performance and efficiency for large-scale bipartite ranking.

A preliminary version of this paper appeared in the 5th Asian Conference on Machine Learning~\cite{shen2013active}.
The paper is then enriched by
\begin{enumerate}[{1)}]
	\item extending the design of the proposed CRC framework to allow a threshold term for the classification part in Section~\ref{sec:threshold}

	\item examining the necessity of each part of the proposed CRC framework in Section~\ref{sub:robustness}

	\item studying the effect of the budget parameter of active sampling in Section~\ref{sub:budget}

\end{enumerate}

\section{Setup and Related Works} \label{sec:setup}

In a bipartite ranking problem, we are given a training set $\mathcal{D} = \{(\x_k, y_k)\}_{k=1}^N$, where each $(\x_k, y_k)$ is a training instance
with the feature vector $\x_k$ in an $n$-dimensional space $\mathcal{X} \subseteq \mathbb{R}^n$ and the binary label $y_k \in \{+1, -1\}$. 
Such a training set is of the same format as the training set in usual binary classification problems. 
We assume that the instances $(\x_k, y_k)$ are drawn {\it i.i.d.} from an unknown distribution $P$ on $\mathcal{X} \times \{+1, -1\}$.
Bipartite ranking algorithms take $\mathcal{D}$ as the input and learn a ranking function $r\colon \mathcal{X} \rightarrow \mathbb{R}$ that maps a feature vector $\x$ to a real-valued score $r(\x)$. 

For any pair of two instances, we call the pair mis-ordered by $r$ {\it iff} the pair contains a positive instance $(\x_+, +1)$ and a negative one $(\x_-, -1)$ while $r(\x_+) \le r(\x_-)$.
For a distribution $P$ that generates instances $(\x, y)$, we can define its pair distribution $P_2$ that generates $(\x, y, \x', y')$ to be the conditional probability of
sampling two instances $(\x, y)$ and $(\x', y')$ from $P$, conditioned on $y \neq y'$.
Then, let the expected bipartite ranking loss $L_{P}(r)$ for any ranking function $r$ be the expected number of mis-ordered pairs over $P_2$.
\begin{equation*}
	L_{P}(r) = \operatorname*{\mathbb{E}}_{(\x,y,\x',y') \sim P_2}\left[\mathrm{I}\Bigl((y-y')(r(\x)-r(\x')) \leq 0\Bigr)\right],
\end{equation*}
where $\mathrm{I}(\bullet)$ is an indicator function that returns $1$ iff the condition $(\bullet)$ is true, and returns $0$ otherwise. 
The goal of bipartite ranking is to use the training set~$\mathcal{D}$ to learn a ranking function $r$ that minimizes the expected bipartite ranking loss~$L_{P}(r)$.

Because $P$ is unknown, $L_P(r)$ cannot be computed directly. 
Thus, bipartite ranking algorithms usually resort to the empirical bipartite ranking loss~$L_\mathcal{D}(r)$, which takes the expectation over the pairs in $\mathcal{D}$ instead of over
the pair distribution $P_2$, with the hope that $L_\mathcal{D}(r)$ would be sufficiently close to $L_P(r)$ when the model complexity of the candidate ranking functions $r$ is properly controlled. 
Denote $\mathcal{D}^+$ as the set of the positive instances in~$\mathcal{D}$, and $\mathcal{D}^-$ as the set of negative instances in~$\mathcal{D}$.
The formal definition of $L_\mathcal{D}(r)$ is
\begin{equation*}
	L_{\mathcal{D}}(r) = \frac{1}{N^+N^-}\sum_{\x_i \in \mathcal{D}^+}\sum_{\x_j \in \mathcal{D}^-} \mathrm{I}\Bigl(r(\x_i) \leq r(\x_j)\Bigr).
\end{equation*}

The bipartite ranking loss $L_{P}(r)$ is closely related to the area under the ROC curve (AUC), which
is commonly used to evaluate the sensitivity and the specificity of binary classifiers~\cite{caruana2004kdd,brefeld2005auc,clemenccon2008ranking,wu2012two}. 
More specifically, AUC calculates the expected number of {\it correctly-ordered} pairs. 
Hence, $\mathrm{AUC}_{\bullet}(r) = 1 - L_{\bullet}(r)$ for $\bullet = P$ or $\mathcal{D}$, and higher AUC indicates better ranking performance.

Bipartite ranking is a special case of the general ranking problem in which the labels $y$ can be any real value, not necessarily $\{+1, -1\}$.
%For example, recommendation systems may allow users to enter their preferences (labels) on the item $\x$ with real-valued or ordinal-scaled scores. 
There are lots of recent studies on improving the accuracy~\cite{herbrich2000large,quoc2007learning,duchi2010consistency} and efficiency~\cite{freund2003efficient,ailon2007efficient} of general ranking problems. 
In this paper, instead of considering the general ranking problem, we focus on using the specialty of bipartite ranking, namely its connection to binary classification, to improve the accuracy and the efficiency.

Motivated by the recent advances of linear models for efficient large-scale learning~\cite{yuan2012recent}, we consider linear models for efficient large-scale bipartite ranking. 
That is, the ranking functions would be of the form $r(\x) = \w^T \x$, which is linear to the components of the feature vector $\x$. 
In particular, we study the linear Support Vector Machine (SVM)~\cite{vapnik1999nature} for bipartite ranking. 
There are two possible approaches for adopting the linear SVM on bipartite ranking problem, the {\it pair-wise} SVM approach and the {\it point-wise} SVM approach.

The pair-wise approach corresponds to the famous RankSVM algorithm~\cite{herbrich2000large}, which is originally designed for ranking with ordinal-scaled scores, but can be easily extended to general ranking with real-valued labels or restricted to bipartite ranking with binary labels. 
For each positive instance $(\x_i, y_i = +1)$ and negative instance $(\x_j, y_j = -1)$, the pair-wise approach transforms the two instances to two symmetric {\it pairs} of instance $((\x_i, \x_j), +1)$ and $((\x_j, \x_i), -1)$, the former for indicating that $\x_i$ should be scored higher than $\x_j$ and the latter for indicating that~$\x_j$ should be scored lower than $\x_i$. In the linear case,
the pairs transformed from~$\mathcal{D}$ are then fed to a linear SVM for learning a ranking function of the form $r(\x) = \w^T \x$.
%When using a linear SVM, $\boldsymbol{\phi}$ is simply the identity function. 
Then, for the pair $((\x_i, \x_j), +1)$, we see that $\mathrm{I}\Bigl(r(\x_i) \leq r(\x_j)\Bigr) = 0$ {\it iff} $\w^T(\x_i-\x_j)>0$. 
Define the transformed feature vector $\x_{ij} = \x_i - \x_j$ and the transformed label $y_{ij} = \mathrm{sign}(y_i-y_j)$, we can equivalently view the pair-wise linear SVM as simply running a linear SVM on the pair-wise training set $\mathcal{D}_{pair}=\left\{(\x_{ij}, y_{ij}) | y_i \neq y_j \right\}$.
The pair-wise linear SVM minimizes the hinge loss as a surrogate to the 0/1 loss on $\mathcal{D}_{pair}$~\cite{steck2007hinge}, and the 0/1 loss on $\mathcal{D}_{pair}$ is equivalent to $L_{\mathcal{D}}(r)$, the empirical bipartite ranking loss of interest. 
That is, if the linear SVM learns an accurate binary classifier using $\mathcal{D}_{pair}$, the resulting ranker $r(\x) = \w^T \x$ would also be accurate in terms of the bipartite ranking loss.

Denote the hinge function $\mathrm{max}(\bullet, 0)$ by $[\bullet]_+$, RankSVM solves the following optimization problem
\begin{equation} \label{eq:ranksvm}
	\min_{\w} \frac{1}{2}\w^T\w + \sum_{ \x_{ij} \in \mathcal{D}_{pair}} C_{ij}[1-\w^T y_{ij} \x_{ij}]_+ \ ,
\end{equation}
where $C_{ij}$ denotes the weight of the pair $\x_{ij}$. 
Because of the symmetry of $\x_{ij}$ and~$\x_{ji}$, we naturally assume that $C_{ij}=C_{ji}$.
In the original RankSVM formulation, $C_{ij}$ is set to a constant for all the pairs.
Here we list a more flexible formulation~\eqref{eq:ranksvm} to facilitate some discussions later.
RankSVM has reached promising bipartite ranking performance in the literature~\cite{brefeld2005auc}.
Because of the symmetry of positive and negative pairs, we can equivalently solve~\eqref{eq:ranksvm} on those positive pairs with $y_{ij} = 1$. 
The number of such positive pairs is $N^+ N^-$ if there are~$N^+$ positive instances and $N^-$ negative ones.
The huge number of pairs make it difficult to solve \eqref{eq:ranksvm} with a na{\"ive} quadratic programming algorithm.

In contrast with the na{\"i}ve RankSVM, the efficient linear RankSVM~\cite{joachims2006training} changes \eqref{eq:ranksvm} to a more sophisticated and equivalent one with an exponential number of constraints, each corresponding to a particular linear combination of the pairs.
Then, it reaches $O(N \log N)$ time complexity by using a cutting-plane solver to identify the most-violated constraints iteratively, while
the constant hidden in the big-$O$ notation depends on the parameter $C_{ij}$ as well as the desired precision of optimization. 
The subquadratic time complexity of the efficient RankSVM can make it much slower than the point-wise approach (to be discussed below), and hence may not always be fast enough
for large-scale bipartite ranking.

The {\it point-wise} SVM approach, on the other hand, directly runs an SVM on the original training set $\mathcal{D}$ instead of $\mathcal{D}_{pair}$.
That is, in the linear case, the point-wise approach solves the following optimization problem
\begin{equation} \label{eq:svm}
	\min_{\w} \frac{1}{2}\w^T\w + C_+\sum_{\x_i \in \mathcal{D}^+}[1-\w^T\x_i]_+ + C_- \sum_{\x_j \in \mathcal{D}^-}[1+\w^T\x_j]_+ \ .
\end{equation}
Such an approach comes with some theoretical justification~\cite{kotllowski2011bipartite}.
In particular,  the 0/1 loss on $\mathcal{D}$ has been proved to be an upper bound of the empirical bipartite ranking loss. 
In fact, the bound can be tightened by adjusting $C_+$ and~$C_-$ to balance the distribution of the positive and negative instances in $\mathcal{D}$.
When $C_+ = C_-$,~\cite{brefeld2005auc} shows that the point-wise approach \eqref{eq:svm} is inferior to the pair-wise approach \eqref{eq:ranksvm} in performance. 
The inferior performance can be attributed to the fact that the point-wise approach only operates with an approximation (upper bound) of the bipartite ranking loss of interest.

Next, inspired by the theoretical result of upper-bounding the bipartite ranking loss with a balanced 0/1 loss, we study the connection between \eqref{eq:ranksvm} and \eqref{eq:svm} by balancing the hinge loss in \eqref{eq:svm}. 
In particular, as shown in Theorem~\ref{upper}, a balanced form of \eqref{eq:svm} can be viewed as minimizing an upper bound of the objective function within \eqref{eq:ranksvm}.
In other words, the weighted point-wise SVM can be viewed as a reasonable baseline algorithm for large-scale bipartite ranking problem.

\begin{theorem} \label{upper}
	Let $ C_{ij}= \frac{C}{2} $ be a constant in \eqref{eq:ranksvm}; $C_+ = 2N^- \cdot C$ and $ C_- = 2N^+ \cdot C$ in \eqref{eq:svm}.
	Then, for every $\w$, the objective function of \eqref{eq:ranksvm} is upper-bounded by $\frac{1}{4}$ times the objective function of \eqref{eq:svm}.
\end{theorem}
\begin{proof} The objective function of \eqref{eq:ranksvm}
%Because \[
%[1-\w^T\x_{ij}]_+ \le \frac{1}{2}\left([1-2\w^T\x_i]_+ + [1+2\w^T\x_j]_+\right),
%\]
%starting from  we get
\begin{eqnarray*}
%	&& \mbox{}\\
%\frac{1}{2} \w^T \w + \sum_{\x_{ij} \in \mathcal{D}_{pair}} C_{ij}[1-\w^T y_{ij} \x_{ij}]_+ \\
 & = & \frac{1}{2} \w^T \w + \sum_{\x_{ij} \in \mathcal{D}_{pair}, y_{ij}=+1} C [1-\w^T\x_{ij}]_+ \\
		& \leq & \frac{1}{2} \w^T \w \\
  &&+ \frac{C}{2} \sum_{\x_i \in \mathcal{D}^+} \sum_{\x_j \in \mathcal{D}^-}\left([1-2\w^T\x_i]_+ + [1+2\w^T\x_j]_+\right) \\
		& = & \frac{1}{2} \w^T \w 
  + \frac{C}{2} \cdot N^-\sum_{\x_i \in \mathcal{D}^+}[1-2\w^T\x_i]_+ \\
  && + \frac{C}{2} \cdot N^+\sum_{\x_j \in \mathcal{D}^-}[1+2\w^T\x_j]_+.
\end{eqnarray*}
The theorem can be easily proved by substituting $2\w$ with a new variable~$\mathbf{u}$.
\end{proof}

\section{Proposed Algorithm} \label{sec:alg}

As discussed in the previous section, the pair-wise approach~\eqref{eq:ranksvm} is infeasible on large-scale data sets due to the huge number of pairs. 
Then, either some random sub-sampling of the pairs are needed~\cite{sculley2010combined}, or the less-accurate point-wise approach~\eqref{eq:svm} is taken as the approximate alternative~\cite{kotllowski2011bipartite}. 
Nevertheless, the better ranking performance of the pair-wise approach over the point-wise one suggest that some of the key pairs shall carry more valuable information than the instance-points. 
Next, we design an algorithm that samples a few key pairs actively during learning. 
%The resulting algorithm achieves better performance than the point-wise approaches because of the key pairs, and enjoys better efficiency than the pair-wise approach because of the sampling.
We first show that some proposed active sampling schemes
%, which are inspired by the many existing methods in active learning~\cite{lewis1994sequential,roy2001toward,settles2010active}, 
can help identify those key pairs better than random sub-sampling.
Then, we discuss how we can unify point-wise and pair-wise ranking approaches under the same framework.

\subsection{Pool-based Active Learning}

The pair-wise SVM approach~\eqref{eq:ranksvm} is challenging to solve because of the huge number of pairs involved in $\mathcal{D}_{pair}$. 
To make the computation feasible, we can only afford to work on a small subset of $\mathcal{D}_{pair}$ during training. 
Existing algorithms conquer the computational difficulty of the huge number of pairs in different ways. 
The Combined Ranking and Regression approach~\cite{sculley2010combined} performs stochastic gradient descent on its objective function, which essentially selects within the huge number of pairs in a random manner; the efficient RankSVM~\cite{joachims2006training} identifies the most-violated constraints during optimization,
which corresponds to selecting the most valuable pairs from an optimization perspective.

We take an alternative route and hope to select the most valuable pairs from a {\it learning} perspective. 
That is, our task is to iteratively select a small number of valuable pairs for training while reaching similar performance to the pair-wise approach that trains with all the pairs. 
One machine learning setup that works for a similar task is {\it active learning}~\cite{settles2010active}, which iteratively select a small number of valuable instances for labeling (and training) while reaching similar performance to the approach that trains with all the instances fully labeled.
\cite{ailon2012active} avoids the quadratic number of pairs in the general ranking problem from an active learning perspective, and 
proves that selecting a subquadratic number of pairs is sufficient to obtain a ranking
function that is close to the optimal ranking function trained by using all the pairs. 
The algorithm is theoretical in nature, while many other promising active learning tools~\cite{lewis1994sequential,roy2001toward,settles2010active} have not been explored for selecting valuable pairs in large-scale bipartite ranking.

Next, we start exploring those tools by providing a brief review about active learning. We focus on the setup of pool-based active learning~\cite{settles2010active} because of its strong connection to our needs. 
In a pool-based active learning problem, the training instances are separated into two parts, the labeled pool~($\mathcal{L}$) and the unlabeled pool~($\mathcal{U}$).
As the name suggests, the labeled pool consists of labeled instances that contain both the feature vector $\x_k$ and its corresponding label $y_k$, and the unlabeled pool contains unlabeled instances $\x_\ell$ only.
Pool-based active learning assumes that a (huge) pool of unlabeled instances is relatively easy to gather, while labeling those instances can be expensive. Therefore, we hope to achieve promising learning performance with as few labeled instances as possible.

A pool-based active learning algorithm is generally iterative. In each iteration, there are two steps: the training step and the querying step. 
In the training step, the algorithm trains a decision function from the labeled pool; in the querying step, the algorithm selects
one (or a few) unlabeled instances, queries an oracle to label those instances, and moves those instances from the unlabeled pool to the labeled one.
The pool-based active learning framework repeats the training and querying steps iteratively until a given budget $B$ on the number of queries is met, with the hope that the decision functions returned throughout the learning steps are as accurate as possible for prediction.

Because labeling is expensive, active learning algorithms aim to select the most valuable instance(s) from the unlabeled pool at each querying step. Various selection criteria have been proposed to describe the value of an unlabeled instance~\cite{settles2010active}, such as
uncertainty sampling~\cite{lewis1994sequential}, expected error reduction~\cite{roy2001toward}, and expected model change~\cite{settles2008multiple}.

Moreover, there are several works that solve bipartite ranking under the active learning scenario~\cite{yu2005svm,donmez2008optimizing,donmez2009active}. 
For example, \cite{donmez2008optimizing} selects points that reduce the ranking loss functions most from the unlabeled pool while~\cite{donmez2009active} selects points that
maximize the AUC in expectation. 
Nevertheless, these active learning algorithms require either sorting or enumerating over the huge unlabeled pool in each querying step. 
The sorting or enumerating process can be time consuming, but have not been considered seriously because labeling is assumed to be even more expensive. 
We will discuss later that those algorithms that require sorting or enumerating may not fit our goal.

\subsection{Active Sampling}

Following the philosophy of active learning, we propose the active sampling scheme for choosing a smaller set of key pairs on the huge training set $\mathcal{D}_{pair}$.
We call the scheme {\it Active Sampling} in order to highlight some differences to active learning.
One particular difference is that RankSVM~\eqref{eq:ranksvm} only requires optimizing with {\it positive pairs}. 
Then, the label $y_{ij}$ of a pair is a constant $1$ and thus easy to get during active sampling, while the label in active learning remains
unknown before the possibly expensive querying step. 
Thus, while active sampling and active learning both focus on using as few labeled data as possible, the costly part of the active sampling scheme is on training rather than querying.

For active sampling, we denote $B$ as the budget on the number of pairs that can be used in training, which plays a similar role to the budget on querying in active learning. 
In brief, active sampling chooses $B$ informative pairs iteratively for solving the optimization problem \eqref{eq:ranksvm}.
We separate the pair-wise training set~$\mathcal{D}_{pair}$ into two parts, the chosen pool~($\mathcal{L}^*$) and the unchosen pool~($\mathcal{U}^*$).
The chosen pool is the subset of pairs to be used for training, and the unchosen pool contains the unused pairs.
The chosen pool is similar to the labeled pool in pool-based active learning; the unchosen pool acts like the unlabeled pool.
The fact that it is almost costless to ``label'' the instances in the unchosen pool allows us to design simpler sampling strategies than those commonly used for active learning, because no effort is needed to estimate the unknown labels.

\begin{algorithm}
	\caption{Active Sampling}
	\label{main}
	\begin{algorithmic}
		\Require
		the initial chosen pool, $\mathcal{L}^*$;
		the initial unchosen pool, $\mathcal{U}^*$;
		the regularization parameters, $\{C_{ij}\}$. 
		the number of pairs sampled per iteration, $b$;
		the budget on the total number of pairs sampled, $B$;
		the sampling strategy, $\textsf{Sample}\colon (\mathcal{U}^*,\w) \rightarrow \x_{ij}$
                %that chooses a pair from $\mathcal{U}^*$.
		\Ensure
		the ranking function represented by the weights $\w$.

		\State  $\w = \textsf{linearSVM}(\mathcal{L}^*,\{C_{ij}\},\mathbf{0})$; 
		\Repeat 
		\For {$i = 1 \to b$}
		\State  $\x_{ij} = \textsf{Sample}(\mathcal{U}^*,\w)$;
		\State  $\mathcal{L}^* = \mathcal{L}^* \cup \{(\x_{ij}, y_{ij})\}$ and $ \mathcal{U}^* = \mathcal{U}^*\setminus\{\x_{ij}\}$;
		\EndFor
		\State  $\w = \textsf{linearSVM}(\mathcal{L}^*,\{C_{ij}\},\w)$; 
		\Until {($|\mathcal{L}^*| \geq B$)}\\		
		\Return $\w$;
	\end{algorithmic}
\end{algorithm}

The proposed scheme of active sampling is illustrated in Algorithm~\ref{main}. 
The algorithm takes an initial chosen pool~$\mathcal{L}^*$ and an initial unchosen pool~$\mathcal{U}^*$, where we simply
mimic the usual setup in pool-based active learning by letting~$\mathcal{L}^*$ be a randomly chosen subset of~$\mathcal{D}_{pair}$ and~$\mathcal{U}^*$ be the set of unchosen pairs in~$\mathcal{D}_{pair}$. 
In each iteration of the algorithm, we use \textsf{Sample} to actively choose~$b$ instances to be moved from $\mathcal{U}^*$ to $\mathcal{L}^*$. 
The strategy \textsf{Sample} takes the current ranking function~$\w$ into account.
After sampling, a \textsf{linearSVM} is called to learn from $\mathcal{L}^*$ along with the weights in~$\{C_{ij}\}$. 
We feed the current $\w$ to the \textsf{linearSVM} solver to allow a {\it warm-start} in optimization. 
The warm-start step enhances the efficiency and the performance.
The iterative procedure continues until the budget~$B$ of chosen instances is fully consumed.

Another main difference between the active sampling scheme and typical pool-based active learning is that we sample~$b$ instances before the training step, while pool-based active learning often considers executing the training step right after querying the label of one instance. 
The difference is due to the fact that the pair-wise labels~$y_{ij}$ can be obtained very easily and thus
sampling and labeling can be relatively cheaper than querying in pool-based active learning. 
Furthermore, updating the weights right after knowing one instance may not lead to much improvement and can be too
time consuming for the large-scale bipartite ranking problem that we want to solve.

\subsection{Sampling Strategies} \label{sec:sampling}

Next, we discuss some possible sampling strategies that can be used in Algorithm~\ref{main}. 
One na{\"ive} strategy is to passively choose a random sample within~$\mathcal{U}^*$.
For active sampling strategies, we define two measures that estimate the (learning) value of an unchosen pair. 
The two measures correspond to well-known criteria in pool-based active learning.
Let~$\x_{ij}$ be the unchosen pair in~$\mathcal{U}^*$ with~$y_{ij} = 1$, the two measures with respect to the current ranking function~$\w$ are
\begin{equation}
	closeness(\x_{ij},\w) = |\w^T\x_{ij}|
\end{equation}
\begin{equation}
	correctness(\x_{ij},\w) =  -[1-\w^T\x_{ij}]_+
\end{equation}

The $closeness$ measure corresponds to one of the most popular criteria in
pool-based active learning called {\it uncertainty sampling}~\cite{lewis1994sequential}. 
It captures the uncertainty of the ranking function $\w$ on the unchosen pair.
Intuitively, a low value of $closeness$ means that the ranking function finds it hard to distinguish the two instances in the pair,
which implies that the ranking function is less confident on the pair.
Therefore, sampling the unchosen pairs that come with the lowest $closeness$ values may improve the ranking performance by resolving the uncertainty.

On the other hand, the $correctness$ measure is related to another common criterion
in pool-based active learning called {\it expected error reduction}~\cite{roy2001toward}.
It captures the performance of the ranking function~$\w$ on the unchosen pair. 
Note that this exact correctness measure is only available within our active sampling scheme because we
know the pair-label~$y_{ij}$ to always be $1$ without loss of generality, while usual active learning algorithms do not know the exact
measure before querying and hence have to estimate it~\cite{donmez2008optimizing,donmez2009active}.
A low value of $correctness$ indicates that the ranking function does not perform well on the pair. 
Then, sampling the unchosen pairs that come with the lowest $correctness$ values 
may improve the ranking performance by correcting the possible mistakes.
Moreover, sampling the pair with lowest $correctness$ value shall change~$\w$ the most in general, which echoes
another criterion in pool-based active learning called {\it expected model change}~\cite{settles2008multiple}.

Similar to other active learning algorithms~\cite{donmez2008optimizing,donmez2009active}, computing the pairs that come with the lowest $closeness$ or $correctness$ values can be time consuming, as it requires at least evaluating the values of $\w^T \x_k$ for each instance $(\x_k, y_k) \in \mathcal{D}$, and then computing the measures on the pairs along with some selection or sorting steps that may be of super-linear time complexity~\cite{joachims2006training}. 
Thus, such a {\it hard version} of active sampling is not computationally feasible for large-scale bipartite ranking. 
Next, we discuss the {\it soft version} of active sampling that randomly chooses pairs that come with lower $closeness$ or $correctness$ values by rejection sampling.

%% \floatname{algorithm}{Algorithm}
%% \begin{algorithm}[h]
%% 	\caption{Soft Version of Active Sampling}
%% 	\label{alg:sample}
%% 	\begin{algorithmic}
%% 		\Require
%% 		the current ranking function represented by the weights $\w$;
%% 		the unchosen pool, $\mathcal{U}^*$.
%% 		\Ensure $\x_{ij}$, the sampled pair.
%% 		\Repeat
%% 		\State Sample a pair $\x_{ij}$ uniformly from $\mathcal{U}^*$;
%% 		\State Calculate a probability value $p_{ij}$ from $\w$;
%% 		\Until {( $\textsf{random}() < p_{ij}$ )}\\
%% 		\Return $\x_{ij}$;
%% 	\end{algorithmic}
%% \end{algorithm}

%Algorithm~\ref{alg:sample} illustrates t
The soft version of active sampling can be described as follows:
we consider a rejection sampling step that samples a pair~$\x_{ij}$ with probability~$p_{ij}$
based on a method~\textsf{random()} that generates random numbers between $[0, 1]$. 
A pair that comes with a lower $closeness$ or $correctness$ values would enjoy a higher probability~$p_{ij}$ of being accepted.

Next, we define the probability value functions that correspond to the hard versions of $closeness$ and $correctness$.
Both value functions are in the shape of the sigmoid function, which is widely used to represent probabilities in logistic regression and neural networks~\cite{baum1988supervised}.
For soft $closeness$ sampling, we define the probability value as:
$p_{ij} \equiv 2 / \left(1+e^{|\w^T\x_{ij}|}\right) $
For soft $correctness$ sampling, we define $p_{ij}$ as: 
$ p_{ij} \equiv 1- 2 / \left(1+e^{[1-\w^T(\x_{ij})]_+}\right) $
We take different forms of soft versions because $closeness$ is of range $[0,\infty)$ while $correctness$ is of range $(-\infty,0]$.

Note that the sampling strategies above, albeit focusing on the most valuable pairs, is inheritedly biased. 
The chosen pool may not be representative enough of the whole training set because of the biased sampling strategies.
There is a simple way that allows us to correct the sampling bias for learning a ranking function that performs well on the original bipartite ranking loss of interest.
We take the idea of~\cite{horvitz1952generalization} to weight the sampled pair by the inverse of its probability of being sampled.
That is, we could multiply the weight~$C_{ij}$ for a chosen pair~$\x_{ij}$ by~$\frac{1}{p_{ij}}$ when it gets returned by the rejection sampling.

\subsection{Combined Ranking and Classification} \label{sec:CRC}

Inspired by Theorem~\ref{upper}, the points can also carry some information for ranking.
Next, we study how we can take those points into account during active sampling.
We start by taking a closer look at the similarity and differences between the point-wise SVM~\eqref{eq:svm} and the pair-wise SVM~\eqref{eq:ranksvm}. 
The pair-wise SVM considers the weighted hinge loss on the pairs $\x_{ij}=\x_i-\x_j$, while the point-wise SVM considers the weighted hinge loss on the points~$\x_k$. 
Consider one positive point $(\x_k, +1)$. Its hinge loss is $[1 - \w^T \x_i]_+$, which is the same as $[1 - \w^T (\x_i - \mathbf{0})]_+$. 
In other words, the positive point $(\x_i, +1)$ can also be viewed as a {\it pseudo-pair} that consists of $(\x_i, +1)$ and $(\mathbf{0}, -1)$. 
Similarly, a negative point $(\x_j, -1)$ can be viewed as a pseudo-pair that consists of $(\x_j, -1)$ and $(\mathbf{0}, +1)$. Let the set of all pseudo-pairs within $\mathcal{D}$ be $\mathcal{D}_{pseu}$.
%\begin{eqnarray*}
%\mathcal{D}_{pseu} &=& \{(\x_{i0} = \x_i- \mathbf{0}, +1) | \x_i \in \mathcal{D}^+ \} \nonumber \\
%&& \cup \{ (\x_{0j} = \mathbf{0}-\x_j, +1)|\x_j \in \mathcal{D}^-\}  \nonumber \\
%&& \cup \{(\x_{0i} = \mathbf{0} - \x_i, -1) | \x_i \in \mathcal{D}^+ \} \\
%&& \cup \{ (\x_{j0} = \x_j - \mathbf{0}, -1)|\x_j \in \mathcal{D}^-\}. \nonumber
%\end{eqnarray*}
Then, the point-wise SVM~\eqref{eq:svm} is just a variant of the pair-wise one~\eqref{eq:ranksvm} using the pseudo-pairs and some particular weights.
Thus, we can easily unify the point-wise and the pair-wise SVMs together by minimizing some weighted hinge loss
on the joint set $\mathcal{D}^* = \mathcal{D}_{pair}\cup \mathcal{D}_{pseu}$ of pairs and pseudo-pairs. 
By introducing a parameter $\gamma \in [0, 1]$ to control the relative importance between the real pairs and the pseudo-pairs, we propose the following novel formulation.
\begin{eqnarray}
	\min_{\w} && \frac{1}{2}\w^T\w + \gamma \sum_{\x_{ij} \in \mathcal{D}_{pair}^+} C_{crc}^{(ij)} [1-\w^T\x_{ij}]_+  \nonumber\\
&&+(1-\gamma) \sum_{\x_{k\ell} \in \mathcal{D}_{pseu}^+} C_{crc}^{(k\ell)} \cdot [1-\w^T\x_{k\ell}]_+
 \ ,  \label{eq:opt}
\end{eqnarray}
where $\mathcal{D}_{pair}^+$ and $\mathcal{D}_{pseu}^+$ denote the set of positive pairs and positive pseudo-pairs, respectively.
The new formulation~\eqref{eq:opt} combines the point-wise SVM and the pair-wise SVM in its objective function, and hence is named the
Combined Ranking and Classification (CRC) framework.
When $\gamma = 1$, CRC takes the pair-wise SVM~\eqref{eq:ranksvm} as a special case with $C_{ij} = 2 C_{crc}^{(ij)}$;
when $\gamma = 0$, CRC takes the point-wise SVM~\eqref{eq:svm} as a special case with $C_+ = C_{crc}^{(i0)}$ and $C_- = C_{crc}^{(0j)}$. 
The CRC framework~\eqref{eq:opt} remains as challenging to solve as the pair-wise SVM approach~\eqref{eq:ranksvm} because of the huge number of pairs. 
However, the general framework can be easily extended to the active sampling scheme, and hence be solved efficiently. 
We only need to change the training set from $\mathcal{D}_{pair}$ to the joint set $\mathcal{D}^*$, and multiply the probability value $p_{ij}$ in the soft version sampling by $\gamma$ or ($1-\gamma$) for actual pairs or pseudo-pairs. 

The CRC framework is closely related to the algorithm of Combined Ranking and Regression (CRR)~\cite{sculley2010combined} for general ranking. 
The CRR algorithm similarly considers a combined objective function of the point-wise terms and the pair-wise terms for improving the ranking performance. 
The main difference between CRR and CRC is that the CRR approach takes the squared loss on the points, while CRC 
takes the nature of bipartite ranking into account and considers the hinge loss on the points. 
On the other hand, the idea of combining pair-wise and point-wise approaches had been used in another machine learning setup, the multi-label classification problem~\cite{tsoumakas2007multi}. 
The algorithm of Calibrated Ranking by Pairwise Comparison~\cite{furnkranz2008multilabel} assumes a calibration label between relevant and irrelevant labels, 
and hence unifies the pair-wise and point-wise label learning for multi-label classification.
The calibration label plays a similar role to the zero-vector in the pseudo-pairs for combining pair-wise and point-wise approaches. 

To the best of our knowledge, while the CRR approach has reached promising performance in practice~\cite{sculley2010combined}, the CRC formulation has not been seriously studied.
The hinge loss used in CRC allows unifying the point-wise SVM and the pair-wise SVM under the same framework, and
the unification is essential for applying {\it one} active sampling strategy on {\it both} the real pairs and the pseudo-pairs.

In summary, we propose the active sampling scheme for RankSVM (ASRankSVM) and the more general CRC framework (ASCRC), and derive two sampling strategies that correspond to popular strategies in pool-based active learning. The soft version of the sampling strategies helps reducing the computational cost, and allows correcting the sampling bias by adjusting the weights with the inverse probability of being sampled.

\subsection{CRC with Threshold}
\label{sec:threshold}

In Theorem~\ref{upper}, we connect the point-wise SVM \textit{without threshold term}~\eqref{eq:svm} to the pair-wise SVM~\eqref{eq:ranksvm}. The standard SVM for binary classification, however, often come with a threshold term $\theta$ to allow the classification hyperplane to be away from the origin.
That is, the standard SVM solves
\begin{eqnarray} \label{eq:standardsvm}
	\min_{\theta, \w} &&\frac{1}{2}\w^T\w + \nonumber
        C_+\sum_{\x_i \in \mathcal{D}^+}[1-\w^T\x_i+\theta]_+ \\ &&+ C_- \sum_{\x_j \in \mathcal{D}^-}[1+\w^T\x_j-\theta]_+ \ .
\end{eqnarray}
Note that for any given $(\theta, \w)$, 
\[[1-\w^T\x_{ij}]_+ \le \frac{1}{2}\left([1-2\w^T\x_i+2\theta]_+ + [1+2\w^T\x_j-2\theta]_+\right).\]
If we revisit the proof of Theorem~\ref{upper} with the equation above, 
we get a similar theorem that connects the standard SVM to the pair-wise SVM.
\begin{theorem} \label{thresholdupper}
	Let $ C_{ij}= \frac{C}{2} $ be a constant in \eqref{eq:ranksvm}; $C_+ = 2N^- \cdot C$ and $ C_- = 2N^+ \cdot C$ in \eqref{eq:standardsvm}.
	Then, for every $(\theta, \w)$, the objective function of \eqref{eq:ranksvm} is upper-bounded by $\frac{1}{4}$ times the objective function of \eqref{eq:standardsvm}.
\end{theorem}

Given the connection between \eqref{eq:standardsvm} to \eqref{eq:ranksvm}
in Theorem~\ref{thresholdupper}, one may wonder whether the trick of pseudo-pair
works for connecting the two formulations. Consider one positive point $(\x_k, +1)$. Its hinge loss within \eqref{eq:standardsvm} is $[1 - \w^T \x_i + \theta]_+$, which is the same as $\left[1 - \begin{bmatrix}\theta & \w^T\end{bmatrix} \left(\begin{bmatrix}-1\\ \x_i\end{bmatrix} - \mathbf{0}_{n+1}\right)\right]_+$, where $\mathbf{0}_{n+1}$ is a zero vector of length $n+1$. Thus, the positive point $(\x_i, +1)$ can also be viewed as an extended {\it pseudo-pair} that consists of $\left(\begin{bmatrix}-1\\ \x_i\end{bmatrix}, +1\right)$ and $(\mathbf{0}_{n+1}, -1)$, ranked by the extended vector $\begin{bmatrix}\theta \\ \w\end{bmatrix}$. 
We will denote the extended vector $\begin{bmatrix}-1\\ \x_i\end{bmatrix}$
as $\tilde{\x}_i$.
Similarly, a negative point $(\x_j, -1)$ can be viewed as an extended pseudo-pair that consists of $\left(\tilde{\x}_j, -1\right)$ and $(\mathbf{0}_{n+1}, +1)$.

Note that if we consider all the extended vectors $\tilde{\x}_i$, ranking pair-wise extended vectors by $\begin{bmatrix}\theta\\ \w\end{bmatrix}$ means calculating
\begin{eqnarray*}
  \begin{bmatrix}\theta\\ \w\end{bmatrix}^T (\tilde{\x}_i - \tilde{\x}_j) = \begin{bmatrix}\theta\\ \w\end{bmatrix}^T \left(\begin{bmatrix}-1\\ \x_i\end{bmatrix} - \begin{bmatrix}-1\\ \x_j\end{bmatrix}\right) = \w^T (\x_i - \x_j)
\end{eqnarray*}
That is, the hinge loss on extended pairs is exactly the same as the hinge loss on the original pairs. 

Based on the discussions above, if we define extended pairs $\tilde{\x}_{ij}$ and extended pseudo-pairs $\tilde{\x}_{k\ell}$ based on the extended vectors $\tilde{\x}_i$ and $\mathbf{0}_{n+1}$,  we can combine the pair-wise SVM and the standard SVM with threshold term to design a variant of the CRC formulation. In the variant, we take
one common trick to include  We use one trick (as taken by LIBLINEAR~\cite{liblinear}) that includes $\theta$ in the regularization term to allow simpler design of optimization routines. That is, we solve
%% \begin{eqnarray}
%% 	\min_{\theta, \w} && \frac{1}{2}\w^T\w + \gamma \sum_{\x_{ij} \in \mathcal{D}_{pair}^+} C_{crc}^{(ij)} [1-\begin{bmatrix}\theta & \w^T\end{bmatrix}\tilde{\x}_{ij}]_+  \nonumber\\
%% &&+(1-\gamma) \sum_{\x_{k\ell} \in \mathcal{D}_{pseu}^+} C_{crc}^{(k\ell)} \cdot [1-\begin{bmatrix}\theta & \w^T\end{bmatrix}\tilde{\x}_{k\ell}]_+
%%  \ ,  \label{eq:optthreshold}
%% \end{eqnarray}
%% Note, however, that $\theta$ in \eqref{eq:optthreshold} is not included in the regularization term $\frac{1}{2}\w^T \w$. 
%% several existing works, including LIBLINEAR~\cite{liblinear}. include $\theta$ in the regularization term to allow simpler design of optimization routines. We adopt the same idea and consider
\begin{eqnarray}
	\min_{\theta, \w} && \frac{1}{2}(\theta^T \theta + \w^T\w) \nonumber\\
&&+ 
\gamma \sum_{\x_{ij} \in \mathcal{D}_{pair}^+} C_{crc}^{(ij)} [1-\begin{bmatrix}\theta & \w^T\end{bmatrix}\x_{ij}]_+  \nonumber\\
&&+(1-\gamma) \sum_{\x_{k\ell} \in \mathcal{D}_{pseu}^+} C_{crc}^{(k\ell)} \cdot [1-\begin{bmatrix}\theta & \w^T\end{bmatrix}\tilde{\x}_{k\ell}]_+
 \   \label{eq:optthreshold2}
\end{eqnarray}
in our study. We call this formulation \eqref{eq:optthreshold2} CRC-threshold, which is simply equivalent to the original CRC formulation~\eqref{eq:opt} applied to the extended vectors. The equivalence allows us to easily test whether the flexibility of $\theta$ (through using the extended vectors $\tilde{\x}_i$) can improve the original CRC formulation.

\section{Experiments} \label{sec:exp}

In this section, we study the performance and efficiency of our proposed {ASCRC} algorithm on real-world large-scale data sets.
We compare ASCRC with random-CRC, which does random sampling under the CRC framework.
In addition, we compare ASCRC with three other state-of-the-art algorithms for large-scale bipartite ranking: 
the point-wise weighted linear SVM~\eqref{eq:svm} (WSVM), an efficient implementation~\cite{joachims2006training} of the pair-wise linear RankSVM \eqref{eq:ranksvm} (ERankSVM), and the combined ranking and regression (CRR)~\cite{sculley2010combined} algorithm for general ranking.

We use 14 data sets from the LIBSVM Tools%
\footnote{\url{http://www.csie.ntu.edu.tw/~cjlin/libsvmtools/}} and the UCI Repository~\cite{uci} in the experiments. 
Table~\ref{tab:data} shows the statistics of the data sets, which contains more than ten-thousands of instances and more than ten-millions of pairs.
The data sets are definitely too large for a na{\"i}ve implementation of RankSVM~\eqref{eq:ranksvm}.
The data sets marked with $(*)$ are originally multi-class data sets, and we take the sub-problem of ranking the first class
ahead of the other classes as a bipartite ranking task.
For data sets that come with a moderate-sized test set, we report the test AUC. 
Otherwise we perform a $5$-fold cross validation and report the cross-validation AUC.

\begin{table}
	\caption{Data Sets Statistics}
	\label{tab:data}
	\centering
{\small
	\begin{tabular}{| c | HH c | c | c | c |}
		\hline
		Data & Positive & Negative & Points & Pairs & Dimension & AUC \\
		\hline
		MQ-08 & 931 & 14280 & 15211 & 26589360 & 10042 & CV \\
		\hline
		letter* & 789 & 19211 & 20000 & 30314958 & 16 & CV\\
		\hline
		yahoo2  & 638 & 34177 & 34815 & 43609852 & 1966 & test \\
		\hline
		yahoo2\_v2 & 1899 & 32916 & 34815 & 125014968 & 1966 & test\\
		\hline
		protein* & 8198 & 9568 & 17766 & 156876928 & 357 & test \\
		\hline
		news20 & 9999 & 9997 & 19996 & 199920006 & 1355191 & CV\\
		\hline
		rcv1 & 10491 & 9751 & 20242 & 204595482 & 47236 & CV\\
		\hline
		a9a & 7841 & 24720 & 32561 & 387659040 & 123 & test\\
		\hline
		bank & 5289 & 39922 & 45211 & 422294916 & 51 & CV\\
		\hline
		ijcnn1 & 4853 & 45137 & 49990 & 438099722 & 22 & CV\\
		\hline
		MQ-07 & 3863 & 65760 & 69623  & 508061760 & 10036 & CV \\
		\hline
		shuttle* & 34108 & 9392 & 43500 & 640684672 & 9 & test\\
		\hline
		mnist* & 5923 & 54077 & 60000 & 640596142 & 780 & test\\
		\hline
		connect* & 44473 & 23084 & 67557 & 2053229464 & 126 & CV\\
		\hline
		acoustic* & 18261 & 60562 & 78823 & 2211845364 & 50 & test\\
		\hline
		real-sim & 22238 & 50071 & 72309 & 2226957796 & 20958 & CV\\
		\hline
		yahoo1 & 8941 & 464193 & 473134 & 8300699226 & 20643 & test \\
		\hline
		yahho1\_v2 & 45111 & 428023 & 473134 & 38617091106 & 20643 & test \\
		\hline
		covtype & 297711 & 283301 & 581012  & 168683648022 & 54 & CV\\
		\hline
		url & 792145 & 1603985 & 2396130 & 2541177395650 & 3231961 & CV\\
		\hline
		%kdda2010 & 7171885 & 1235867 & 8407752 & 17726991998590 & 20216830 & test\\
		%\hline
	\end{tabular}
}
\end{table}

\subsection{Experiment Settings}

Given a budget $B$ on the number of pairs to be used in each algorithm and a global regularization parameter $C$, 
we set the instance weights for each algorithm to fairly maintain the numerical scale between the regularization term and the loss terms.
The global regularization parameter $C$ is fixed to $0.1$ in all the experiments. 
In particular, the setting below ensures that the total $C^{(ij)}$, summed over all the pairs (or pseudo-pairs), would be $C \cdot B$ for all the algorithms.

\begin{itemize}	
	\item WSVM: As discussed in Section~\ref{sec:setup}, $C_+$ and $C_-$ shall be inverse-proportional to $N^+$ and $N^-$ to
make the weighted point-wise SVM a reasonable baseline for bipartite ranking. Thus, we set $C_+ = \frac{B}{2N^+} \cdot C$ and $C_-= \frac{B}{2N^-} \cdot C$ in~\eqref{eq:svm}. We solve the weighted SVM by the LIBLINEAR~\cite{liblinear} package with its extension on instance weights.
	
	\item ERankSVM: We use the $SVM^{perf}$~\cite{joachims2006training} package to efficiently solve the linear RankSVM~\eqref{eq:ranksvm} with the AUC optimization option. We set the regularization parameter $C_{perf} = \frac{B}{100} \cdot C$ where the $100$ comes from a suggested value of the $SVM^{perf}$ package.

	\item CRR: We use the package sofia-ml~\cite{sculley2010combined} with the {\it sgd-svm} learner type, {\it combined-ranking} loop type and the default number of iterations that SGD takes to solve the problem.
%	      We didn't pay additional effort to search for the optimal maximum number of iterations in order to compare the effort fairly with the others, 
          We set its regularization parameter $\lambda = \frac{1}{B\cdot C}$.
	
  	\item ASCRC (ASRankSVM): We initialize $|\mathcal{L}^*|$ to $b$, and assign $C_{crc}^{(ij)}=\frac{\Gamma |\mathcal{L}^*|}{p_{ij} Z} \cdot C$
in each iteration, where $\Gamma$ equals to either $\gamma$ or $(1-\gamma)$ for either real or pseudo pairs and $Z$ is a normalization constant $\sum_{\x_{ij} \in \mathcal{L}^*} \frac{1}{p_{ij}}$ that prevents $C_{crc}^{(ij)}$ from being too large. We solve the \textsf{linearSVM} within ASCRC by
the LIBLINEAR~\cite{liblinear} package with its extension on instance weights.

	\item random-CRC: random-CRC simply corresponds to ASCRC with $p_{ij} = 1$ for all the pairs. That is, random-CRC samples uniformly within the unlabeled pool.\\ 
\\
To evaluate the average performance of ASCRC and random-CRC algorithms, we average their results over $10$ different initial pools.

\end{itemize}

\subsection{Performance Comparison and Robustness}
\label{sub:robustness}
Next, we examine the necessity of three key designs within the active sampling framework: soft-version versus hard-version, sampling bias correction within soft-version of active sampling, and the choice of soft-version value functions.
We first set $\gamma = 1$ in ASCRC and random-CRC, which makes ASCRC equivalent to ASRankSVM.
We let $b = 100$ and $B = 8000$, which is a relatively small budget out of the millions of pairs.
%We will study the effect of a larger budget in Section~\ref{sub:budget}
%and the effect of using different $\gamma$ in the more general ASCRC in Section~\ref{sub:crc}. 

\begin{figure}[th]
	\centering
	\subfloat[real-sim]{\includegraphics[clip=true,trim= 2cm 7cm 1.5cm 9cm,width=0.23\textwidth]{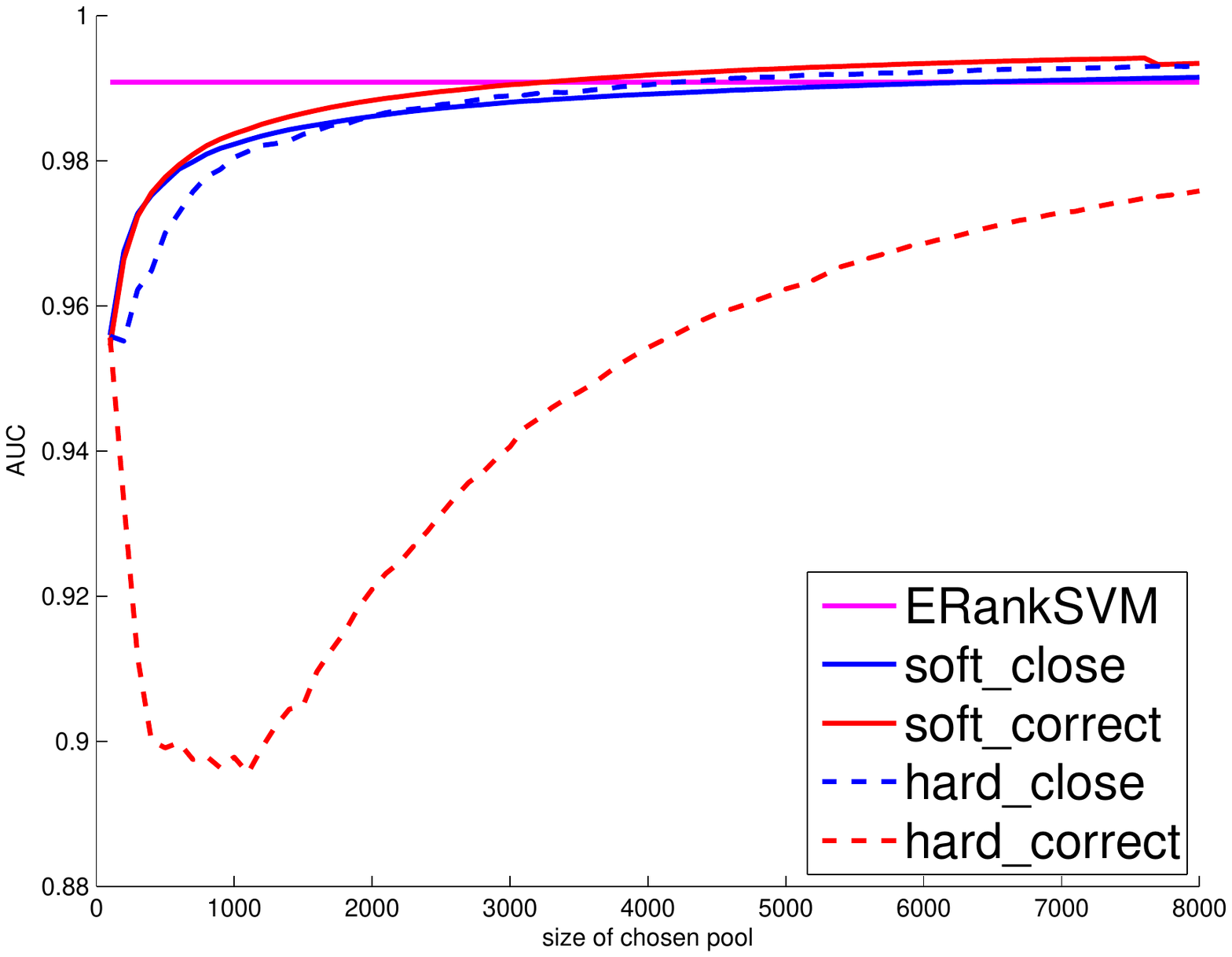}}
	%\subfloat[kdda2010:P=500000,t=1250]{\includegraphics[clip=true,trim= 2cm 7cm 1.5cm 9cm,width=0.23\textwidth]{./figures/auc_kdda2010_P500000_t1250_itr1000_a1.pdf}} \\
	\subfloat[acoustic]{\includegraphics[clip=true,trim= 2cm 7cm 1.5cm 9cm,width=0.23\textwidth]{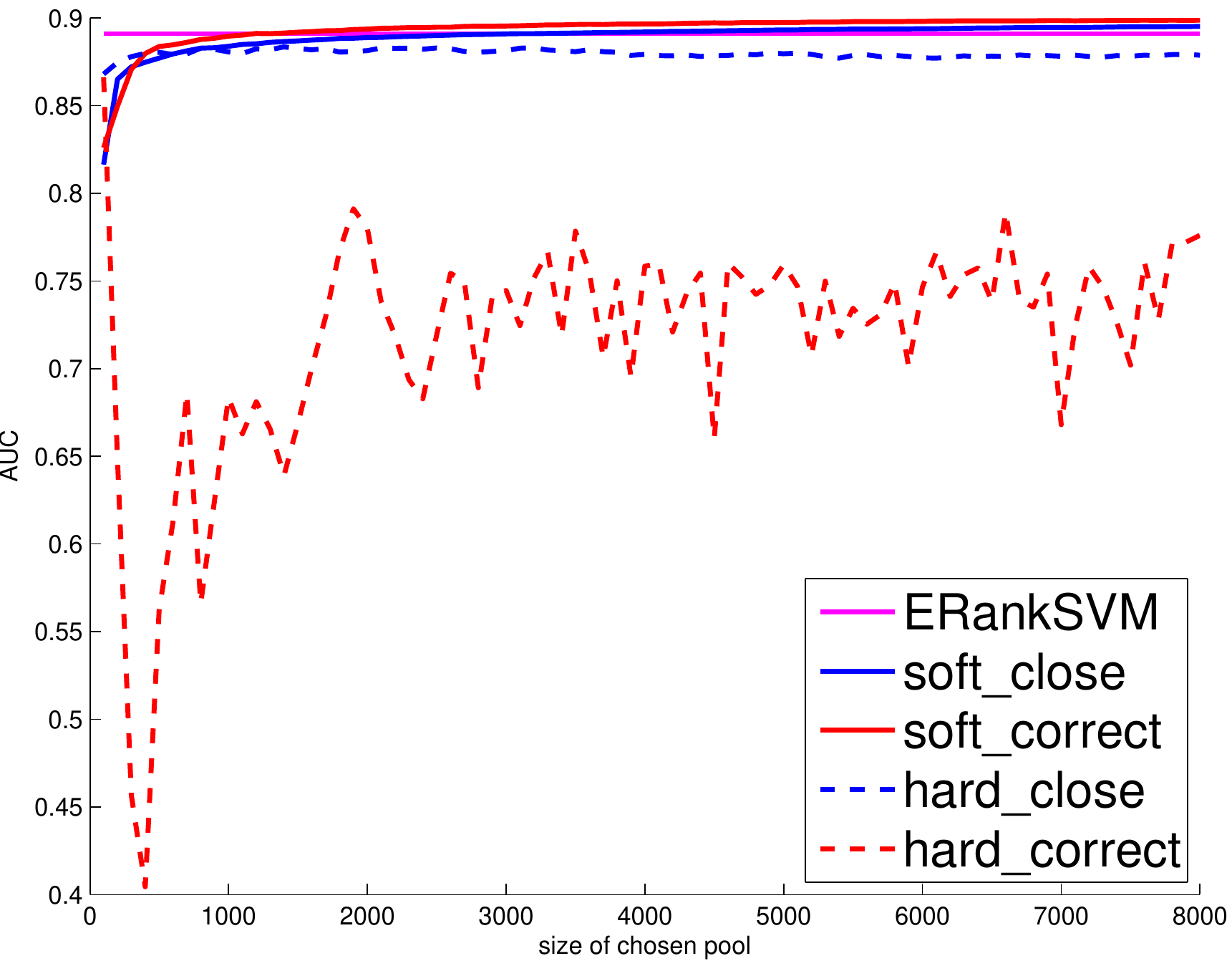}}
	%\subfloat[covtype:P=50000,t=1250]{\includegraphics[clip=true,trim= 2cm 7cm 1.5cm 9cm,width=0.23\textwidth]{./figures/auc_covtype_P50000_t1250_itr1000_a1.pdf}}
	\caption{Performance Curves for Hard Version Samplings}
	\label{fig:hard}
\end{figure}

\begin{table}[t]
	\centering
	\caption{Results Summary under 95\% t-test (win/loss/tie)}
	\label{tab:hard}
	%{\scriptsize
		\begin{tabular}{| c | c | c |}
			\hline
			              & Close & Correct \\
			\hline
			Soft v.s Hard &  10/3/1  &  13/0/1 \\
			\hline
		\end{tabular}
	%}
\end{table}

\subsubsection{Soft-Version versus Hard-Version}

We will discuss the time difference between the soft- and hard-versions of sampling in Table~\ref{tab:time} of Section~\ref{sub:time}. The soft-versions are both coupled with bias correction. Intuitively, the soft version is much faster than the hard version. Here we examine the performance difference between the two versions first.
In Table~\ref{tab:hard}, we compare the soft- and hard-versions of closeness and correctness sampling under the $t$-test of 95\% confidence level.
For closeness sampling, the soft version performs better than the hard version on 9 data sets and ties with 3; 
for correctness sampling, the soft version performs better than the hard version on 12 data sets and ties with 1.
The results justify that the soft version is a better choice than the hard-version in terms of AUC performance. 

Fig.~\ref{fig:hard} further show how the AUC changes as $|\mathcal{L}^*|$ grows for different versions of sampling, along with the baseline ERankSVM algorithm. We see that hard-correctness-sampling always leads to unsatisfactory performance. One possible reason is that hard correctness-sampling can easily suffer from sampling the \textit{noisy} pairs, which come with larger hinge loss. On the other hand, hard-closeness-sampling is competitive to the soft-versions (albeit slower), but appears to be saturating to less satisfactory model in Fig.~\ref{fig:hard}(b). The saturation corresponds to a known problem of uncertainty sampling in active learning because of the restricted view of the non-perfect model used for sampling~\cite{hintsvm}. The soft-version, on the other hand, has some probability of escaping from the restricted view, and hence enjoys better performance.

\begin{table*}[t]
	\centering
	\caption{The Benefit of Bias Correction at $|\mathcal{L}^*|=8000, b=100, \gamma=1.0$}
	\label{tab:biased}
	\begin{tabular}{ | c | c | c | c | c | c | c | c |}
		\hline
		& letter & protein & news20 & rcv1 & a9a & bank & ijcnn1 \\
		\hline
		Close & -1.0540e-03 & -1.0400e-03 & -3.2210e-03 & -5.4200e-04 & -3.3000e-05 & -2.2160e-03 & -2.9800e-04 \\
		\hline
		Correct & -2.4430e-03 & 9.3760e-03 & -2.9550e-03 & -9.6000e-05 & 2.1680e-03 & 4.5100e-04 & 7.3300e-04 \\
		\hline
	\end{tabular}
	
	\begin{tabular}{ | c | c | c | c | c | c | c | c | c |}
		\hline
		& shuttle & mnist & connect & acoustic & real-sim & yahoo1 & covtype & url \\
		\hline
		Close & -9.3600e-04 & 2.7600e-04 & -7.1700e-04 & -2.6960e-03 & -8.2700e-04 & -1.0650e-03 & -1.7220e-03 & 3.8000e-05\\
		\hline
		Correct & 3.0700e-04 & 2.4700e-04 & 3.2030e-03 & 4.9000e-05 & -1.8710e-03 & 5.4300e-04  & 2.8380e-03 & -2.7900e-04\\
		\hline
	\end{tabular}

\end{table*}

\subsubsection{Bias Correction for Soft Version Sampling}

Next, we show the AUC difference between doing bias correction (see Section~\ref{sec:sampling}) and not doing so for soft-version sampling in Table~\ref{tab:biased}. A positive difference indicates that doing bias correction leads to better performance. First of all, we see that the difference of the bias correction
is relatively small. For soft-close sampling, performing bias correction is slightly worse in 12 data sets; for soft-correct sampling,
performing bias correction is slightly better in 9 data sets.
Note correctness sampling is inheritedly more biased towards the noisy pairs as discussed during hard-version sampling. Thus, performing bias correction can be necessary and helpful in ensuring the stability, as justified by the better performance in those 9 data sets.

\subsubsection{Value Functions for Soft Version Sampling}

We show how the AUC changes as $|\mathcal{L}^*|$ grows throughout the active sampling steps of ASRankSVM in Fig.~\ref{fig:auc}. 
For WSVM, ERankSVM and CRR, we plot a horizontal line on the AUC achieved when using the whole training set. 
We also list the final AUC with the standard deviation of all the algorithms in Table~\ref{tab:auc}.

\begin{table*}
	\caption{AUC (mean$\pm$std) at $|\mathcal{L}^*|=8000, b=100, \gamma=1.0$}
	\label{tab:auc}
	\centering
%{\small
	\begin{tabular}{| c | l | l | l | l | l | l |}
		\hline
		     &      &          &     &        & \multicolumn{2}{c|}{ASRankSVM} \\
		\cline{6-7}
		%\hline
		Data & WSVM & ERankSVM & CRR & Random & Soft-Close & Soft-Correct \\
		\hline
		MQ-08   & .8174 (*)     & \textbf{.864} & .8596    & .8587 $\pm$ .0041      & .8621 $\pm$ .0029     & .8599 $\pm$ .0032 \\
		\hline
		letter & .9808 & .9877 & .9874 & .9883 $\pm$ .0003 & \textbf{.9883 } $\pm$ .0002 & .9874 $\pm$ .0123 \\
		\hline
		yahoo2 & .9544       & \textbf{.9873} & .7885(*) & .9525 $\pm$ .0015      & .9561 $\pm$ .0007     & .9481 $\pm$ .0047 \\
		\hline
		yahoo2\_v2 & .9014   & \textbf{.9513} & .7134(*) & .8963 $\pm$ .0015 (*)  & .8999 $\pm$ .0008     & .8974 $\pm$ .0011 \\
		\hline
		protein & \textbf{.8329} & .8302 & .8306 & .8229 $\pm$ .0031 & .8240 $\pm$ .0016 & .8233 $\pm$ .0028 \\
		\hline
		news20 & .9379 (*) & .9753 (*) & .9743 (*) & .9828 $\pm$ .0008 (*) & .9836 $\pm$ .0006 (*) & \textbf{.9903} $\pm$ .0003\\
		\hline
		rcv1 & .9876 (*) & .9916 (*) & .9755 (*) & .9920 $\pm$ .0004 (*) & .9923 $\pm$ .0003 (*) & \textbf{.9944} $\pm$ .0002\\
		\hline
		a9a & .9008 & \textbf{.9047} & .8999 (*) & .9003 $\pm$ .0006 & .9012 $\pm$ .0004 & .9007 $\pm$ .0006\\
		\hline
		bank & .8932 (*) & .9023 (*) & .8972 (*) & .9051 $\pm$ .0010 (*) & .9057 $\pm$ .0011 (*) & \textbf{.9083} $\pm$ .0007\\
		\hline
		ijcnn1 & .9335 (*) & .9343 (*) & .9336 (*) & .9342 $\pm$ .0004 (*) & .9345 $\pm$ .0006 & \textbf{.9348} $\pm$ .0003\\
		\hline
		MQ-07   & .8238 (*)  & .8951 (*)      & .8894(*) & .8886 $\pm$ .0028 (*)  & .8922 $\pm$ .0024 (*) & \textbf{.897 $\pm$ .0017} \\
		\hline
		shuttle & .9873 (*) & .9876 (*) & .9888 (*) & .9894 $\pm$ .0001 (*) & .9896 $\pm$ .0001 (*) & \textbf{.9907} $\pm$ .0000\\
		\hline
		mnist & \textbf{.9985} & .9983 & .9973 (*) & .9967 $\pm$ .0004 (*) & .9979 $\pm$ .0001 & .9976 $\pm$ .0002 \\
		\hline
		connect & .8603 & \textbf{.8613} & .8532 (*) & .8594 $\pm$ .0008 (*) & .8604 $\pm$ .0007 & .8603 $\pm$ .0009 \\
		\hline
		acoustic & .8881 (*) & .8911 (*) & .8931 (*) & .8952 $\pm$ .0005 (*) & .8952 $\pm$ .0005 (*) & \textbf{.8988} $\pm$ .0004 \\
		\hline
		real-sim & .9861 (*) & .9908 & .9907 & .9908 $\pm$ .0003 & .9915 $\pm$ .0002 & \textbf{.9934} $\pm$ .0064 \\
		\hline
		yahoo1  &  .955 (*)  & \textbf{.9609} & .941 (*) &  .9524 $\pm$ .001 (*)  & .9554 $\pm$ .0007 (*) & .9571 $\pm$ .0006 \\
		\hline
		yahoo1\_v2 & .806    & \textbf{.8234} & .7653(*) & .8014 $\pm$ .0010 (*)  & .8038 $\pm$ .0008 (*) & .8049 $\pm$ .0011 \\
		\hline
		covtype & .8047 (*) & .8228 (*) & .8189 (*) & .8238 $\pm$ .0008 (*) & .8239 $\pm$ .0007 (*) & \textbf{.8249 } $\pm$ .0006 \\
		\hline
		url & .9963 (*) & .9967 (*) & .9956 (*) & .9940 $\pm$ .0003 (*) & .9961 $\pm$ .0001 (*) & \textbf{.9984 } $\pm$ .0015\\ 
		\hline
		Total (ASRankSVM win/loss/tie) &  12/5/3 &  9/9/2 &  16/1/3 &  14/1/5 &  10/4/6 & --\\ 
		\hline
	\end{tabular}
        \begin{center}
        (*) marks the case where ASRankSVM wins significantly under the $t$-test under a 95\% significance level.
        \end{center}
%}

\end{table*}

From Fig.~\ref{fig:auc} and Table~\ref{tab:auc}, we see that soft-correct sampling is generally the best.
We also conduct the right-tail $t$-test for soft-correct against the others, and list the results in Table~\ref{tab:auc} to show whether the improvement of soft-correct sampling is significant. 
%In Table~\ref{tab:ttest}, we list the {\it p}-values of the $t$-test.
%The results are summarized under a 95\% significance level, which means we say soft-correct performs better when the corresponding {\it p}-value is less than 0.05.
%Actually, we can see that most of the {\it p}-values are much smaller than 0.05, which suggests that the improvement is usually significant.

First, we compare soft-correct with random sampling and discover that soft-correct performs better on 10 data sets and ties with 4, which shows that active sampling is working reasonably well.
While comparing soft-close with soft-correct in Table~\ref{tab:auc}, we find that soft-correct outperforms soft-close on 7 data sets and ties with 5. 
Moreover, Fig.~\ref{fig:auc} shows the strong performance of soft-correct comes from the early steps of active sampling. 
Finally, when comparing soft-correct with other algorithms, we discover that soft-correct performs the best on 8 data sets: 
it outperforms ERankSVM on 8 data sets, WSVM on 9 data sets, and CRR on 11 data sets. 
The results demonstrate that even when using a pretty small sampling budget of $8,000$ pairs, ASRankSVM with soft-correct sampling can achieve significant improvement over those state-of-the-art ranking algorithms that use the whole training data set. 
Also, the tiny standard deviation shown in Table~\ref{tab:auc} and the significant results from the $t$-test suggest the robustness of ASRankSVM with soft-correct in general.

%% \input{nSV}

%% Furthermore, we check the support vectors identified by ERankSVM and ASRankSVM, and list them in Table~\ref{tab:nSV}. shows the number of support vectors in the final chosen pool, we can find that most of the pairs selected are the valuable support vectors for both ERankSVM and ASRankSVM. 

Nevertheless, we observe a potential problem of soft-correct sampling from Fig.~\ref{fig:auc}. 
In data sets letter and mnist, the performance of soft-correct increases faster than soft-close in the beginning, but starts dropping in the middle. 
The possible reason, similar to the hard-version sampling, is the existence of noisy pairs that shall better not to be put into the chosen pool. When sampling more pairs, the probability that some noisy pairs (which come with larger hinge loss) are sampled by soft-correct sampling is higher, and can in term lead to degrading of performance. The results suggest a possible future work in combining the benefits of soft-close and soft-correct sampling to be more noise-tolerant.

\begin{table*}
	\centering
	\caption{CPU Time Table under 8000 Pair Budget(Seconds)}
	\label{tab:time}
	{\scriptsize
		\begin{tabular}{|c|c|c|c|c|c|c|c|c|}
			\hline
			Data & Random & Soft-Close & Soft-Correct & Hard-Close & Hard-Correct & WSVM & ERankSVM & CRR \\
			\hline
			MQ-08 & 0.662 & 0.72 & 1.3192 & 14.6544 & 8.6444 & 0.092 & 1.506 & 0.256 \\
			\hline
			letter    & 0.1616 & 0.306 & 3.9204 & 38.5342 & 8.5106 & 0.058 & 0.426 & 0.174\\
			\hline
			yahoo2    & 19.372 & 19.83 & 39.1 & 180.097 & 91.761 & 11.39 & 28.05 & 7.68\\
			\hline
			yahoo2\_v2 & 22.732 & 22.828 & 35.059 & 95.662 & 99.154 & 11.42 & 34.05 & 7.83\\
			\hline
			protein & 3.136 & 2.943 & 4.29 & 12.128 & 8.315 & 0.85 & 3.05 & 0.99\\
			\hline
			news20 & 11.7188 & 11.3012 & 13.2788 & 36.8466 & 25.6112 & 4.02 & 10.454 & 3.938\\
			\hline
			rcv1 & 1.1744 & 1.2636 & 4.3578 & 22.8056 & 7.0516 & 0.366 & 2.186 & 0.82\\
			\hline
			a9a & 0.374 & 0.504 & 1.065 & 30.384 & 19.537 & 0.28 & 1.96 & 0.31\\
			\hline
			bank & 0.3914 & 0.4602 & 0.9288 & 20.8064 & 13.8512 & 0.142 & 3.368 & 0.3\\
			\hline
			ijcnn1 & 0.1914 & 0.3016 & 0.8062 & 21.4004 & 15.864 & 0.138 & 1.768 & 0.324\\
			\hline
			MQ-07 & 1.9086 & 1.9132 & 2.2412 & 44.915 & 37.3924 & 0.512 & 11.618 & 1.42\\
			\hline
			shuttle & 0.146 & 0.288 & 3.02 & 26.307 & 17.577 & 0.18 & 0.62 & 0.36\\
			\hline
			mnist & 1.61 & 2.851 & 56.604 & 205.135 & 50.174 & 4.92 & 14.13 & 2.75\\
			\hline
			connect4 & 0.5468 & 0.6458 & 1.0986 & 23.4094 & 24.2718 & 0.5 & 9.156 & 0.732\\
			\hline
			acoustic & 0.488 & 0.624 & 1.03 & 33.39 & 41.167 & 1.82 & 6.66 & 1.93\\
			\hline
			real-sim & 0.805 & 0.87 & 2.3296 & 63.7404 & 27.8084 & 0.63 & 4.79 & 1.77\\
			\hline
			yahoo1 & 42.306 & 42.284 & 62.08 & 803.822 & 517.441 & 73.29 & 500.01 & 53.26\\
			\hline
			yahoo1\_v2 & 38.886 & 37.779 & 41.868 & 412.988 & 848.064 & 50.07 & 585 & 33.16\\
			\hline
			covtype & 0.294 & 0.3602 & 0.5478 & 160.108 & 1147.99 & 1.282 & 29.206 & 3.122\\
			\hline
			url & 6.2052 & 6.2022 & 32.6044 & 2440.91 & 5707.88 & 23.322 & 536.146 & 56.474\\
			\hline

			%kdda2010 & 23.856 & 22.597 & 23.386 & 93.53 & 3908.84 & 76.57\\
			%\hline

		\end{tabular}
	}
\end{table*}

\subsection{Efficiency Comparison}
\label{sub:time}
First, we study the efficiency of soft active sampling by checking the average number of rejected samples before passing the probability threshold during rejection sampling. 
The number is plotted against the size of $\mathcal{L}^*$ in Fig.~\ref{fig:eff}. 
The soft-close strategy usually needs fewer than $10$ rejected samples, while the soft-correct strategy generally needs an increasing number of rejected samples. 
The reason is that when the ranking performance becomes better throughout the iterations, the probability threshold behind soft-correct
could be pretty small. 
The results suggest that the soft-close strategy is generally efficient, while the soft-correct strategy may be less efficient as $|\mathcal{L}^*|$ grows.

Next, we list the CPU time consumed for all algorithms under $8,000$ pairs budget in Table~\ref{tab:time}, and the data sets are ordered ascendantly by its size.
We can see that WSVM and CRR run fast but give inferior performance; ERankSVM performs better but the training time grows fast as the data size increases.
The result is consistent with the discussion in Section~\ref{sec:intro} that conducting bipartite ranking efficiently and accurately at the same time is challenging.

For ASRankSVM, random runs the fastest, then soft-close, and soft-correct is the slowest. The results reflect the average number of rejected samples discussed above.
In addition, not surprisingly, the soft version samplings are usually much faster then the corresponding hard versions, which validate that the time consuming enumerating or sorting steps do not
fit our goal in terms of efficiency.

More importantly, when comparing soft-correct with ERankSVM, soft-correct runs faster on 7 data sets, which suggests ASRankSVM is as efficient as 
the state-of-the-art ERankSVM on large-scale data sets in general. 
Nevertheless, we can find that the CPU time of soft-correct grows much slower than ERankSVM as data size increases because the time complexity of ASRankSVM mainly depends on the budget $B$ and the step size $b$, not the size of data.

\begin{figure}[t]
	\centering
	\subfloat[protein:B=80000,b=500]{\includegraphics[clip=true,trim= 2cm 7cm 1.5cm 9cm,width=0.23\textwidth]{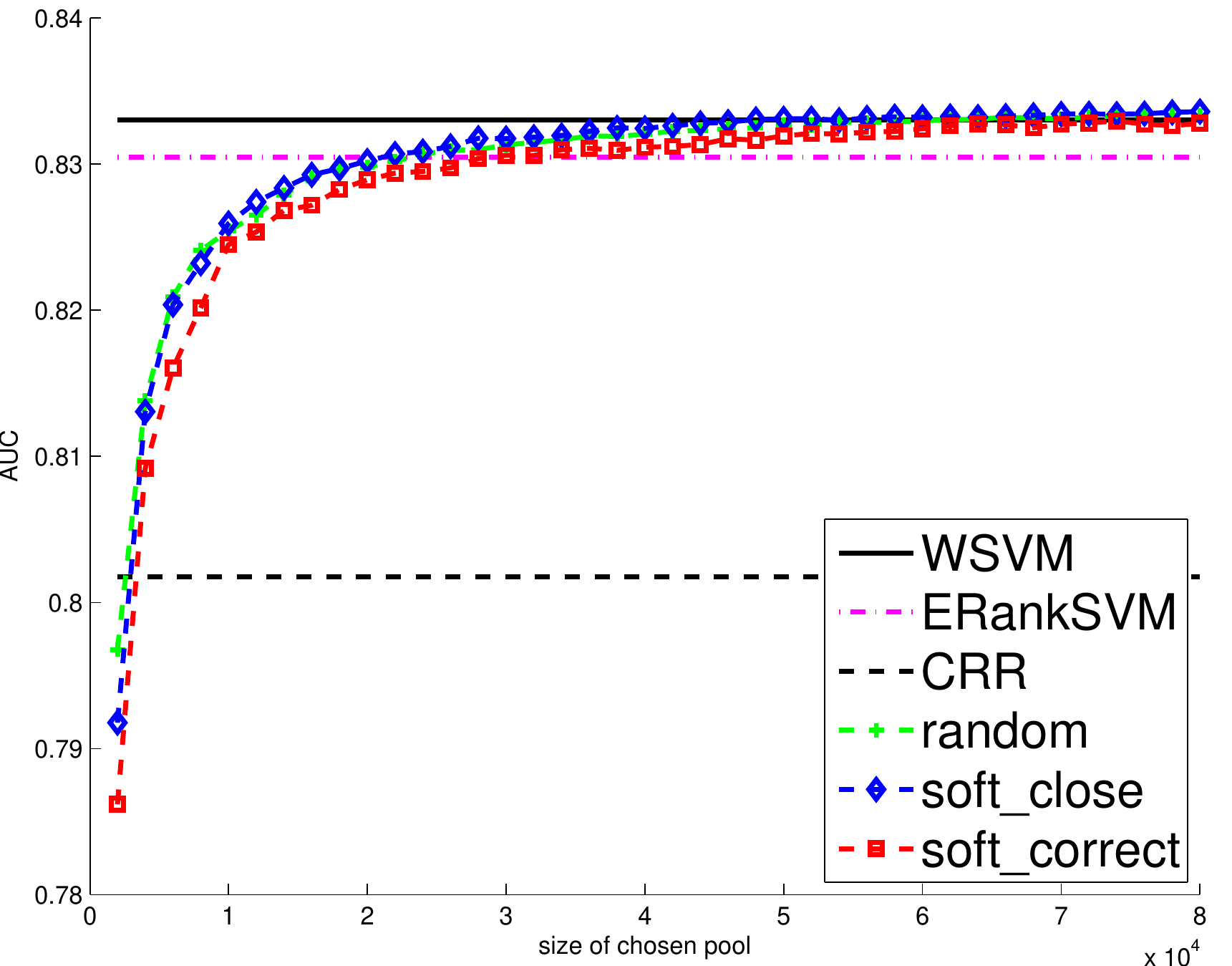}}
	%\subfloat[kdda2010:P=500000,t=1250]{\includegraphics[clip=true,trim= 2cm 7cm 1.5cm 9cm,width=0.23\textwidth]{./figures/auc_kdda2010_P500000_t1250_itr1000_a1.pdf}} \\
	\subfloat[bank:B=45000,b=1125]{\includegraphics[clip=true,trim= 2cm 7cm 1.5cm 9cm,width=0.23\textwidth]{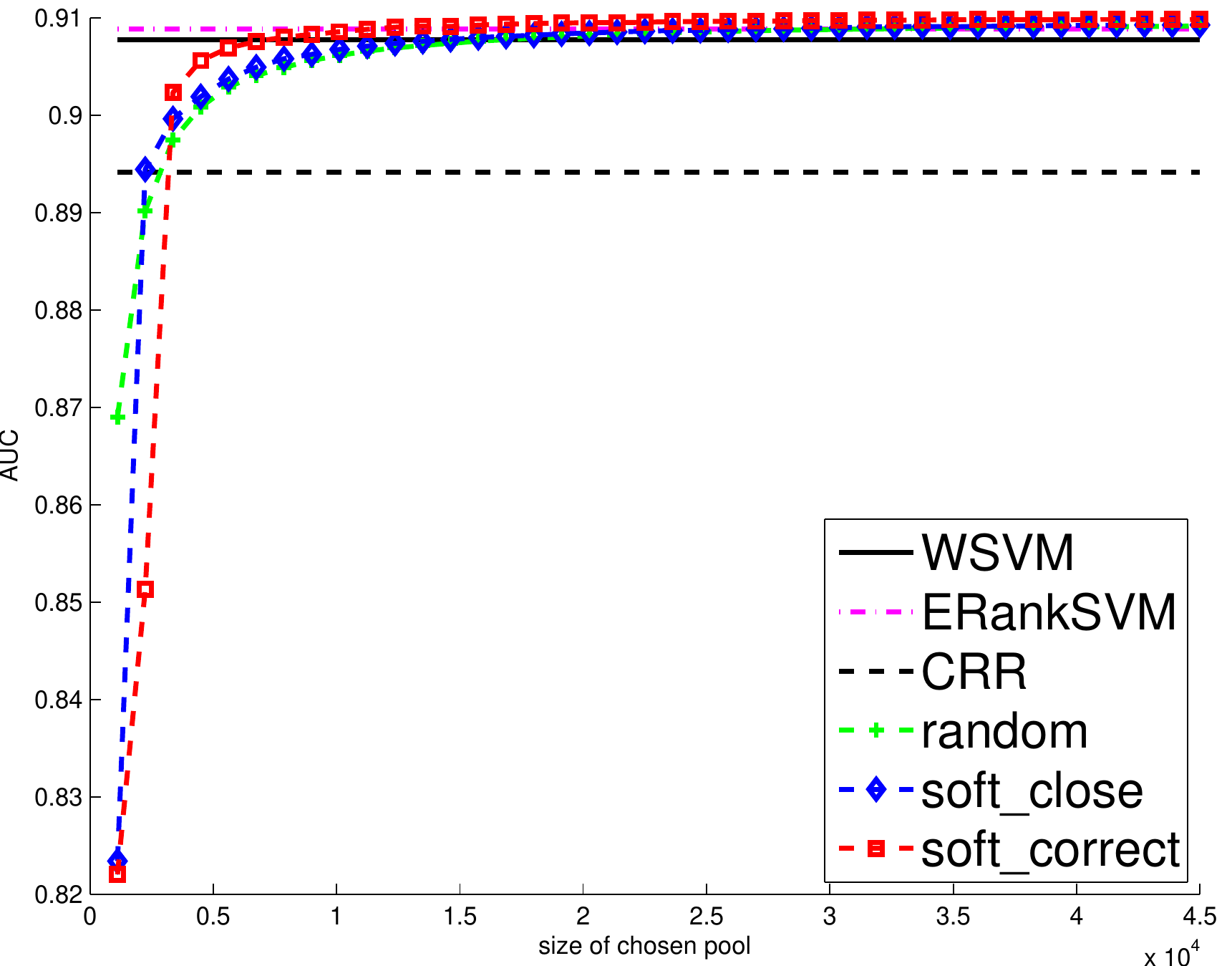}}
	%\subfloat[covtype:P=50000,t=1250]{\includegraphics[clip=true,trim= 2cm 7cm 1.5cm 9cm,width=0.23\textwidth]{./figures/auc_covtype_P50000_t1250_itr1000_a1.pdf}}
	\caption{Performance on Larger Budget under $\gamma=1.0$}
	\label{fig:large_budget}
\end{figure}

\begin{figure}[t]
	\centering
	\subfloat[AUC Curve]{\includegraphics[clip=true,trim= 2cm 7cm 1.5cm 9cm,width=0.23\textwidth]{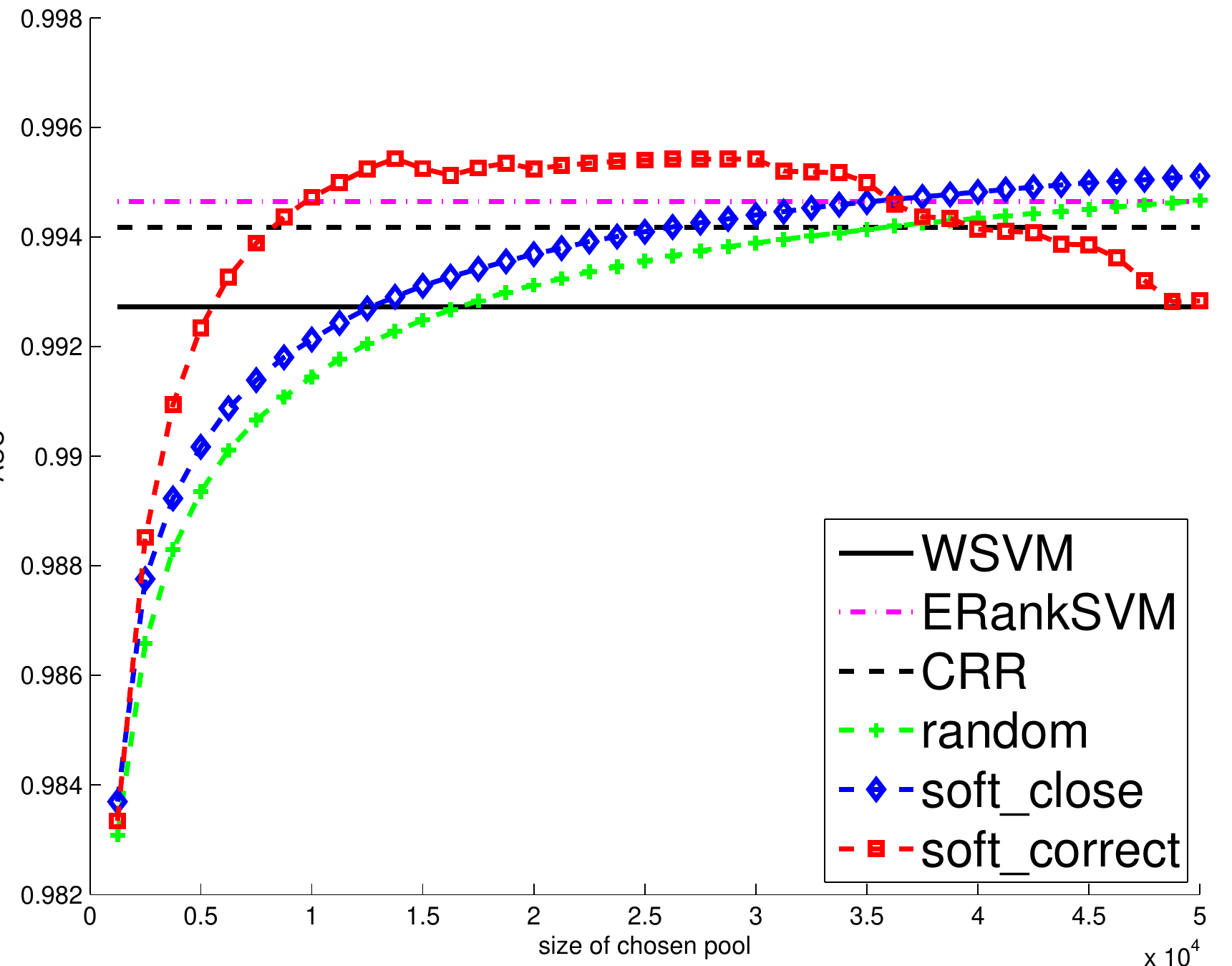}}
%	\subfloat[Variance Curve]{\includegraphics[clip=true,trim= 2cm 7cm 1.5cm 9cm,width=0.233\textwidth]{var_real-sim_P50000_t1250_itr1000_a1.pdf}}
	\subfloat[Number of Pairs Rejected]{\includegraphics[clip=true,trim= 2cm 7cm 1.5cm 9cm,width=0.23\textwidth]{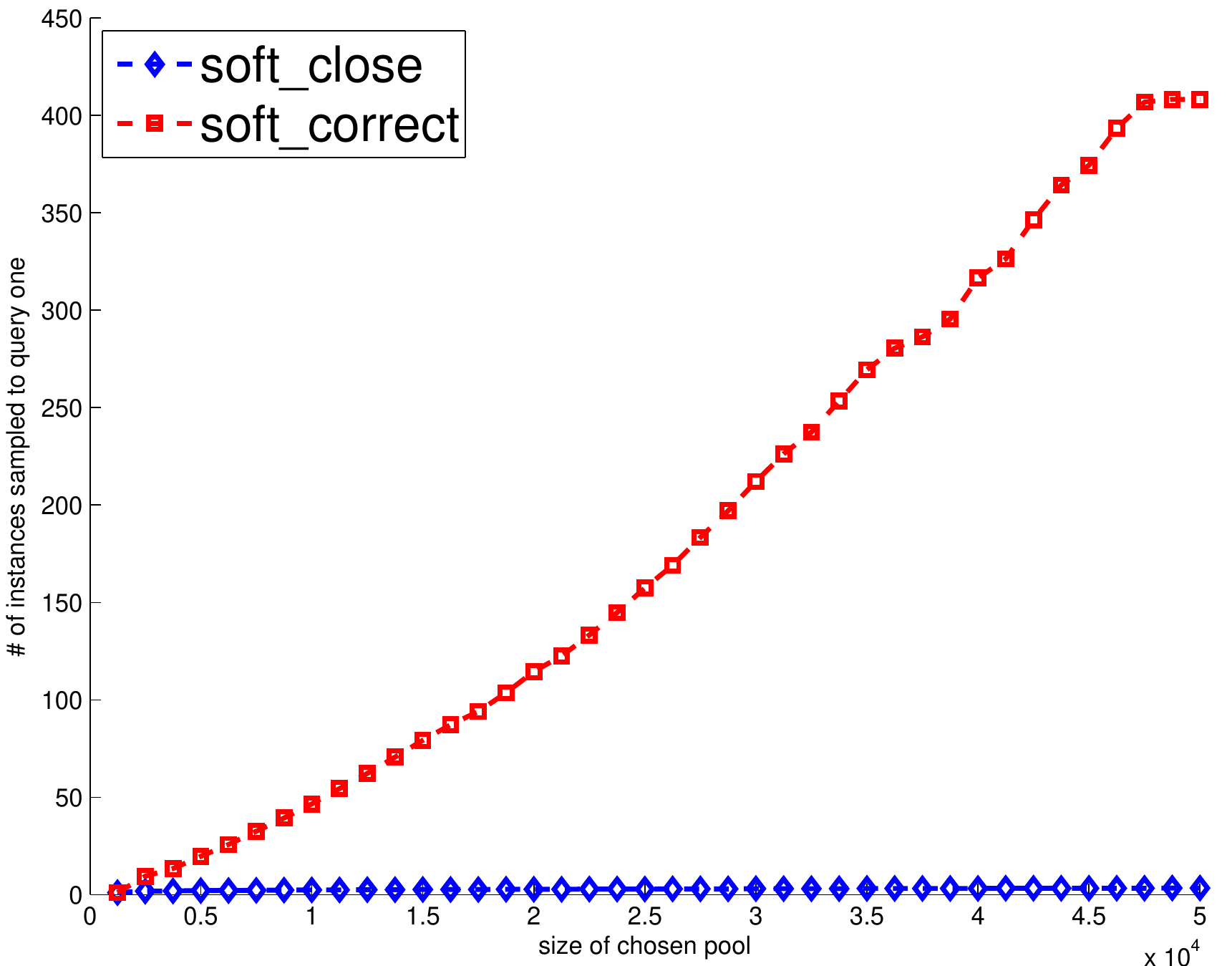}}
	\caption{Potential Problem with Larger Budget}
	\label{fig:large_problem}
\end{figure}

\subsection{The Usefulness of Larger Budget}
\label{sub:budget}
From the previous experiments, we have shown that ASRankSVM with a budget of $8,000$ pairs can perform better than other competitors on large-scale data sets.
Now, we check the performance of ASRankSVM with different budget size. 
In Figure~\ref{fig:large_budget}, we show the AUC curves with much larger budgets on two data sets. 
Then, we find that the performance of ASRankSVM can be improved or maintained as the budget size increases.
For example, in data set protein, we can match the performance of WSVM with around 40,000 pairs and surpass it slightly with around 80,000 pairs.
Nevertheless, in most data sets, we find that the slope of AUC curves become flat around 10,000 pairs, and eventually converge as the budget increases. That is, increasing the budget in ASRankSVM leads to consistent but marginal improvements.

Note that the potential problem of sampling noisy pairs within the soft-correct sampling can be more serious when the budget size increases. 
Fig.~\ref{fig:large_problem} illustrates the problem with the real-sim data when $B=50000$, $b=1250$, and $\gamma=1.0$, where the performance of soft-correct degrades and rejects many more pairs in the latter sampling iterations. On the other hand, soft-close maintains the robustness and the efficiency as the budget increases, and improves the performance consistently throughout the iterations. Thus, if a larger budget is used, soft-close can be a better choice than soft-correct.

\begin{table}[t]
	\centering
	\caption{Optimal $\gamma$ on Different Datasets}
	\label{tab:gamma}
	%\resizebox{\textwidth}{8mm}{
	%\begin{tabular}{ | c | c | c | c | c | c | c | c | }
	%	\hline
	%	method & a9a & acoustic & bank & connect & covtype & ijcnn1 & letter \\
	%	\hline
	%	Soft-Close & uniform & 1 & 1 & 1 & 1 & 1 & uniform \\
	%	\hline
	%	Soft-Correct     & 1 & 1 & 1 & 1 & 1 & 1 & uniform \\
	%	\hline
	%\end{tabular}

	%\begin{tabular}{ | c | c | c | c | c | c | c | c | }
	%		\hline
	%		method  & mnist & news20 & protein & rcv1 & real-sim & shuttle & url\\
	%		\hline
	%		Soft-Close  & 0.1 & uniform & 1 & 1 & 1 &  uniform& 1\\
	%		\hline
	%		Soft-Correct      & 0.6 & 1 & 1 & 1 & 1 & 1 & 1\\
	%		\hline
	%	\end{tabular}
	%}
	\begin{tabular}{|c|c|c|}
		\hline
		Data & Soft-Close & Soft-Correct \\
		\hline
		letter &\textbf{uniform},\textbf{uniform-th}&\textbf{0.9},\textbf{0.9-th}\\
		\hline
		protein &1&0.4-th\\
		\hline
		news20 &\textbf{uniform},\textbf{1},\textbf{uniform-th}&\textbf{uniform},\textbf{1},\textbf{uniform-th}\\
		\hline
		rcv1 &\textbf{uniform},\textbf{1},\textbf{uniform-th}&\textbf{uniform},\textbf{1},\textbf{uniform-th}\\
		\hline
		a9a &uniform,uniform-th&0.6-th\\
		\hline
		bank &\textbf{1}&\textbf{uniform},\textbf{1},\textbf{uniform-th}\\
		\hline
		ijcnn1 &1,uniform-th&\textbf{0.8},\textbf{0.9},\textbf{1},\textbf{uniform-th},\textbf{0.9-th}\\
		\hline
		shuttle &\textbf{0.1-th},\textbf{0.2-th}&\textbf{uniform-th}\\
		\hline
		mnist &\textbf{0.1}&0.6\\
		\hline
		connect &uniform,1,uniform-th&uniform,1,uniform-th\\
		\hline
		acoustic &\textbf{1}&\textbf{uniform},\textbf{1},\textbf{uniform-th}\\
		\hline
		real-sim &\textbf{uniform},\textbf{1},\textbf{uniform-th}&\textbf{uniform},\textbf{uniform-th}\\
		\hline
		yahoo1 & uniform & uniform-th, 0.9-th \\
		\hline
		covtype &\textbf{uniform},\textbf{1}&\textbf{uniform},\textbf{0.9},\textbf{1},\textbf{uniform-th},\textbf{0.9-th}\\
		\hline
		url &uniform,1,uniform-th&\textbf{0.8}\\
		\hline

	\end{tabular}
\end{table}

\subsection{The Usefulness of the CRC Framework}
\label{sub:crc}

Next, we study the necessity of the CRC framework by comparing the performance of soft-closeness and soft-correctness under different choices of $\gamma$.
We report the best $\gamma$ under a 95\% significance level within $\{\mathit{uniform},0.1,0.2,...,1.0\}$, where $\mathit{uniform}$ means balancing the influence
of actual pairs and pseudo-pairs by $\gamma=\frac{|\mathcal{D}_{pair}|}{|\mathcal{D}^*|}$. 
Moreover, we check whether CRC-threshold can be useful.
Table~\ref{tab:gamma} shows the best~$\gamma$ and formulation for each sampling strategy. 
The entries with ``-th'' indicates CRC-threshold. The bold entries indicates that the setting outperforms ERankSVM significantly.
The table suggests three observations.
Firstly, the choice of sampling strategy does not effect the optimal $\gamma$ much, and most data sets have similar optimal $\gamma$ for both soft-closeness and soft-correctness sampling.
Secondly, we find that adding a threshold term for CRC can sometimes reach better performance. 
Last, we see that using $\gamma = 1$ (real pairs only) performs well in most data sets, while a smaller $\gamma$ or $\mathit{uniform}$ can sometimes reach better performance. 
The results justify that the real pairs are more important than the pseudo-pairs, while the latter can sometimes be helpful. 
%When pseudo-pair helps, as shown in Fig.~\ref{fig:gamma_mnist} for the mnist data set, the flexibility of the CRC framework can be useful.

%\subsection{Experiment Result Summary}

In summary, a special case of the proposed ASCRC algorithm that only samples actual pairs (ASRankSVM) works reasonably well for a budget of $8,000$ when coupled with soft-correct sampling. 
The setting significantly outperforms WSVM, ERankSVM, CRR and soft-close on most of the data sets, also the execution time shown the efficiency of soft-correct sampling is comparable with ERankSVM. 
%The cons of the soft-correct sampling is that it becomes increasingly difficult to pass rejection sampling and it is more sensitive to noisy instances than soft-close sampling. 
While $\gamma = 1$ leads to promising performance on most of the data sets, further tuning with a smaller $\gamma$ or adding a threshold term helps in some data sets. 
%Moreover, using budget size around or larger than the training size with soft-close sampling may also help in some data set such as protein. 
%The results validate the usefulness of active sampling (with soft-correct) as well as CRC (with a flexible $\gamma$).

%\captionsetup[subfloat]{aboveskip=0pt,belowskip=0pt,font=small}
%\captionsetup[subfloat]{labelformat=empty}
\begin{figure*}[h]
	\centering
	\subfloat[a9a]{\includegraphics[clip=true,trim= 2cm 7cm 1.5cm 9cm,width=0.23\textwidth]{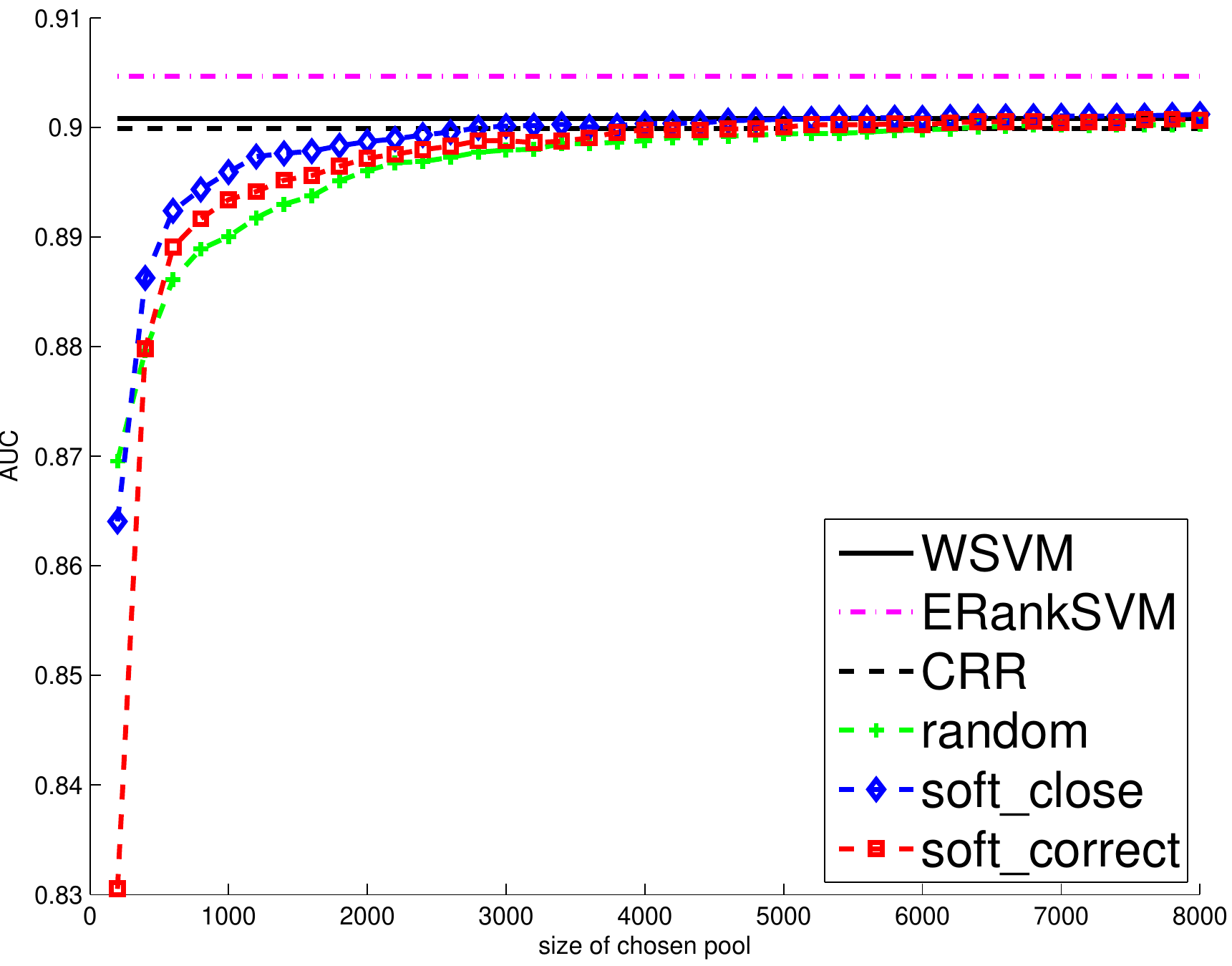}}
	\subfloat[acoustic]{\includegraphics[clip=true,trim= 2cm 7cm 1.5cm 9cm,width=0.23\textwidth]{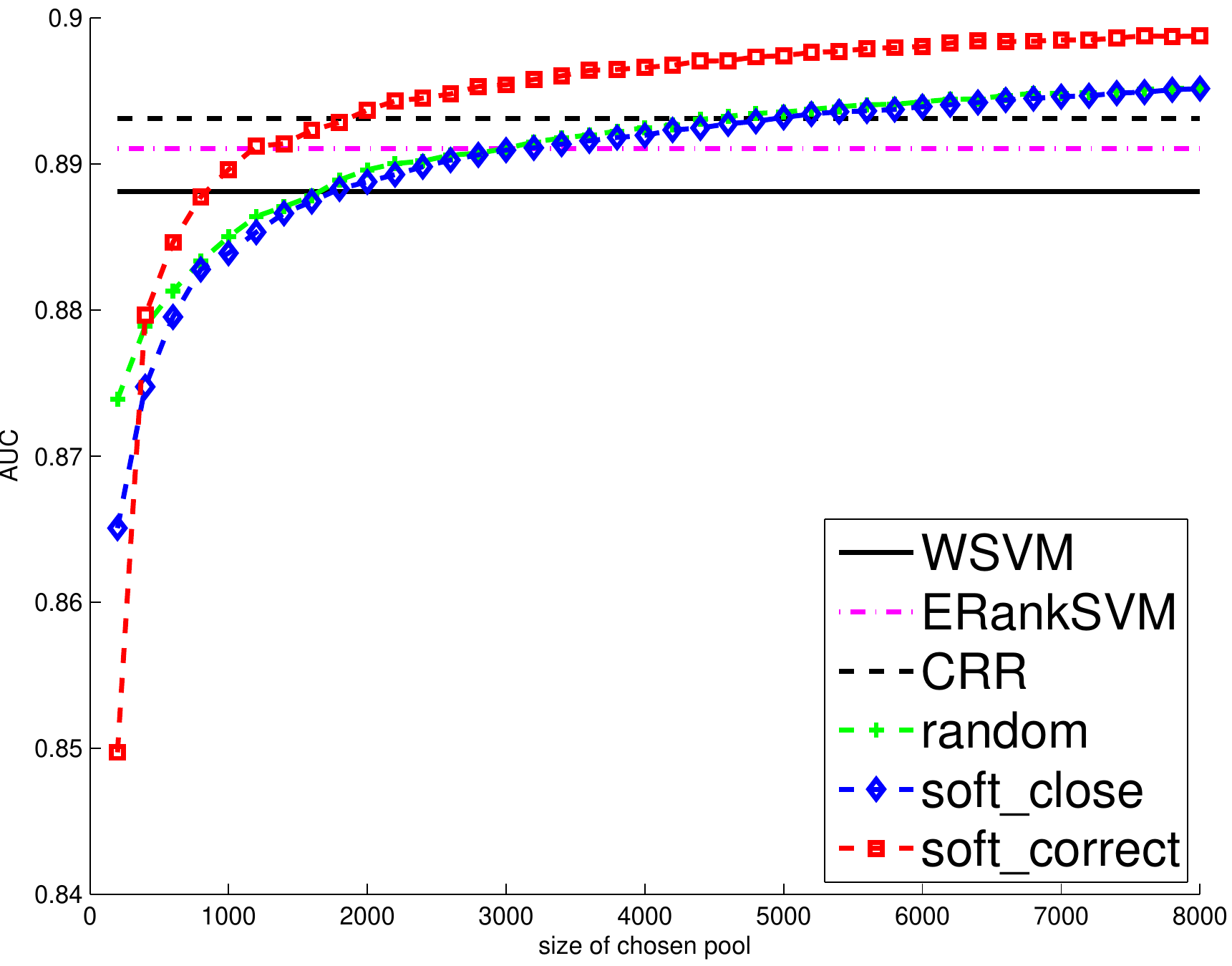}} 
	\subfloat[bank]{\includegraphics[clip=true,trim= 2cm 7cm 1.5cm 9cm,width=0.23\textwidth]{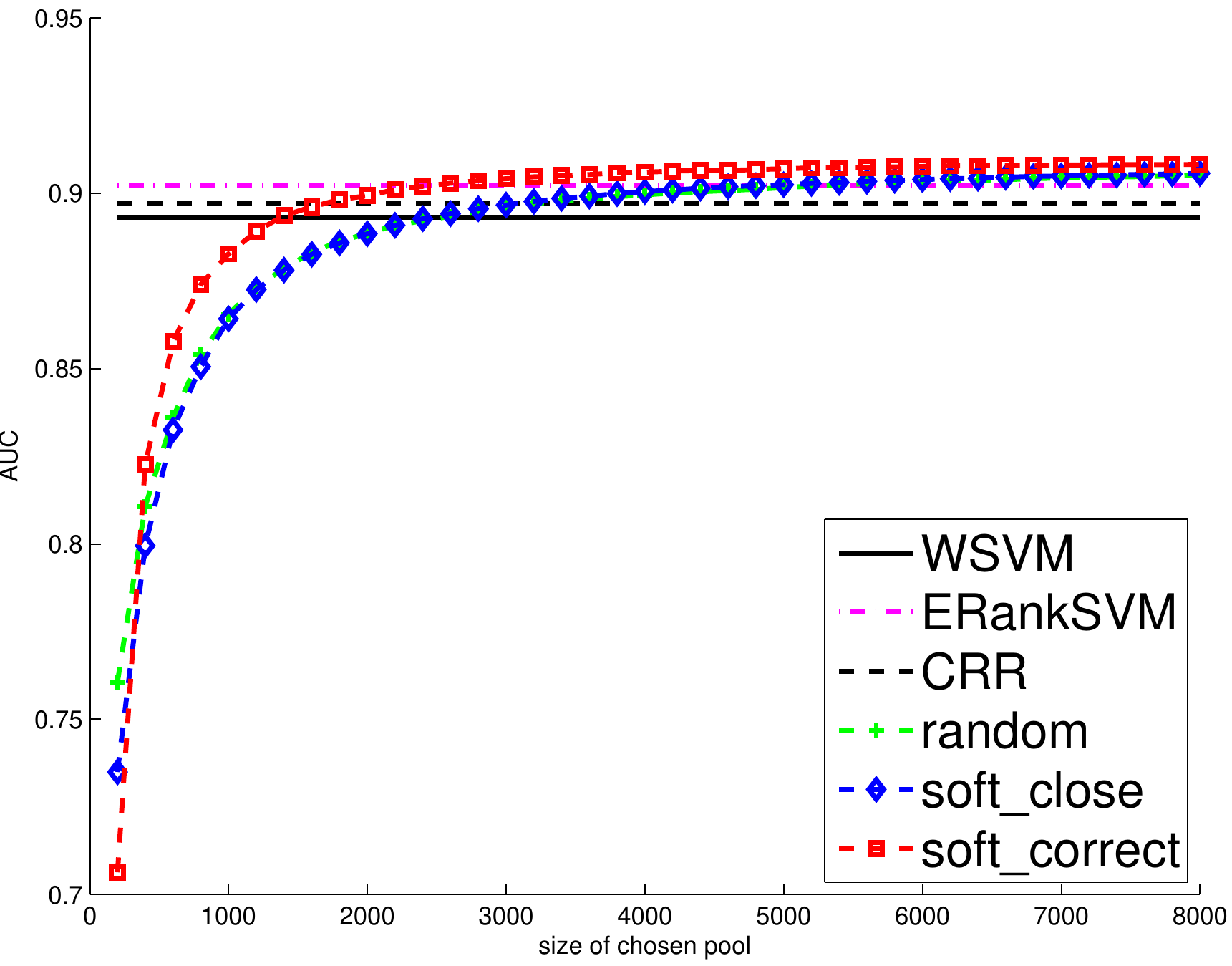}}	
	\subfloat[connect]{\includegraphics[clip=true,trim= 2cm 7cm 1.5cm 9cm,width=0.23\textwidth]{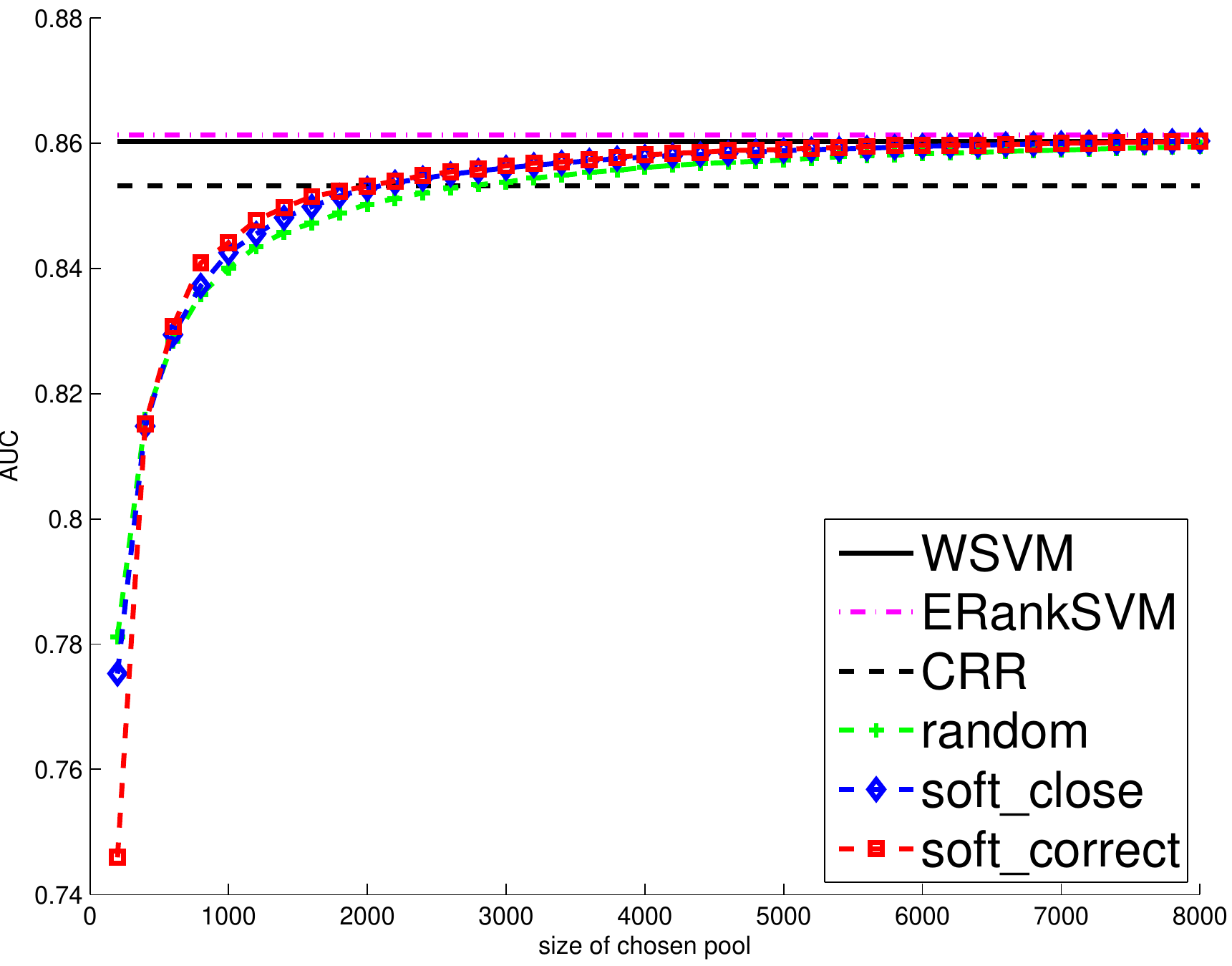}}
\\
	\subfloat[covtype]{\includegraphics[clip=true,trim= 2cm 7cm 1.5cm 9cm,width=0.23\textwidth]{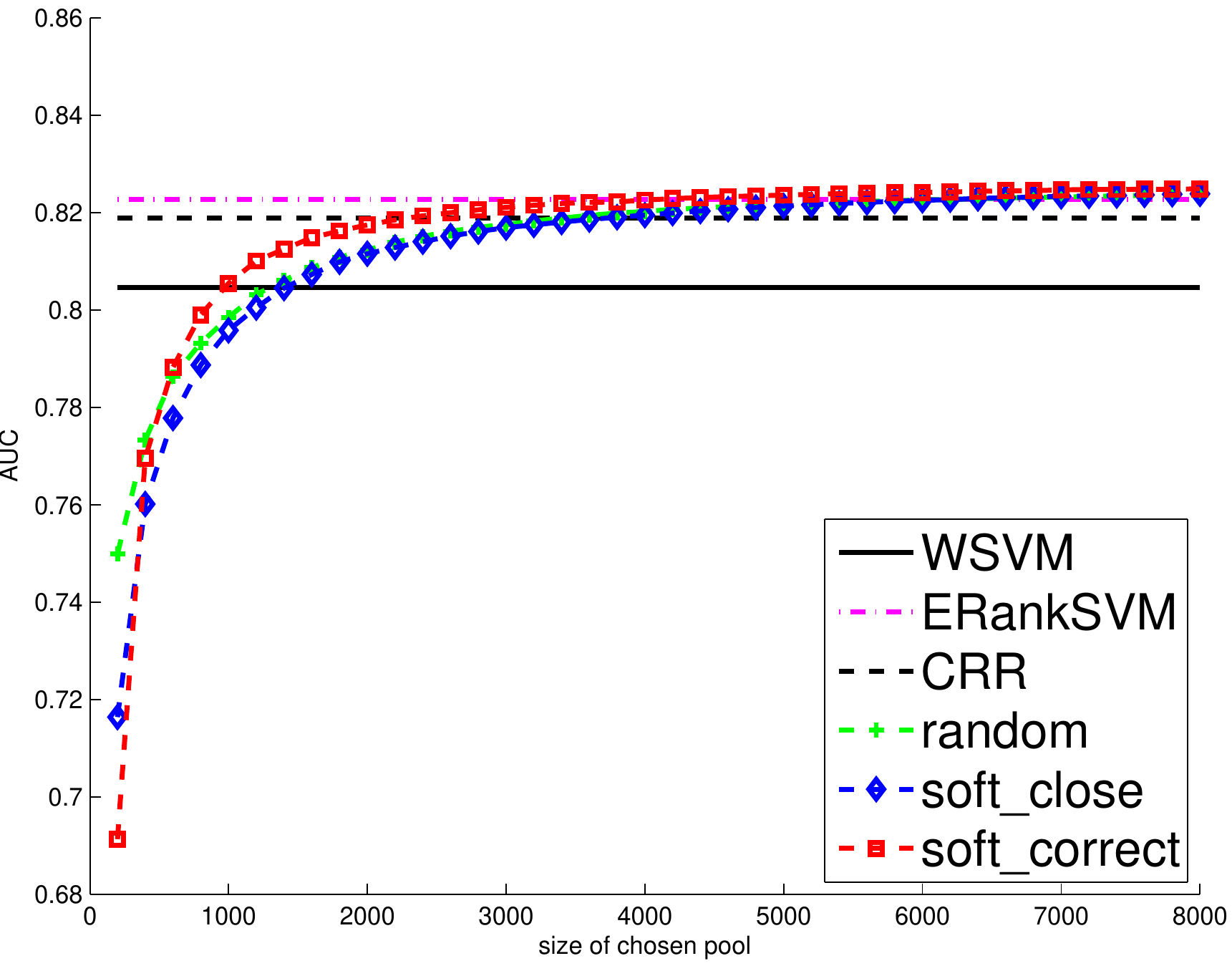}}
	\subfloat[ijcnn1]{\includegraphics[clip=true,trim= 2cm 7cm 1.5cm 9cm,width=0.23\textwidth]{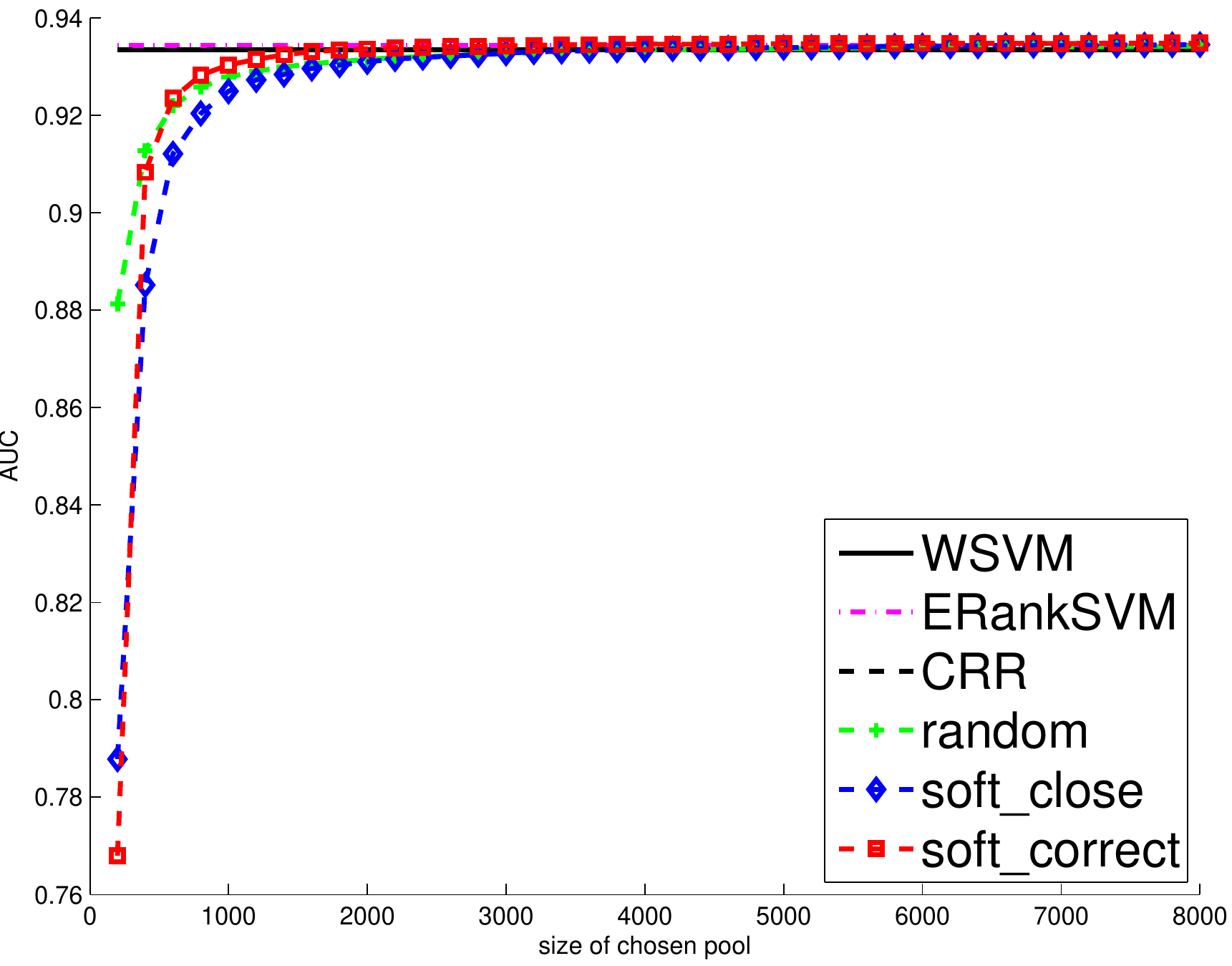}}
	%\subfloat[kdda2010]{\includegraphics[clip=true,trim= 2cm 7cm 1.5cm 9cm,width=0.23\textwidth]{auc_kdda2010_P8000_t100_itr1000_a1.pdf}}
	\subfloat[letter]{\includegraphics[clip=true,trim= 2cm 7cm 1.5cm 9cm,width=0.23\textwidth]{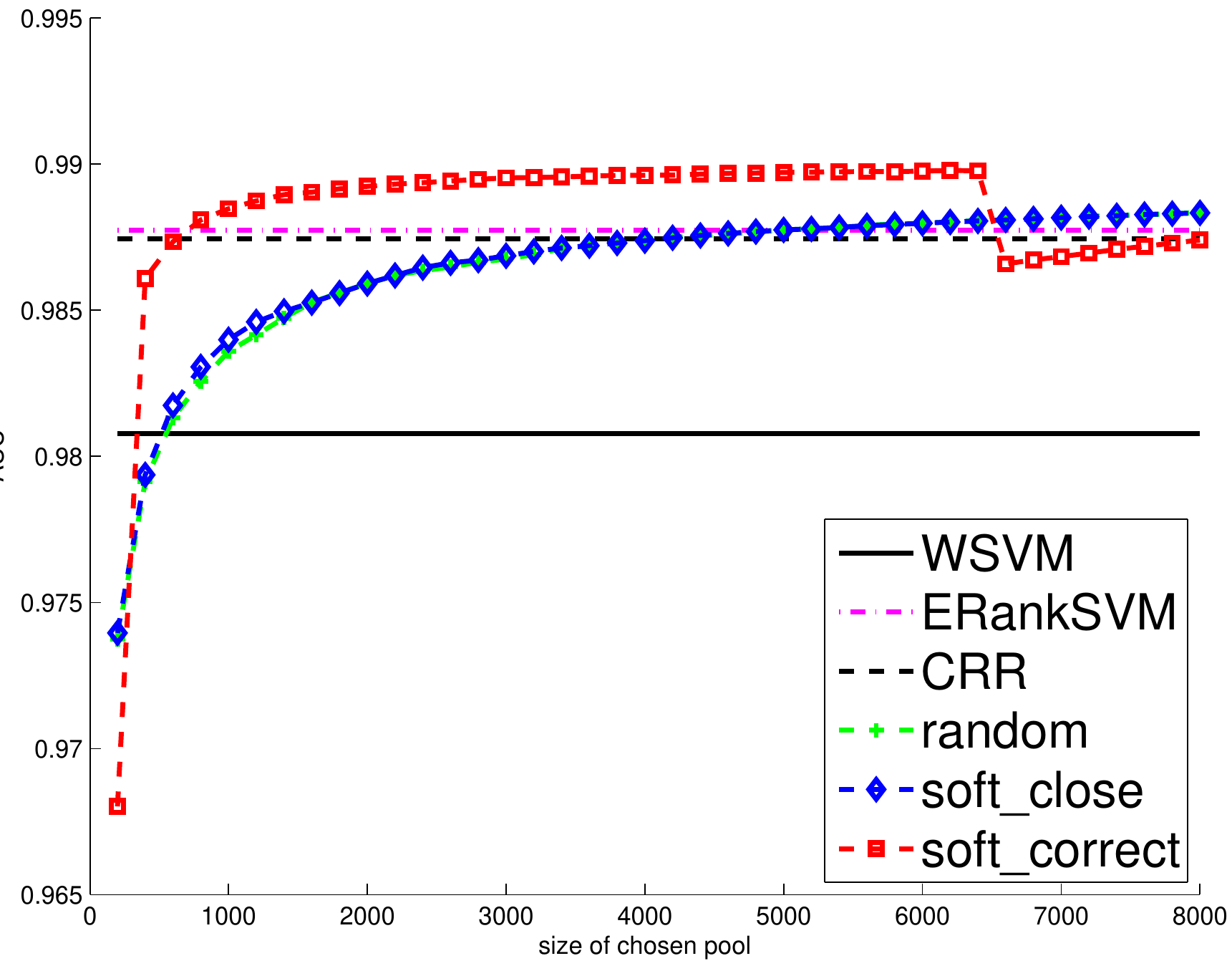}}
	\subfloat[mnist]{\includegraphics[clip=true,trim= 1.5cm 7cm 1.5cm 9cm,width=0.23\textwidth]{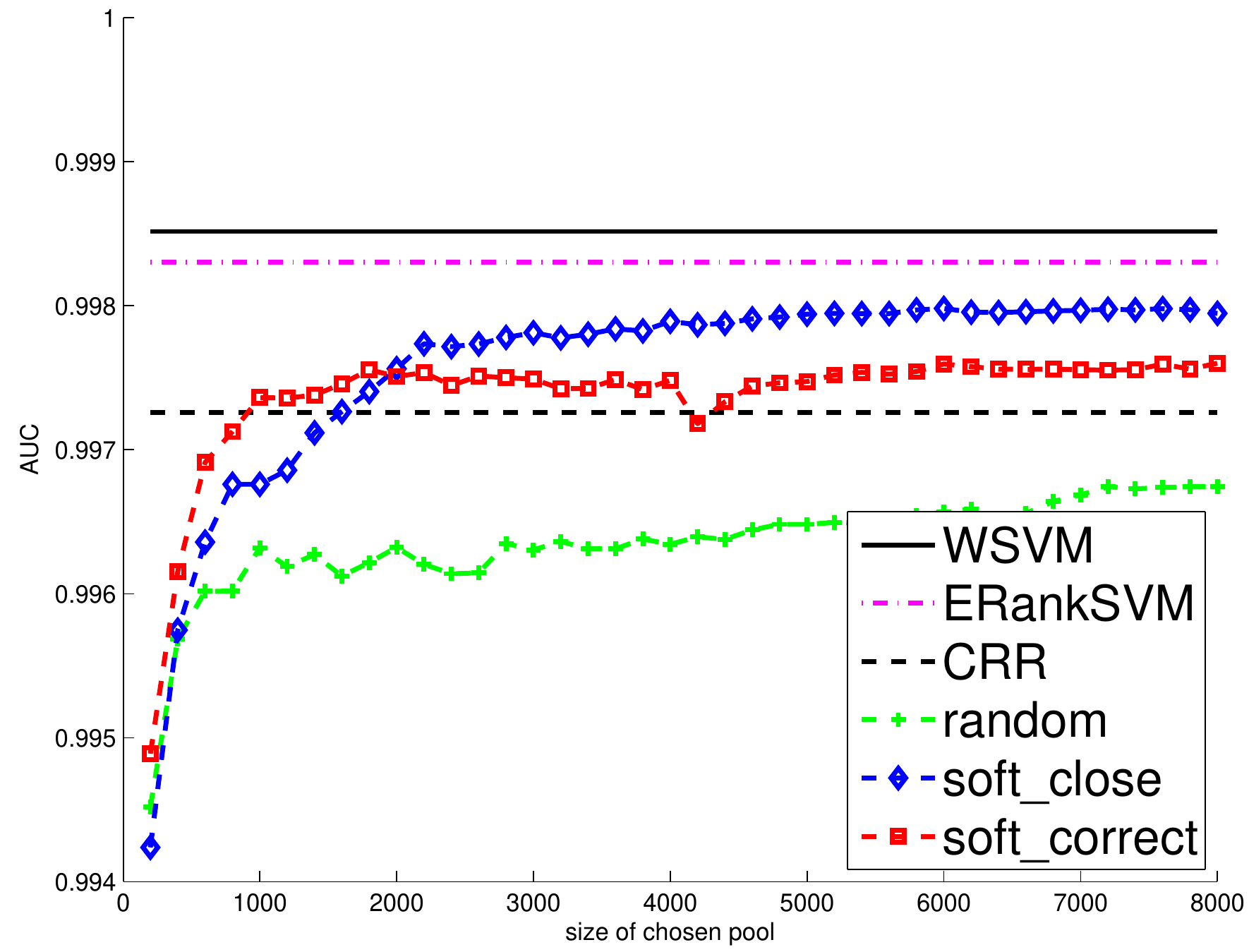}}
\\
	\subfloat[news20]{\includegraphics[clip=true,trim= 2cm 7cm 1.5cm 9cm,width=0.23\textwidth]{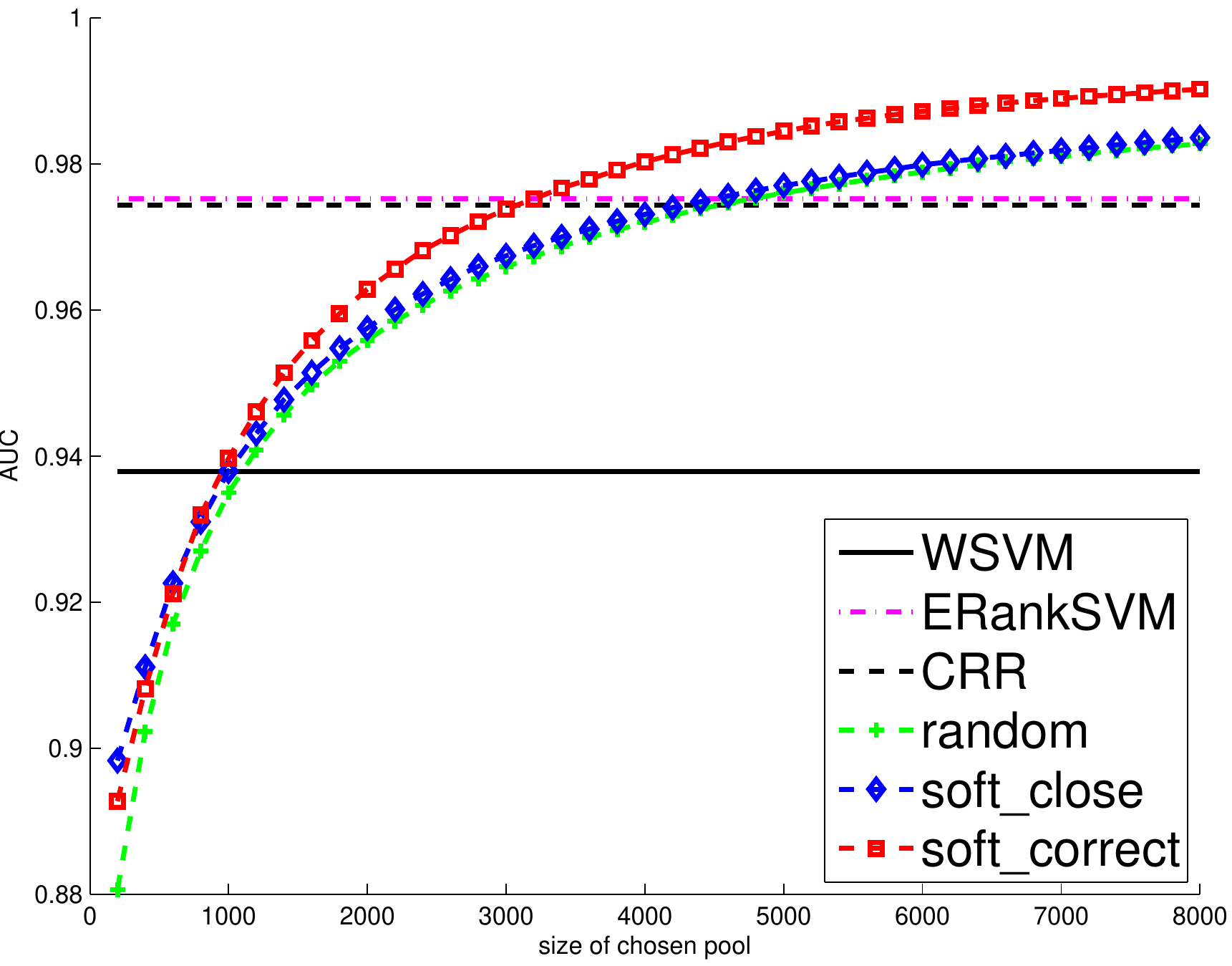}}
	\subfloat[protein]{\includegraphics[clip=true,trim= 2cm 7cm 1.5cm 9cm,width=0.23\textwidth]{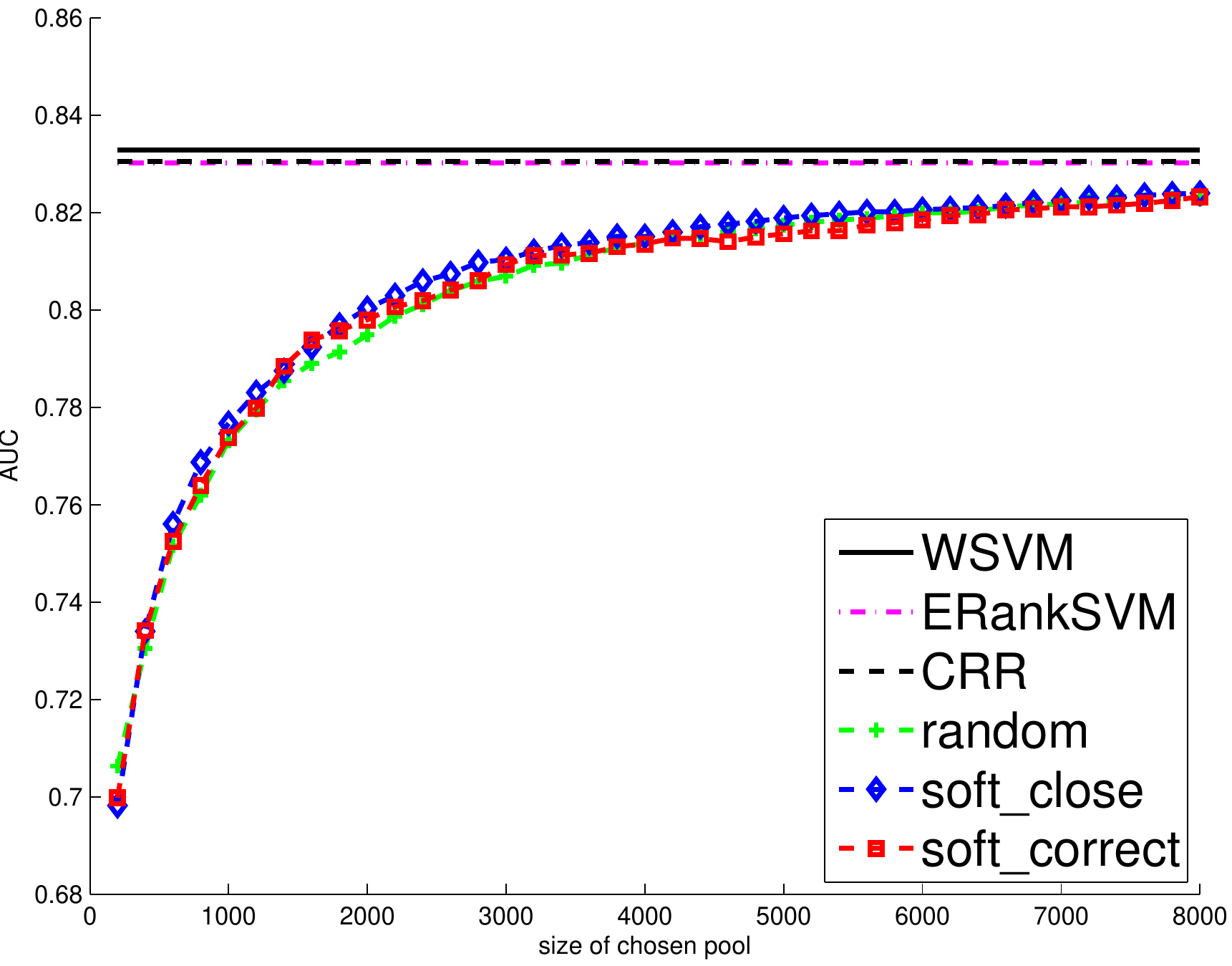}}
	\subfloat[rcv1]{\includegraphics[clip=true,trim= 2cm 7cm 1.5cm 9cm,width=0.23\textwidth]{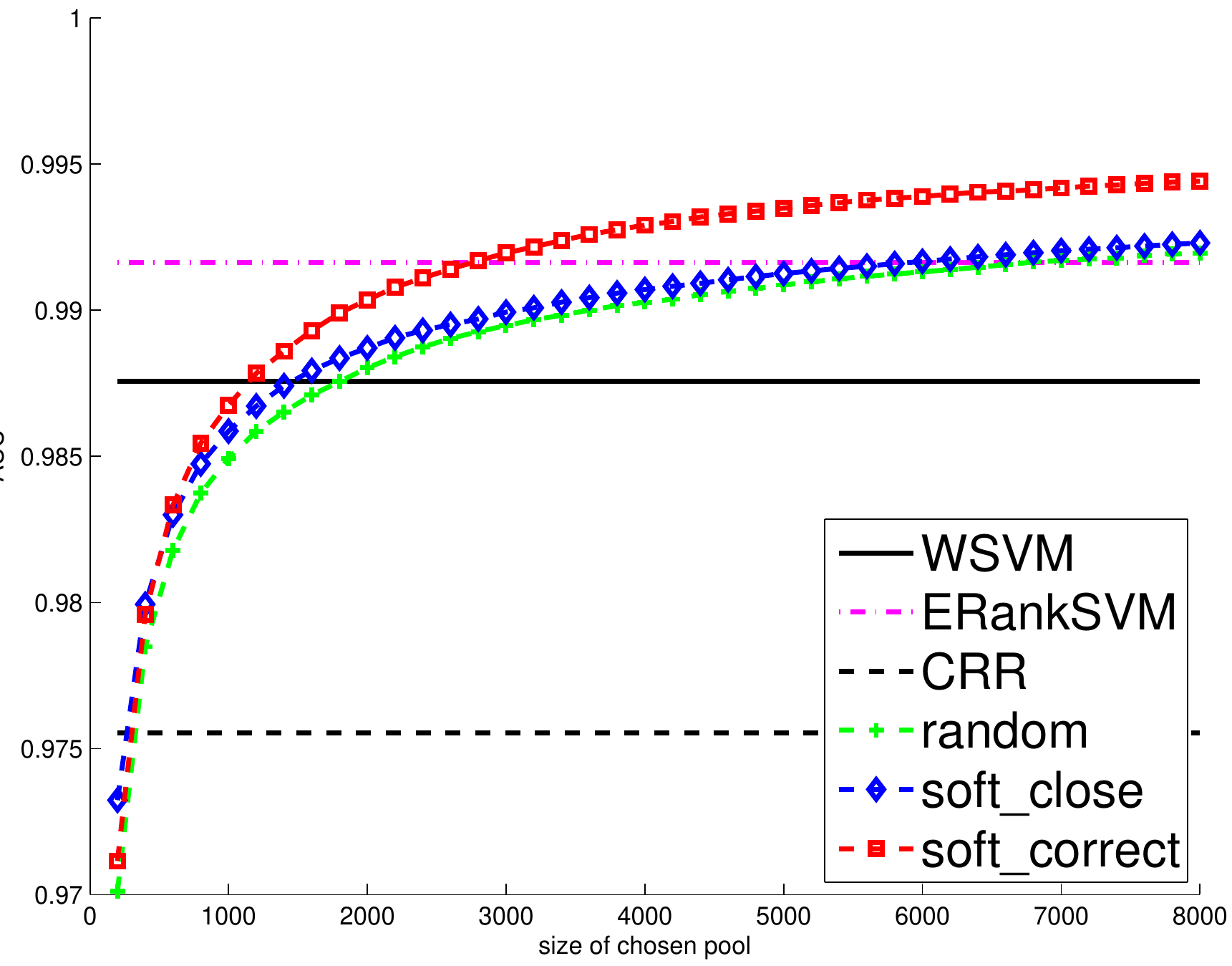}}
	\subfloat[real-sim]{\includegraphics[clip=true,trim= 2cm 7cm 1.5cm 9cm,width=0.23\textwidth]{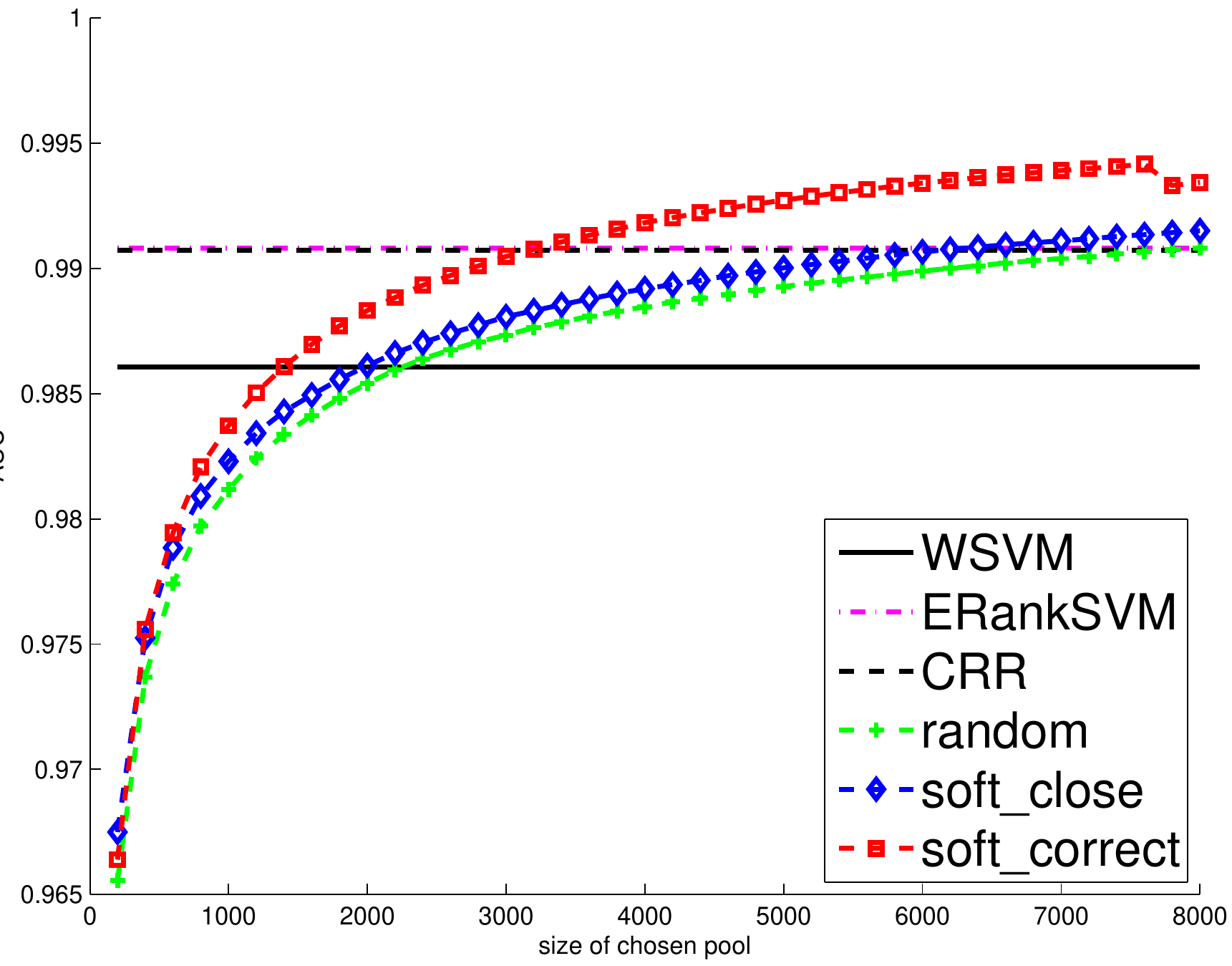}}
\\
	\subfloat[shuttle]{\includegraphics[clip=true,trim= 2cm 7cm 1.5cm 9cm,width=0.23\textwidth]{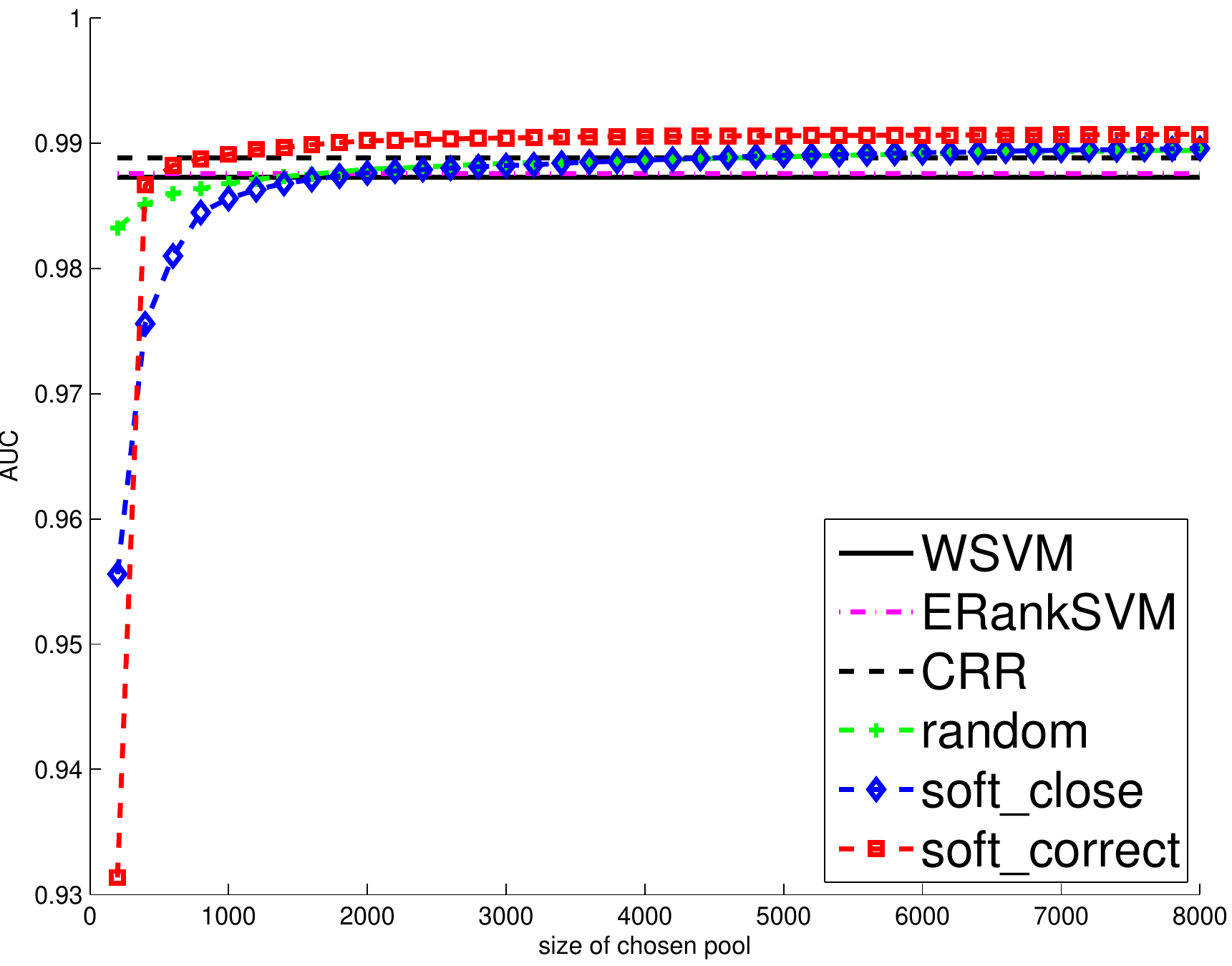}}
	\subfloat[url]{\includegraphics[clip=true,trim= 1.5cm 7cm 1.5cm 9cm,width=0.23\textwidth]{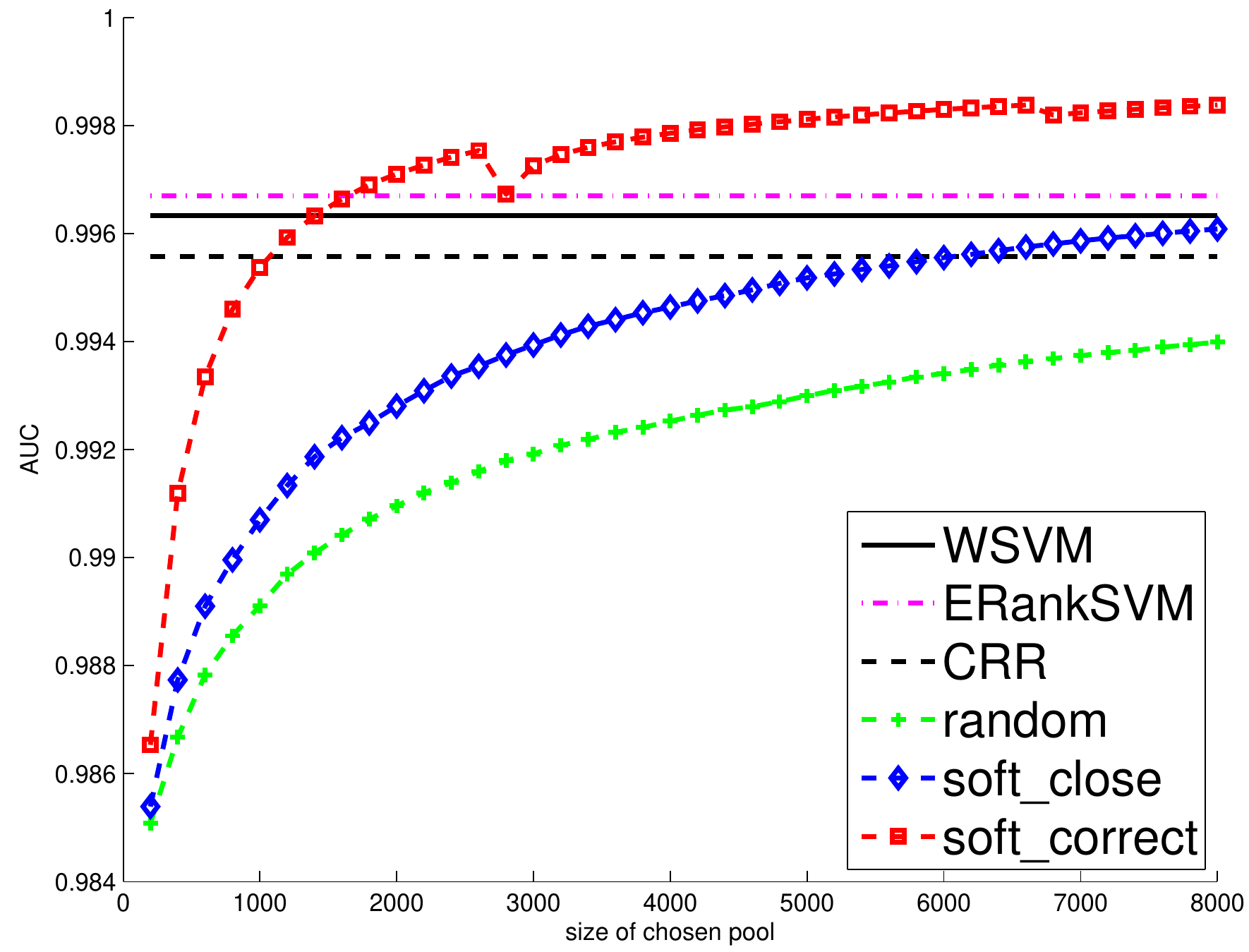}}
	\subfloat[yahoo1]{\includegraphics[clip=true,trim= 1.5cm 7cm 1.5cm 9cm,width=0.23\textwidth]{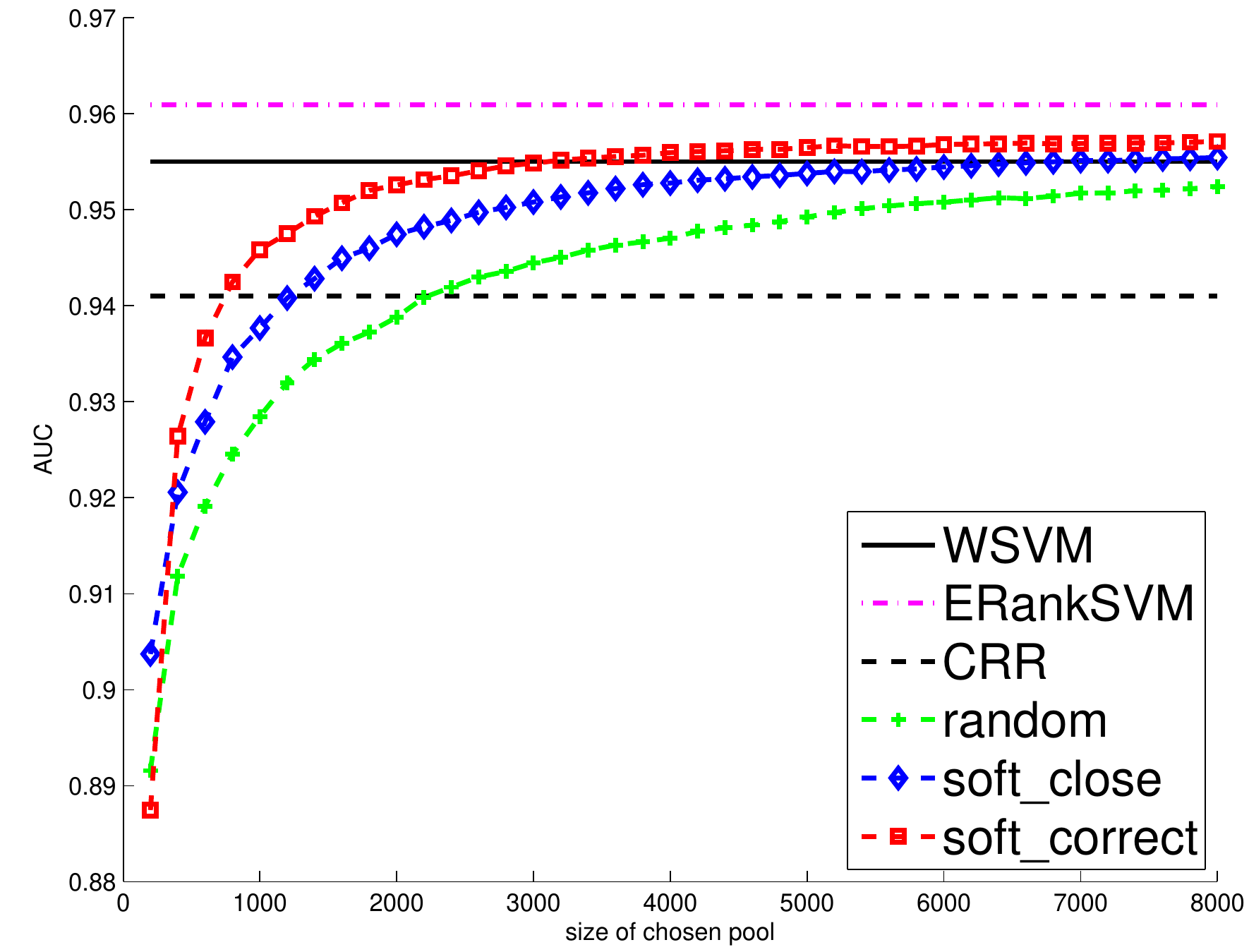}}

	%\label{fig:auc}
	\caption{Performance Curves on Different Datasets}
	\label{fig:auc}
\end{figure*}

\begin{figure*}[h]
	\centering
	\subfloat[a9a]{\includegraphics[clip=true,trim= 2cm 7cm 1.5cm 9cm,width=0.23\textwidth]{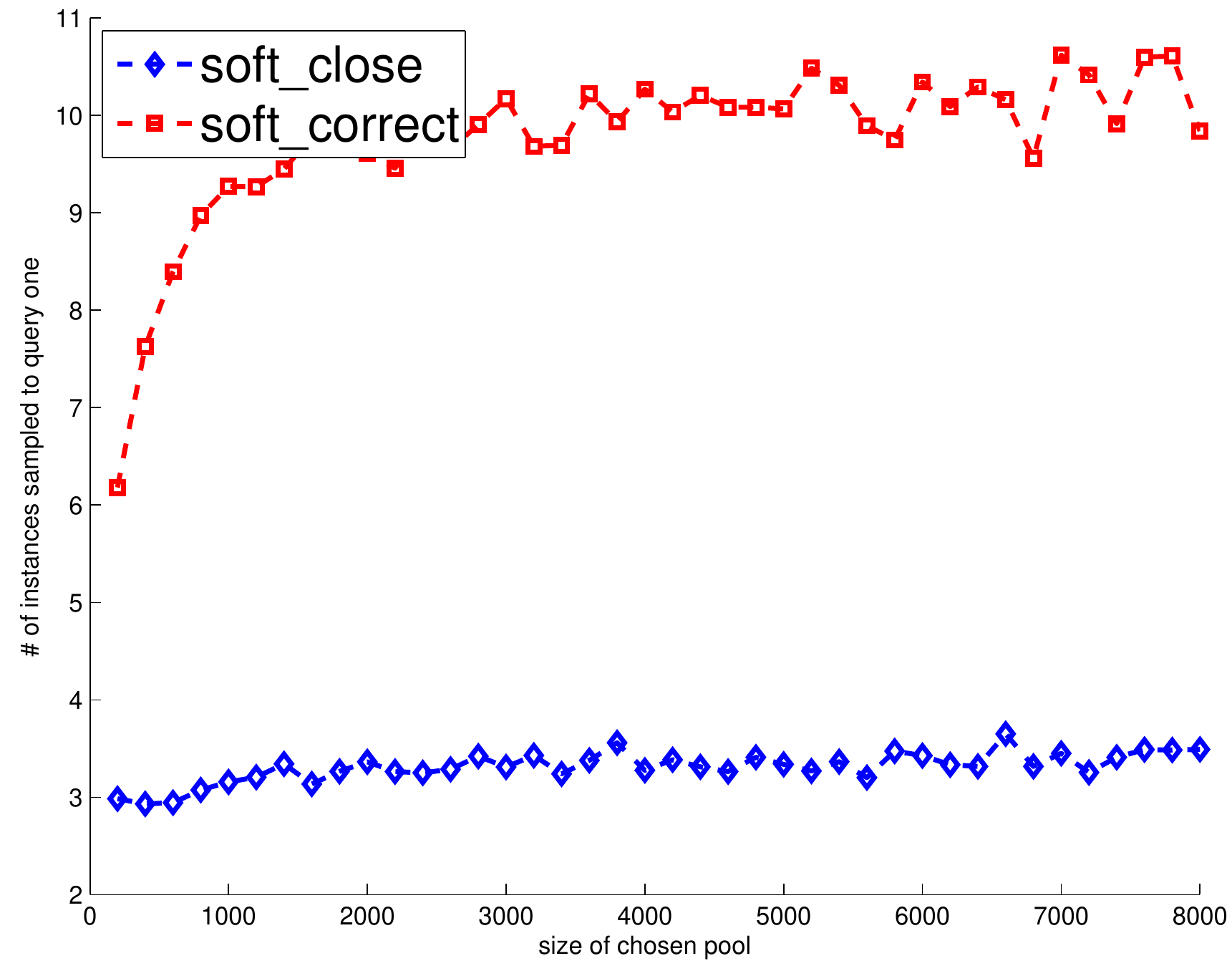}}
	\subfloat[acoustic]{\includegraphics[clip=true,trim= 2cm 7cm 1.5cm 9cm,width=0.23\textwidth]{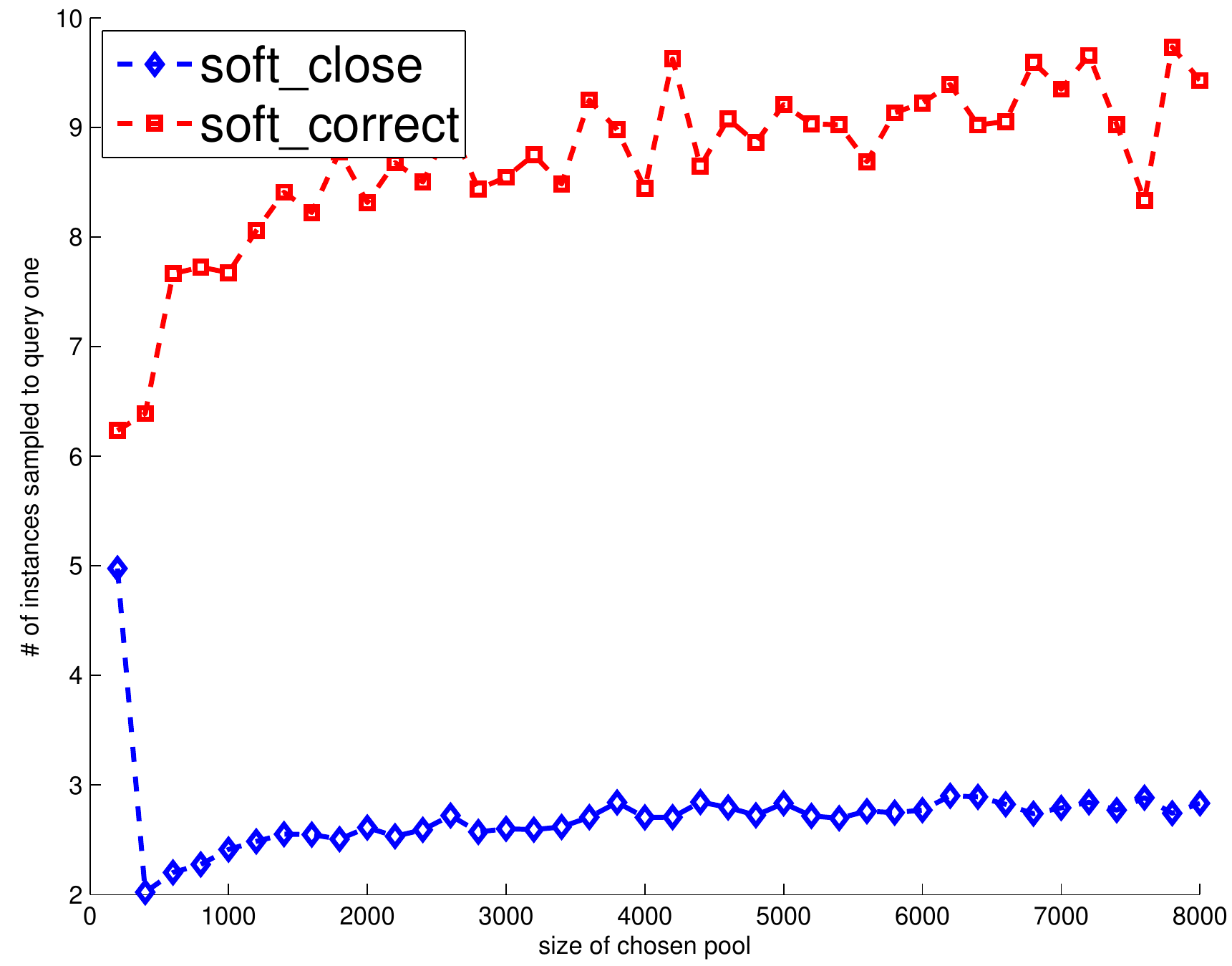}}
	\subfloat[bank]{\includegraphics[clip=true,trim= 2cm 7cm 1.5cm 9cm,width=0.23\textwidth]{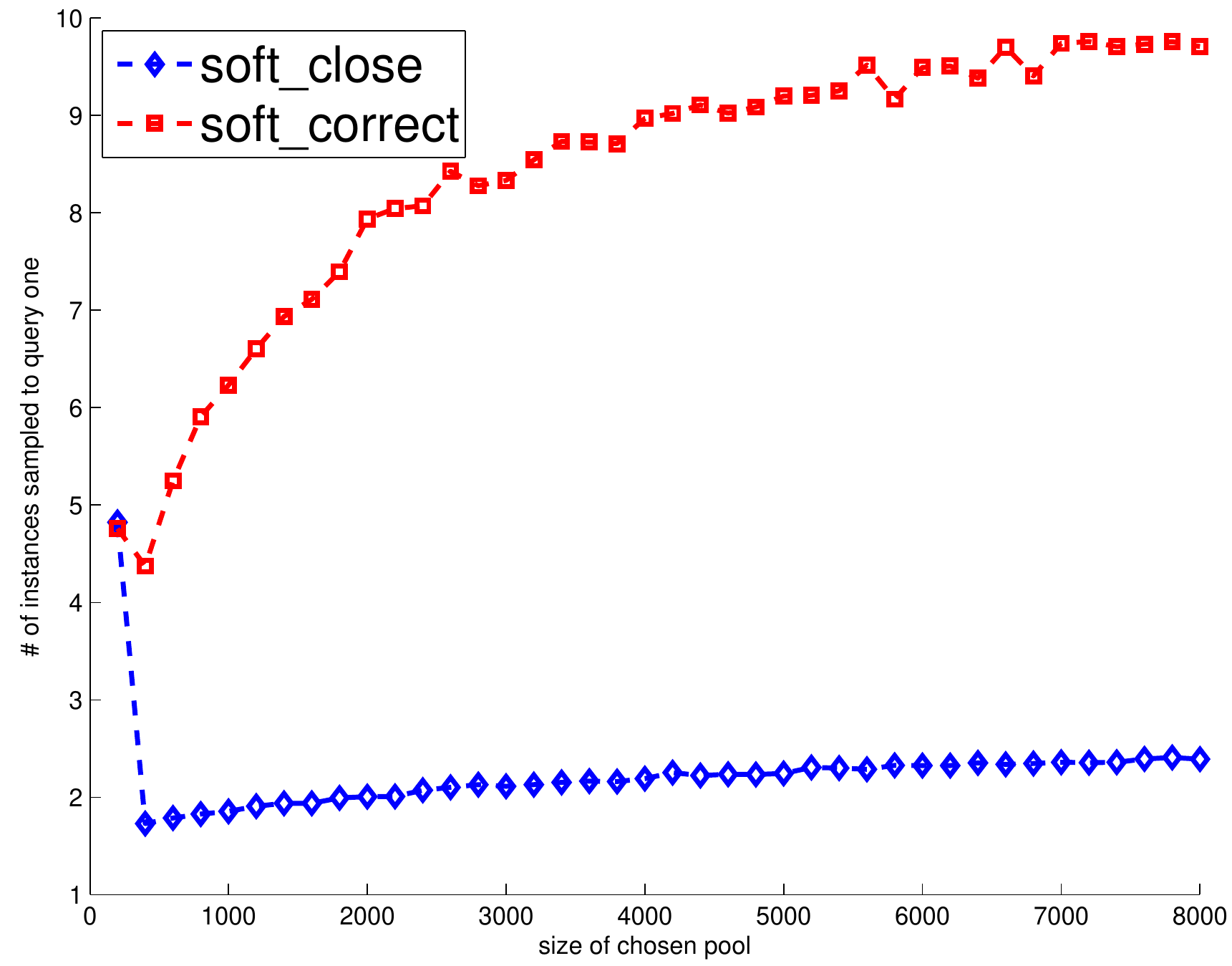}}
	\subfloat[connect]{\includegraphics[clip=true,trim= 2cm 7cm 1.5cm 9cm,width=0.23\textwidth]{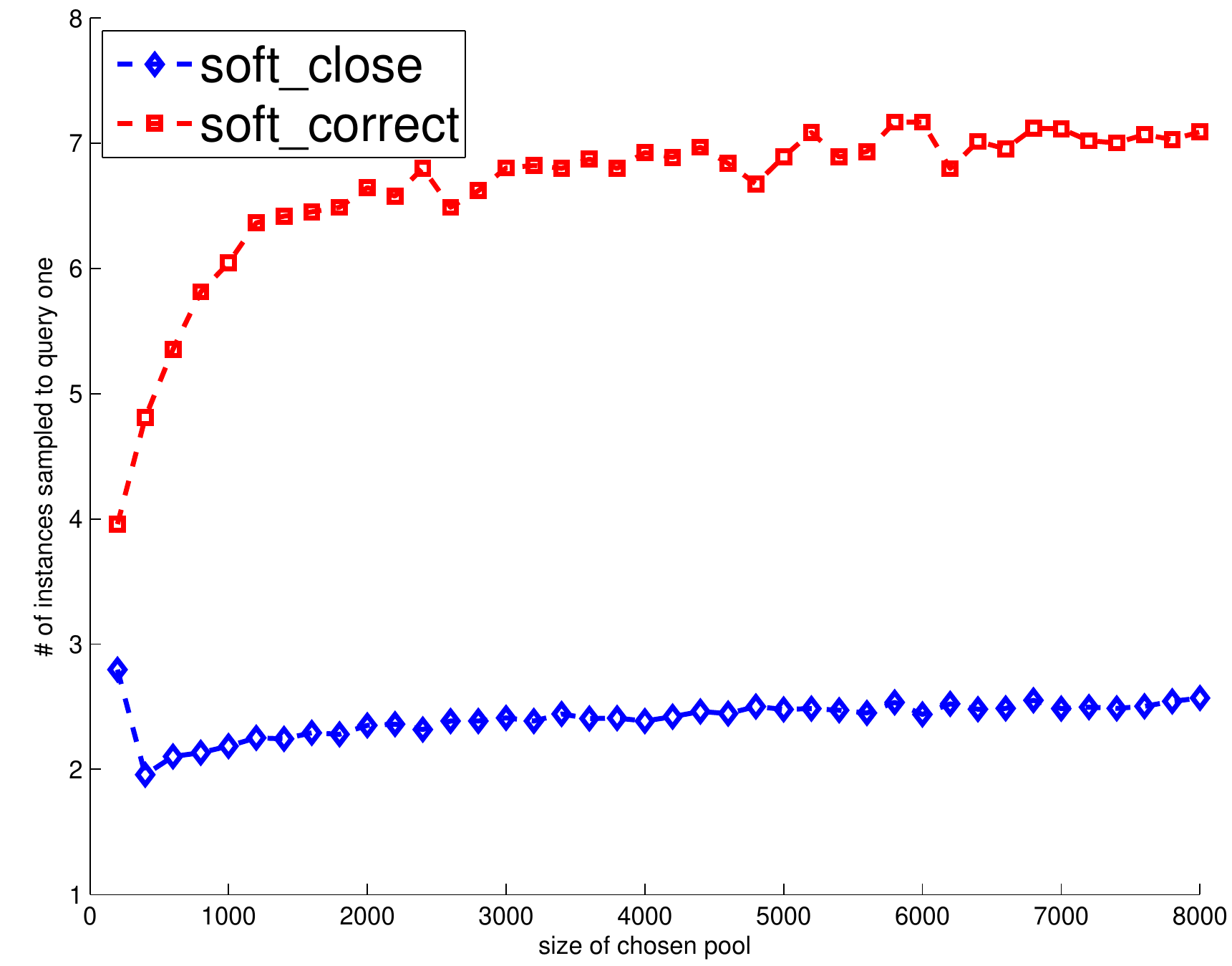}}
	\\
	\subfloat[covtype]{\includegraphics[clip=true,trim= 2cm 7cm 1.5cm 9cm,width=0.23\textwidth]{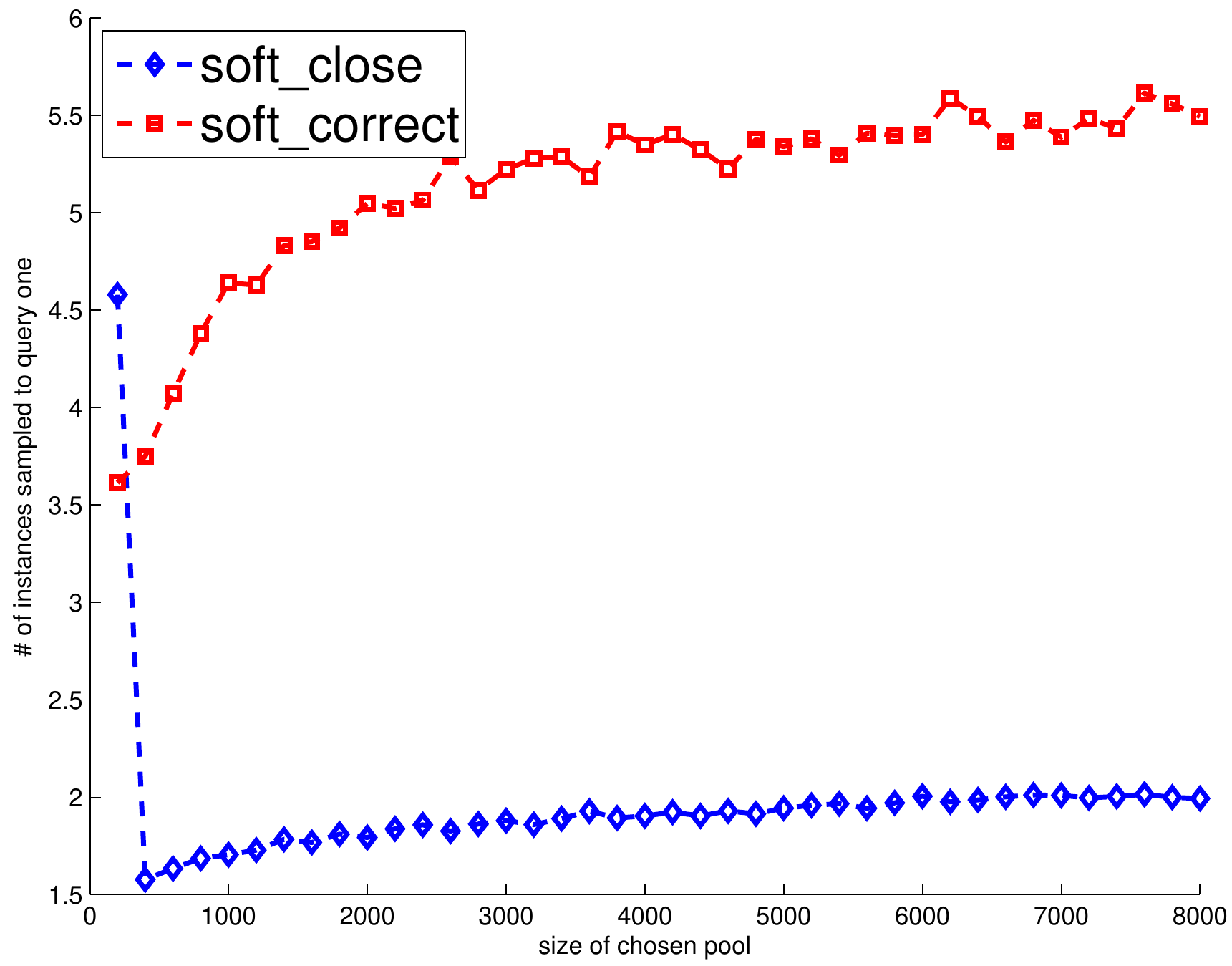}}
	\subfloat[ijcnn1]{\includegraphics[clip=true,trim= 2cm 7cm 1.5cm 9cm,width=0.23\textwidth]{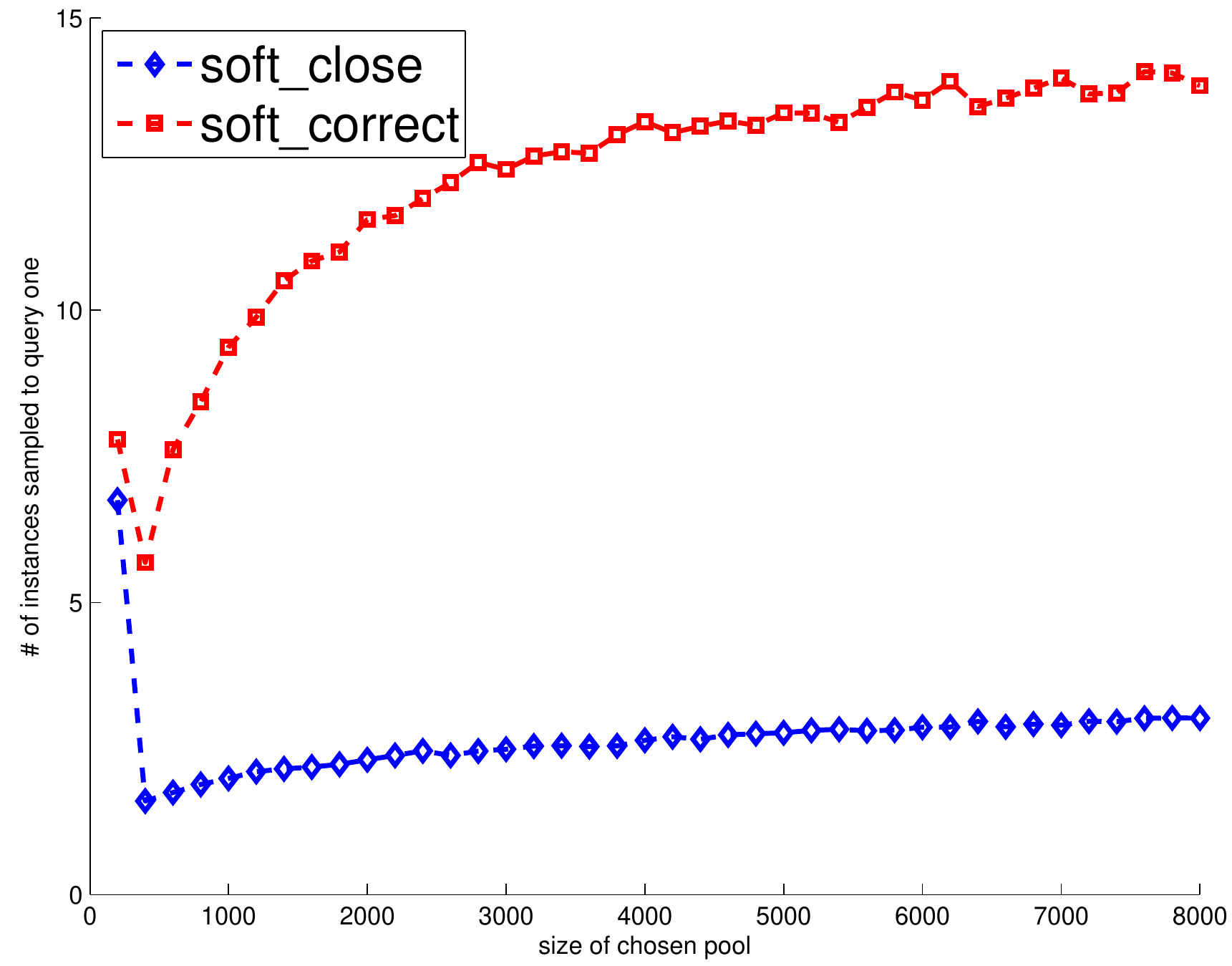}}
	\subfloat[letter]{\includegraphics[clip=true,trim= 2cm 7cm 1.5cm 9cm,width=0.23\textwidth]{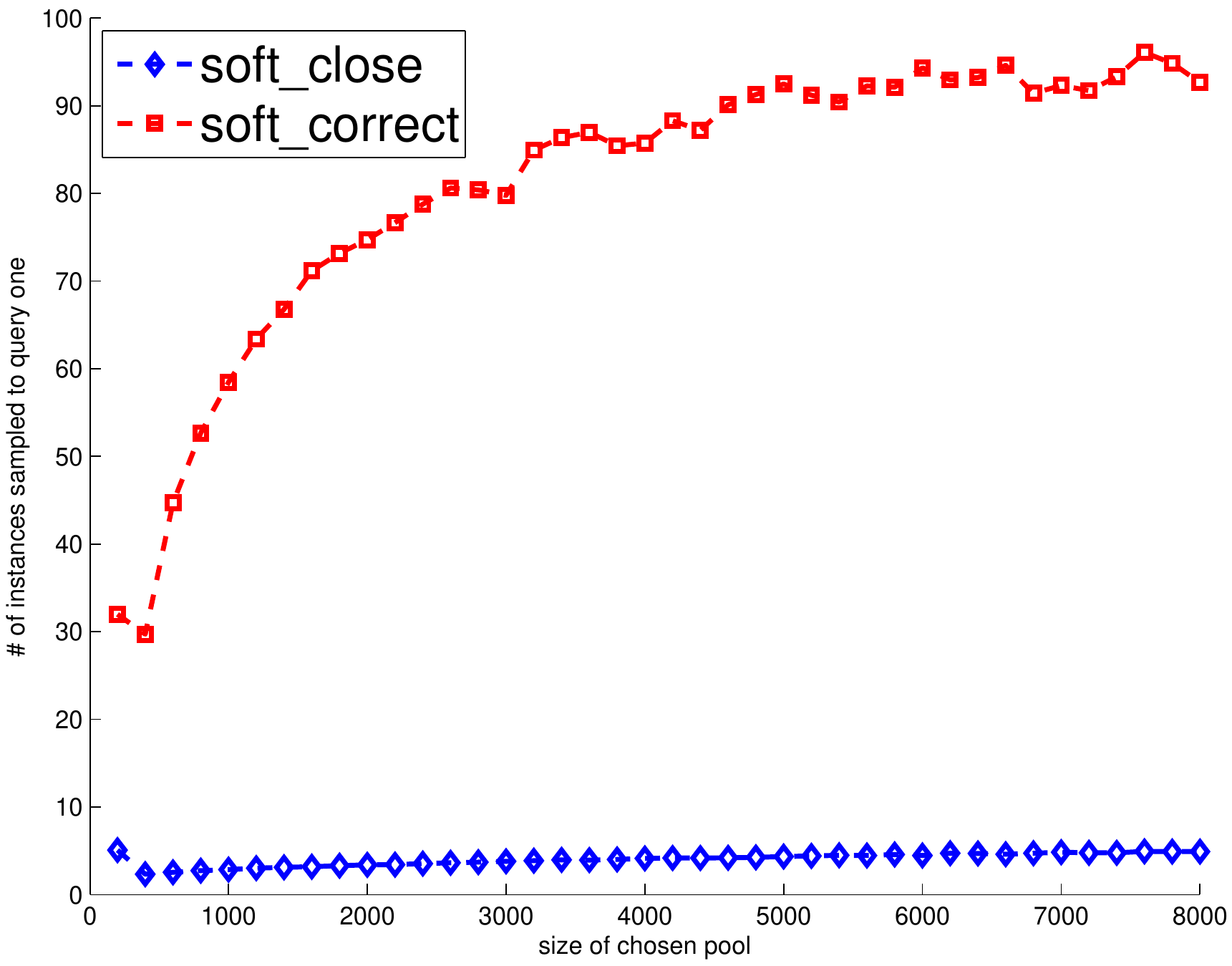}}
	\subfloat[mnist]{\includegraphics[clip=true,trim= 2cm 7cm 1.5cm 9cm,width=0.23\textwidth]{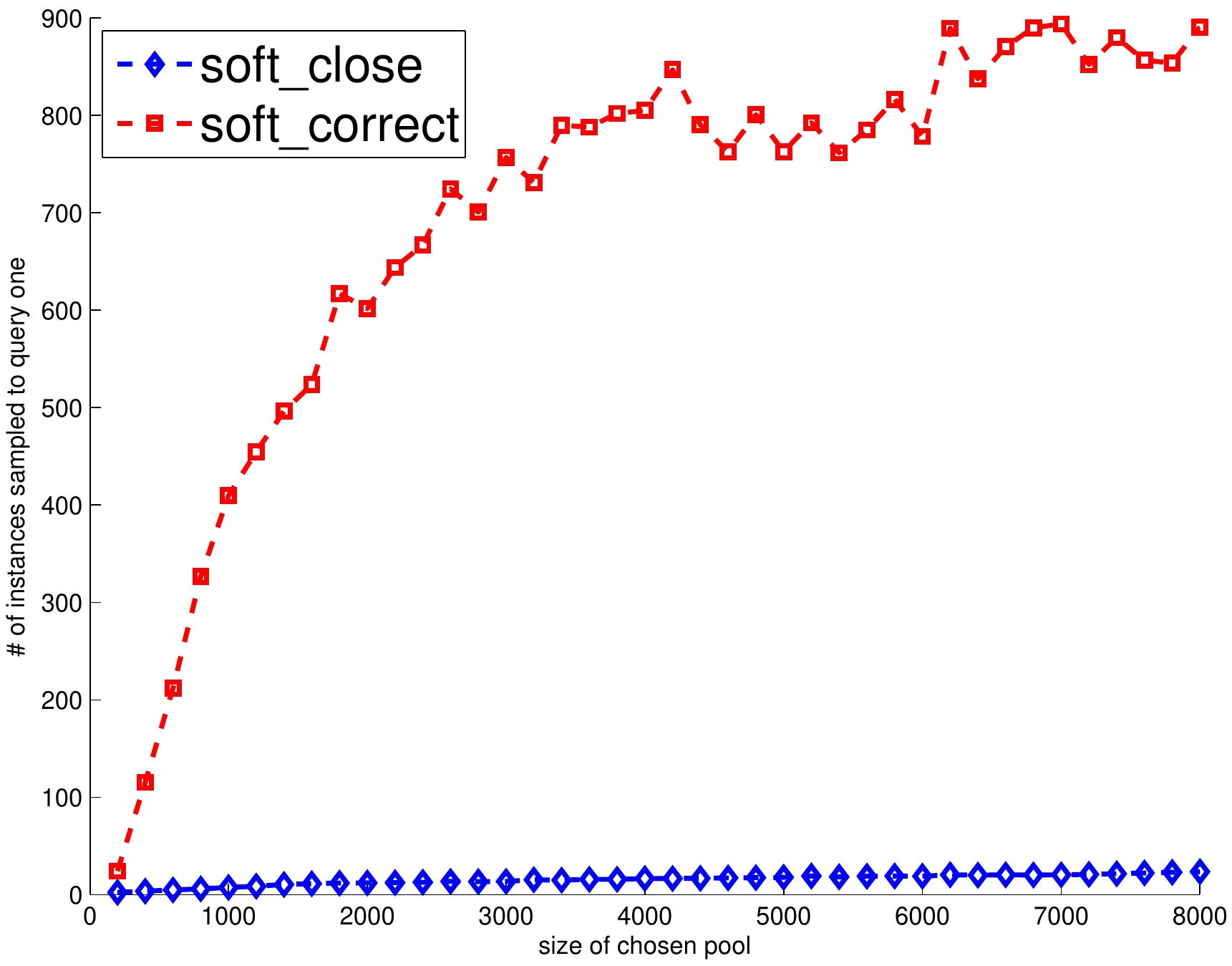}}
\\
	\subfloat[news20]{\includegraphics[clip=true,trim= 2cm 7cm 1.5cm 9cm,width=0.23\textwidth]{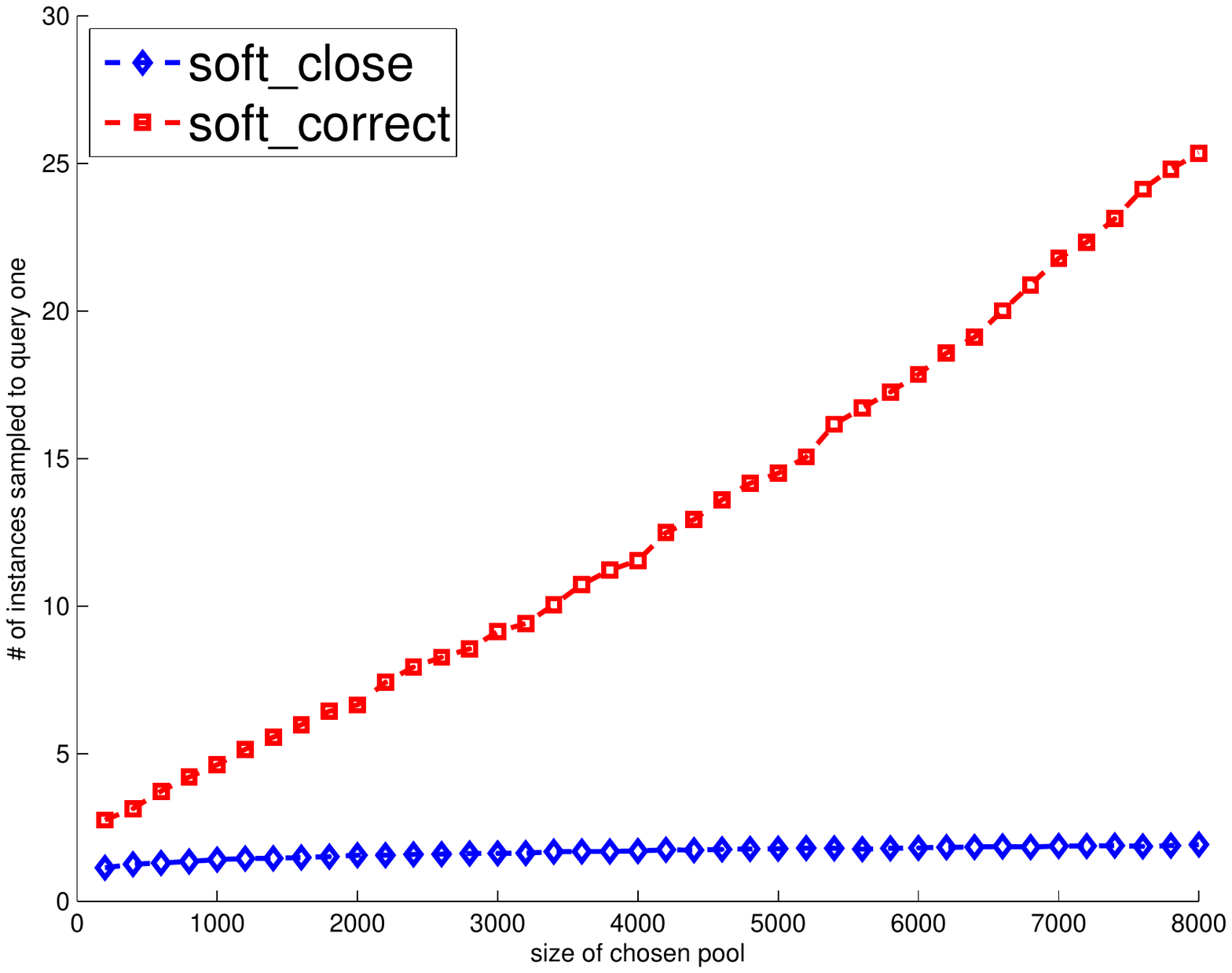}}
	\subfloat[protein]{\includegraphics[clip=true,trim= 2cm 7cm 1.5cm 9cm,width=0.23\textwidth]{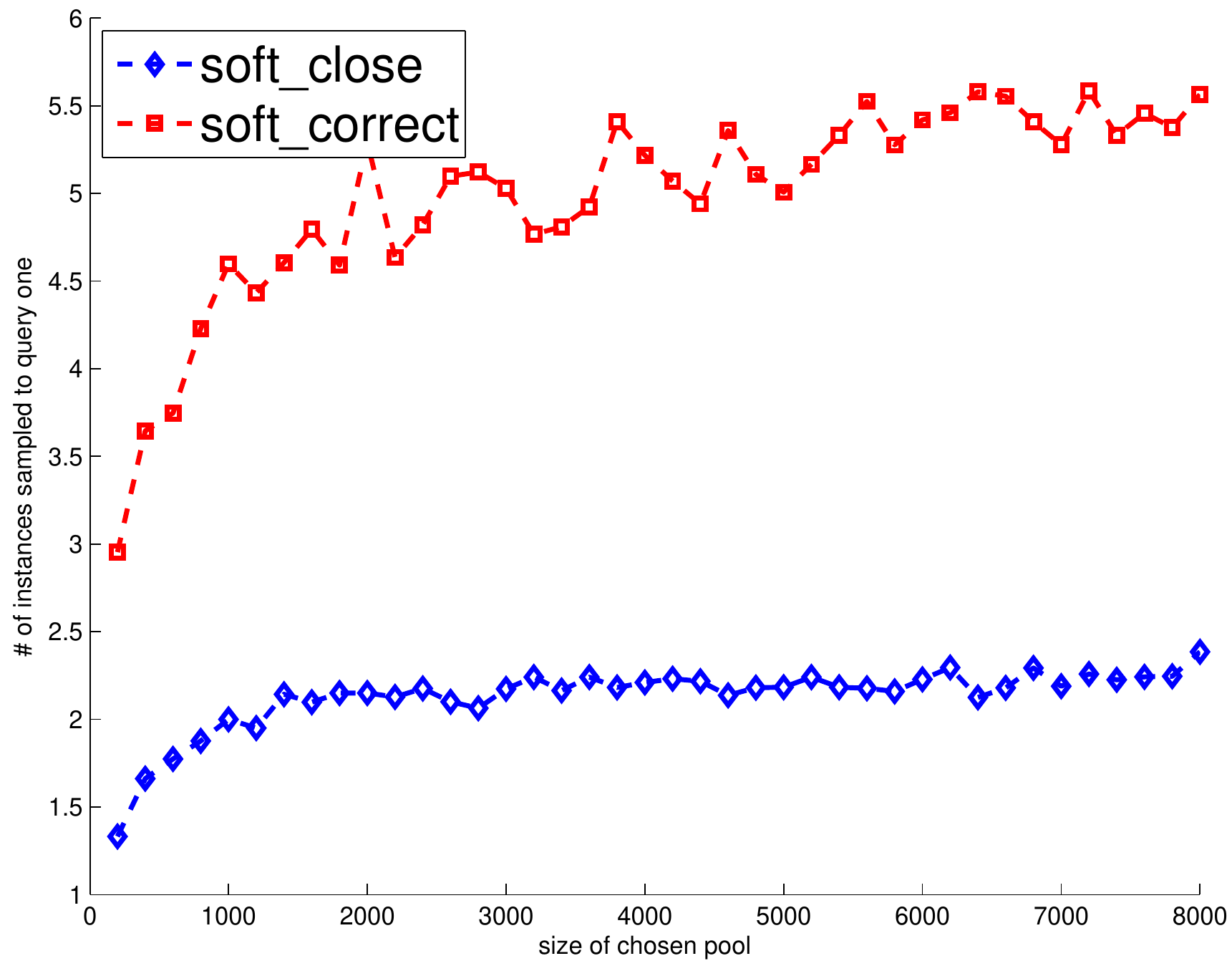}}
	\subfloat[rcv1]{\includegraphics[clip=true,trim= 2cm 7cm 1.5cm 9cm,width=0.23\textwidth]{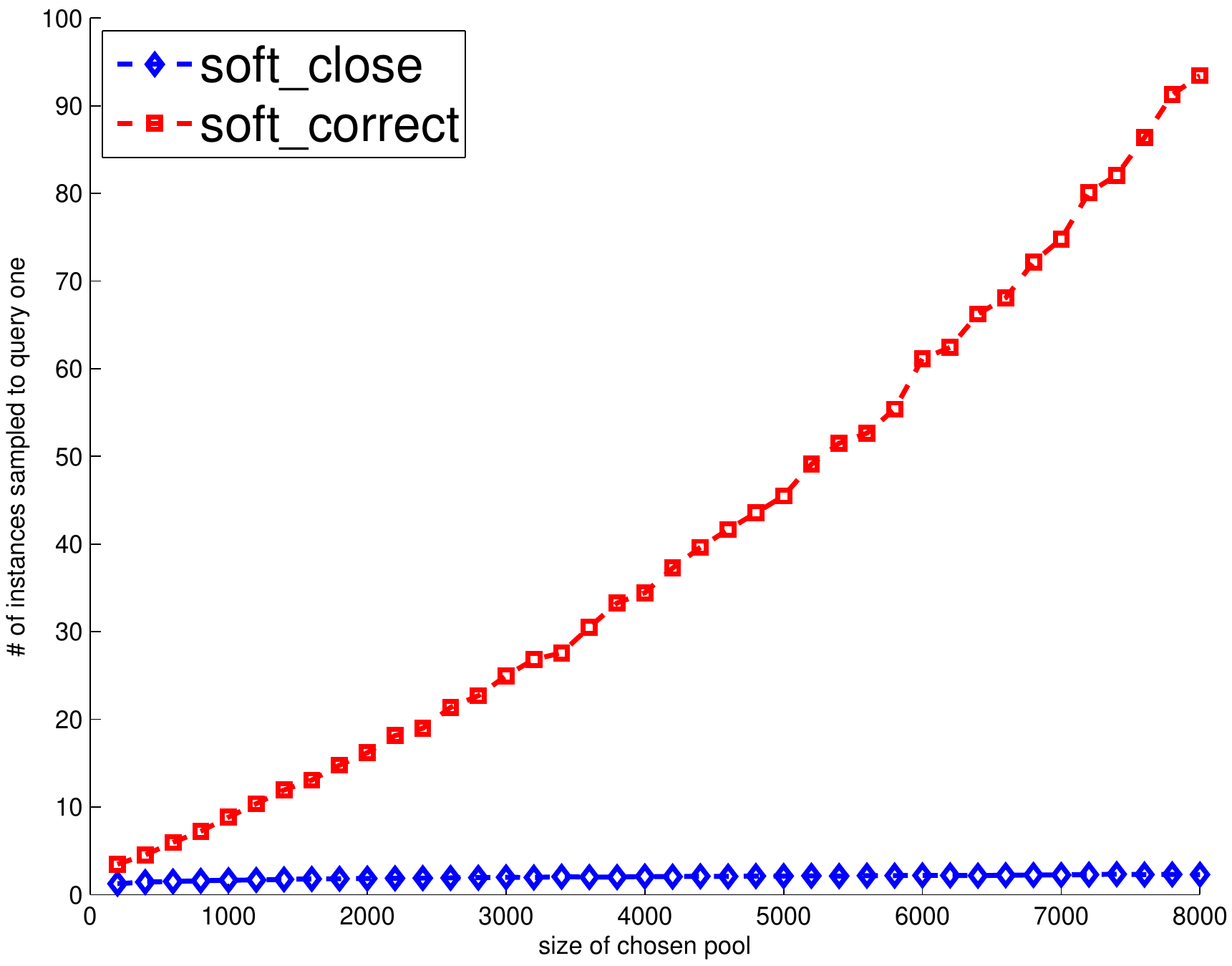}}
	\subfloat[real-sim]{\includegraphics[clip=true,trim= 2cm 7cm 1.5cm 9cm,width=0.23\textwidth]{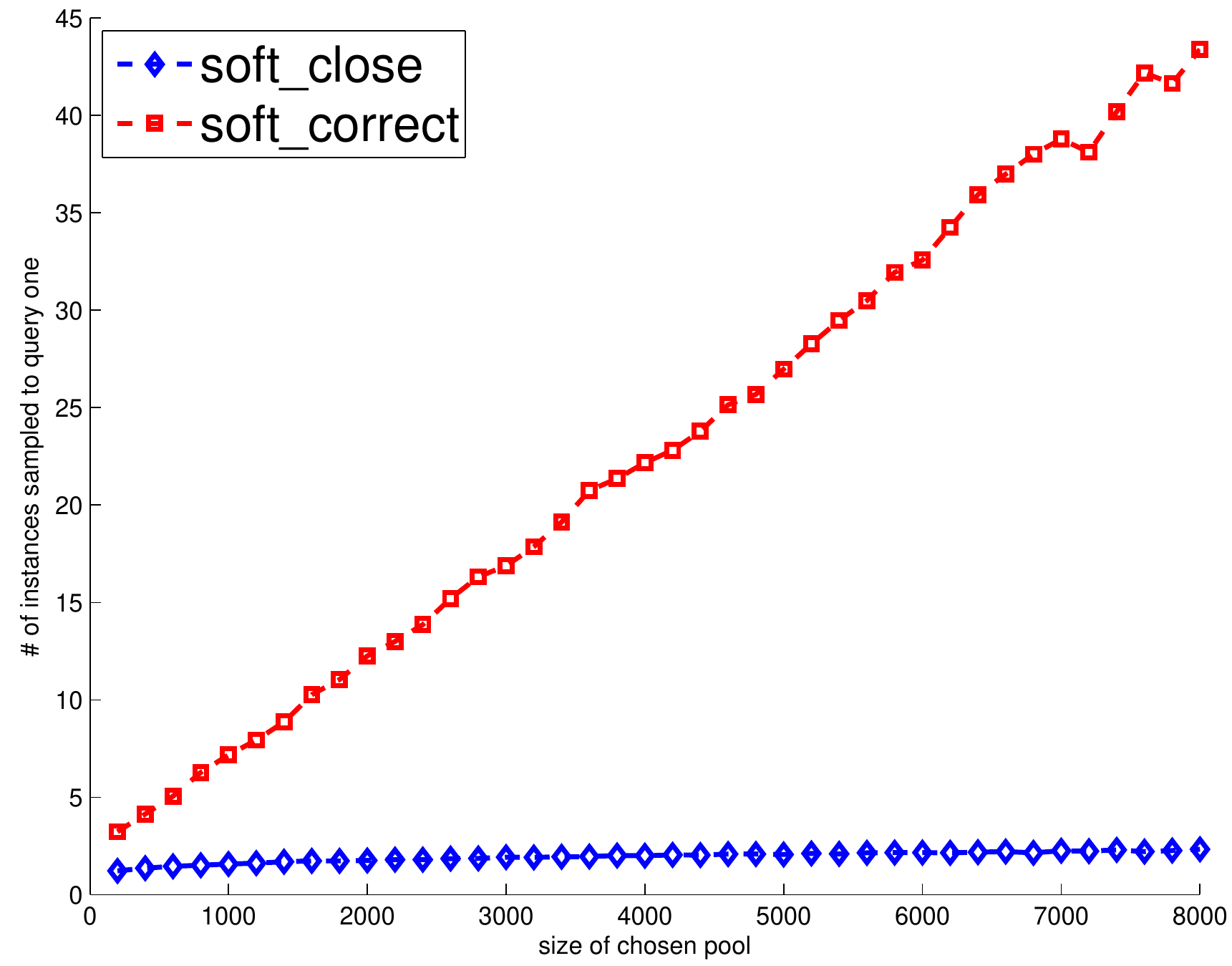}}
	\\
	\subfloat[shuttle]{\includegraphics[clip=true,trim= 2cm 7cm 1.5cm 9cm,width=0.23\textwidth]{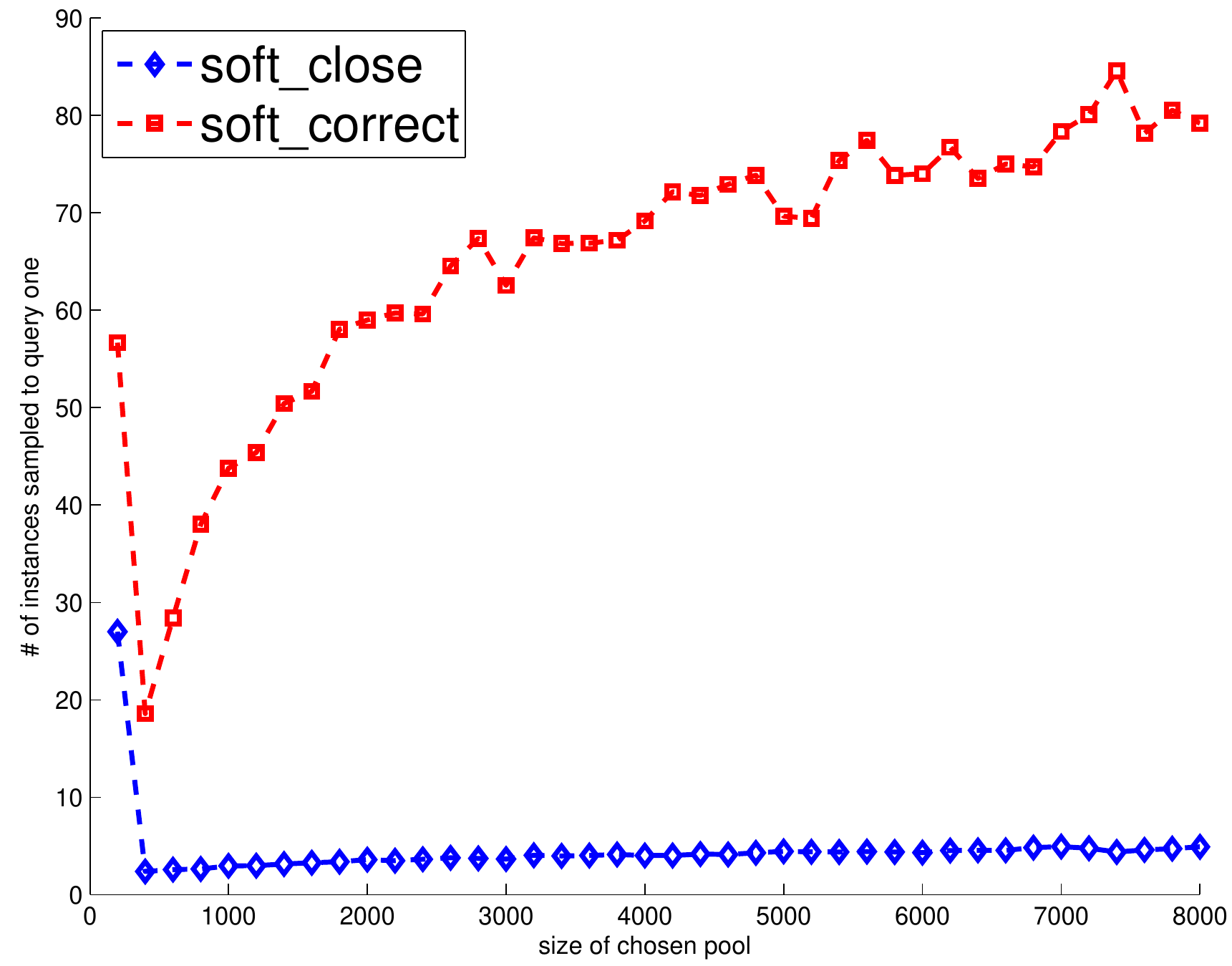}}
	\subfloat[url]{\includegraphics[clip=true,trim= 2cm 7cm 1.5cm 9cm,width=0.23\textwidth]{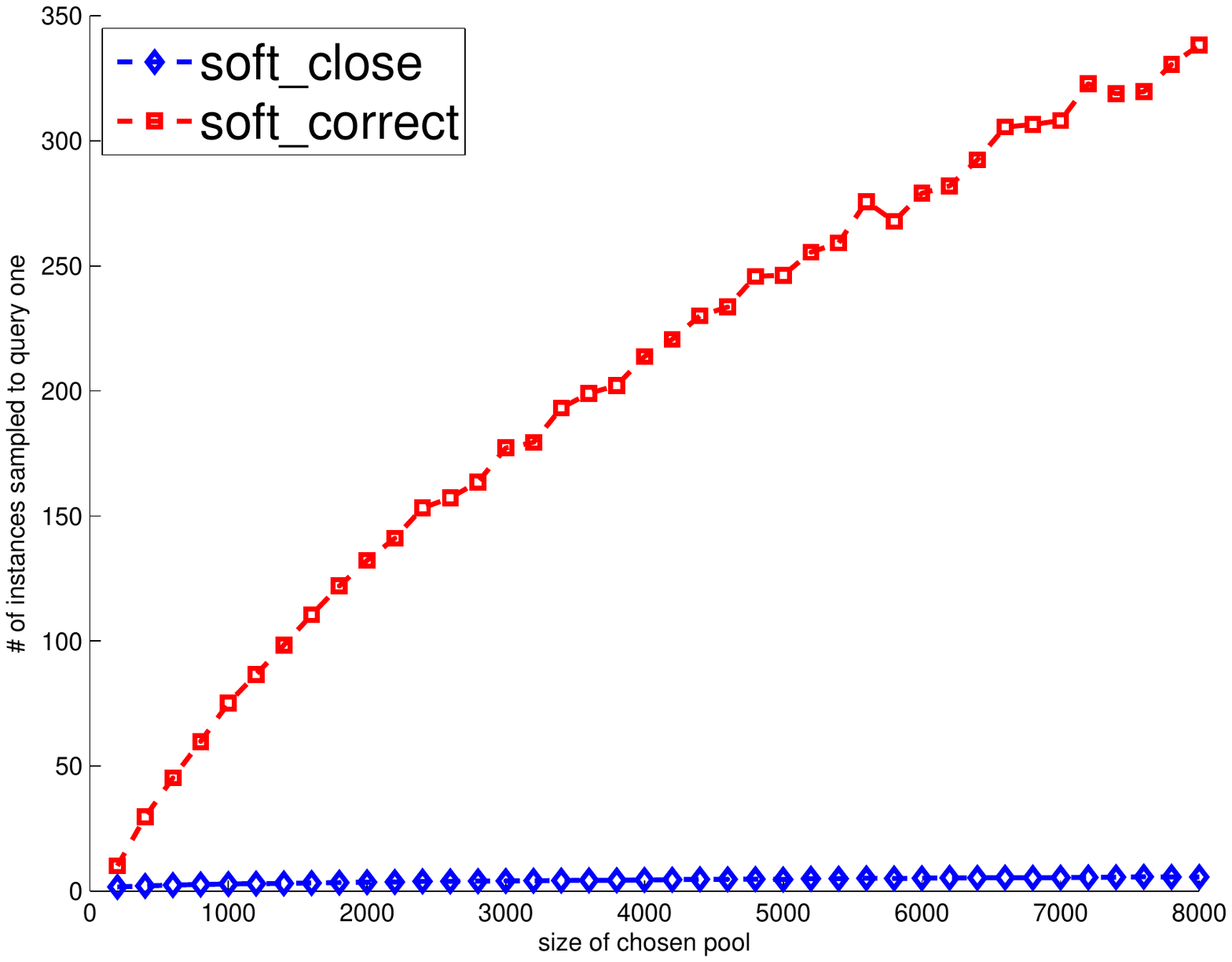}}
	\subfloat[yahoo1]{\includegraphics[clip=true,trim= 2cm 7cm 1.5cm 9cm,width=0.23\textwidth]{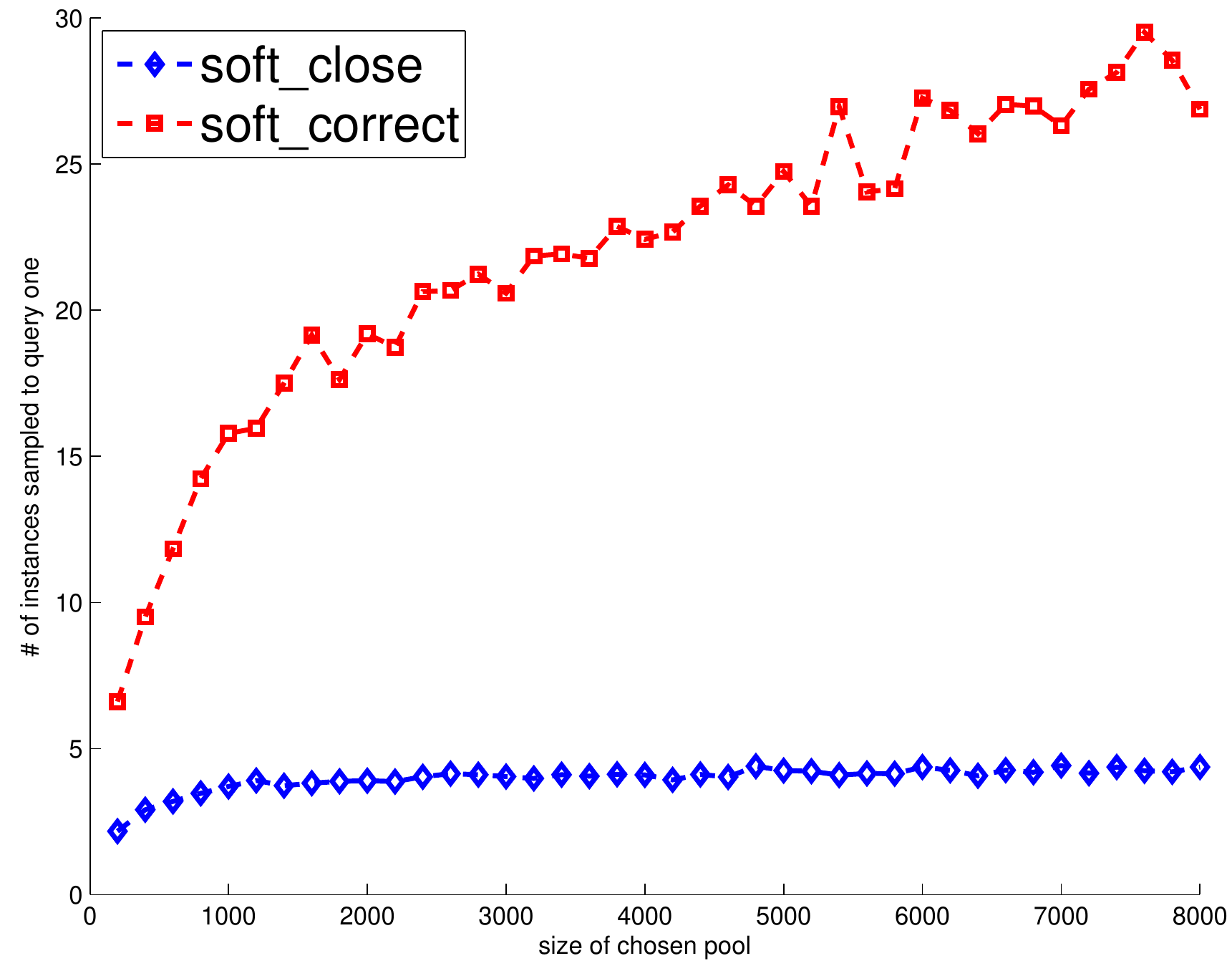}}

	%\label{fig:eff}
	\caption{Number of Pairs Rejected on Different Datasets}
	\label{fig:eff}
\end{figure*}

\section{Conclusion} \label{sec:conclusion}

We propose the algorithm of Active Sampling (AS) under Combined Ranking and Classification (CRC) based on the linear SVM. 
There are two major components of the proposed algorithm. 
The AS scheme selects valuable pairs for training and resolves the computational burden in large-scale bipartite ranking. 
The CRC framework unifies the concept of point-wise ranking and pair-wise ranking under the same framework, and can perform better than pure point-wise ranking or pair-wise ranking.
The unified view of pairs and points (pseudo-pairs) in CRC allows using one AS scheme to select from both types of pairs. 

%% ,sometimes leads to better performance 
%% In this paper, we shown the difficulty in enjoying promising performance and efficiency for large-scale bipartite ranking problem; hence we propose the Combined Ranking and Classification with Active Querying.
%% Our proposed method consists of two contributions that efficiently lead to satisfactory result:
%% the Combine Ranking and Classification(CRC) formulation that we solved for bipartite ranking problem, 
%% and the Active Querying algorithm for saving computational burden.

%% The CRC problem takes both {\it pairs} and {\it points} into it's objective function, and thus can achieve better performance than solving 
%% {\it pairs} or {\it points} alone.
%% On the other hand, the development of Active Querying algorithm allows us to solve the CRC problem efficiently on large-scale data sets.
%% Inspired form the active learning, active querying selects several informative pairs or pseudo-pairs (points) out of the huge unchosen pool iteratively. Therefore, we can solve the CRC problem efficiently on those selected key pairs instead of using all pairs and points within the data set.

Experiments on 14 real-world large-scale data sets demonstrate the promising performance and efficiency of the ASRankSVM and ASCRC algorithms.
The algorithms usually outperform state-of-the-art bipartite ranking algorithms, including the point-wise SVM, the pair-wise SVM, and the combined ranking and regression approach. 
The results not only justify the validity of ASCRC, but also shows the valuable pairs or pseudo-pairs can be helpful for large-scale bipartite ranking.

\end{document}